\documentclass[a4paper,12pt]{article}
\usepackage{cite}
\usepackage{manyfoot}

\usepackage{array}
\usepackage{textcomp}
\usepackage{stfloats}
\usepackage{verbatim}
\usepackage{graphicx}
\usepackage{caption}
\usepackage{subcaption}
\usepackage{scalerel,nicefrac}
\usepackage[super]{nth}

\usepackage{pdfsync}
\usepackage{amsmath,amsfonts,stmaryrd,amssymb,bbm,mathrsfs,amsthm} %
\usepackage{stmaryrd}
\usepackage{comment}

\usepackage{mathtools}
\usepackage{xcolor}
\usepackage{enumitem}
\usepackage{verbatim}
\usepackage{graphicx}
\usepackage{hyperref}
\usepackage{bm}
\usepackage{setspace}
\usepackage{float}
\usepackage{aligned-overset}
\usepackage[acronym]{glossaries}

\usepackage[toc]{appendix}

\usepackage{pdfrender}
\DeclareRobustCommand*{\pmbb}[1]{%
  \textpdfrender{
    TextRenderingMode=Stroke,
    LineWidth=.5pt,
  }{#1}%
}

\usepackage{./inputs/notation}

\newtheoremstyle{spaced}
  {\topsep}               
  {\topsep}               
  {\setstretch{1.2}\itshape}   
  {}                      
  {\bfseries}             
  {.}                     
  {.5em}                  
  {}                      

\newtheorem{proposition}{Proposition}[section]
\newtheorem{definition}[proposition]{Definition}

\newtheorem{theorem}[proposition]{Theorem}

\newtheorem*{assumption*}{Assumption}
\newtheorem{corollary}[proposition]{Corollary}
\newtheorem*{remark}{Remark}

\newtheorem{lemma}[proposition]{Lemma}
\newtheorem*{lemma*}{Lemma}
\newtheorem*{proposition*}{Proposition}
\newtheorem*{theorem*}{Theorem}

\usepackage[normalem]{ulem}

\setlist[description]{font=\normalfont\space}

\makeatletter
\renewcommand*\env@matrix[1][\arraystretch]{%
  \edef\arraystretch{#1}%
  \hskip -\arraycolsep
  \let\@ifnextchar\new@ifnextchar
  \array{*\c@MaxMatrixCols c}}
\makeatother

\usepackage{framed}

\usepackage{geometry}
\usepackage{setspace}

\newif\ifappendix

\definecolor{IterColor}{RGB}{0,0,0} 
\definecolor{Iter1Color}{RGB}{0,0,0} 
\definecolor{Iter2Color}{RGB}{0,0,0} 

\glsdisablehyper
\newacronym[plural=RNNs, longplural=recurrent neural networks]{rnn}{RNN}{recurrent neural network}
\newacronym{lfm}{LFM}{Lipschitz fading-memory}
\newacronym{relu}{ReLU}{rectified linear unit}
\newacronym{elfm}{ELFM}{exponentially Lipschitz fading-memory}
\newacronym{plfm}{PLFM}{polynomially Lipschitz fading-memory}

\newacronym{pou}{p.o.u.}{\emph{partition of unity}}
\newcommand{\uc}{universal constant}

\title{Metric-Entropy Limits on the Approximation of Nonlinear Dynamical Systems}
\author{Yang Pan \\ yangpan@mins.ee.ethz.ch \and Clemens Hutter \\ huttercl@mins.ee.ethz.ch \and Helmut Bölcskei\footnote{H. Bölcskei gratefully acknowledges support by the Lagrange Mathematics and Computing Research Center, Paris, France.} \\ hboelcskei@ethz.ch}
\date{}

\onecolumn

\begin{document}

\maketitle

\begin{abstract}
This paper is concerned with fundamental limits on the approximation of nonlinear dynamical systems. Specifically, we show that recurrent
neural networks (RNNs) can approximate nonlinear systems---that satisfy a Lipschitz property and forget past inputs fast enough---in metric-entropy-optimal manner. As the sets of sequence-to-sequence mappings realized by the dynamical systems we consider are significantly more massive than function classes generally analyzed in approximation theory, a refined metric-entropy characterization is needed, namely in terms of order, type, and generalized dimension. We compute these quantities for the classes of exponentially- and polynomially Lipschitz fading-memory systems and show that RNNs can achieve them. 
\end{abstract}

\vspace{1em} 
\par{\centering \em This paper is dedicated to Professor Andrew R.~Barron on the occasion of his \nth{65} birthday. \par}
\vspace{1em}

\setstretch{1.13}
\section{Introduction}\label{sec:intro}
It is well known that neural networks can approximate almost every function arbitrarily well \cite{cybenko89, funahashi89, hornik89,Barron_1993, Barron_1994}. The recently developed Kolmogorov-Donoho rate-distortion theory for deep neural network approximation \cite{bolcskeiOptimalApproximationSparsely2019, deepAT2019} goes a step further by quantifying how effective such approximations are in terms of the description complexity of the neural networks relative to that of the functions they are to approximate. Specifically, \cite{deepAT2019} 
determines the length of the bitstring needed to specify (deep) \gls{relu} networks approximating functions (mapping $\R^d$ to $\R$) in a given class 
to within a prescribed error $\epsilon$.
It is then shown that 
for a wide variety of function classes, the length of this bitstring exhibits the same scaling behavior, in $\epsilon$, as the metric entropy of the function class under consideration (see Table \ref{tbl:nn} below). This means that neural networks are universally Kolmogorov-Donoho optimal for all these function classes.

In the present paper, we extend the philosophy of \cite{deepAT2019} to the approximation of nonlinear sequence-to-sequence mappings through \glspl{rnn}. 
Concretely, we consider \gls{lfm} systems. In essence, this notion describes systems that gradually forget past inputs, with the speed of memory decay quantified in terms of a Lipschitz property. 
Such systems find application, inter alia, in finance \cite{nagel2022asset} and material science \cite{dill1975simple, wang1965principle, day2013thermodynamics}.
We first develop tools for quantifying the metric entropy of classes of \gls{lfm} systems with a given memory decay rate.
A general construction of \glspl{rnn} approximating \gls{lfm} systems is then shown to yield Kolmogorov-Donoho-optimality for  \gls{lfm} systems of exponentially or polynomially decaying memory.

\textit{Related work.} 
Learning of linear dynamical systems has been studied extensively in the literature \cite{bakshi2023new, boots2007constraint, hazan2017learning, wagenmaker2020active,hutter2022metric}. Notably, \cite{hutter2022metric} provides explicit RNN realizations and approximations for wide classes of linear dynamical systems, including time-varying ones. 
Going beyond linear systems, learning of nonlinear finite memory systems through RNNs within the finite-state-machine framework was considered in \cite{giles1995learning}. 
\textcolor{Iter1Color}{This program is extended to approximately finite-memory systems in
\cite{sandberg1994approximately} and fading-memory systems in \cite{voltera59, wiener1966nonlinear, boyd1985fading}.  }
In particular, \cite{boyd1985fading} formalizes the concept of fading-memory systems in control theory, demonstrating that continuous-time fading-memory operators can be approximated by Volterra series. Subsequently, \cite{Matthews93} established that discrete-time fading-memory systems can be identified using neural networks. Moreover, \cite{grigoryeva2018echo} demonstrated that echo state networks, a specialized architecture within the RNN family, serve as universal approximators for discrete-time fading-memory systems. None of the studies just reviewed addresses the issue of quantifying the \gls{rnn} description complexity relative to that of the class of (nonlinear) systems being approximated.

\textit{Organization of the paper.} %
The remainder of Section \ref{sec:intro} summarizes notation. In Section \ref{sec:Setup}, we introduce our setup and provide a definition of metric-entropy optimality in a very general context, encompassing the approximation of functions as well as dynamical systems.
Section \ref{sec:app_rate_lfms} develops tools for characterizing the metric entropy of LFM systems. In Section \ref{sec:rate_elfm_plfm}, \textcolor{Iter1Color}{we employ these tools to derive precise scaling results for the metric entropy of \gls{elfm} and \gls{plfm} systems. Section \ref{sec:RNN2G} presents a construction for the approximation of general \gls{lfm} systems by \glspl{rnn}. Finally, in Section \ref{sec:optimality_RNN}, we combine the results developed in the previous sections to prove that \glspl{rnn} can \textcolor{black}{approximate} \gls{elfm} and \gls{plfm} systems in metric-entropy-optimal manner}.

    \textit{Notation.} $\Nzero$ and $\Nplus$ denote the set of natural numbers including and excluding $0$, respectively. For $N\in\Nzero$, $\ssidedEnum{N}$ stands for the set $\{0,1,\dots,N\}$, while $\dsidedEnum{N}$ refers to $\{-N,\dots,-1, 0,1,\allowbreak \dots,N\}$. The cardinality of the finite set $U$ is designated by $|U|$. Sequences $x[t] \in \R$ are indexed by $t\in \mathbb{Z}$ or $t\in \Nzero$ and we use $\R^{\mathbb{Z}}$ and $\R^{\Nzero}$, respectively, to denote the set of such sequences. We refer to the set of all finite-length bitstrings by $\bitstringset$. The transpose of the matrix $A$ is $A^T$. For matrices $A_1,\dots,A_N$, 
    $\operatorname{diag}(A_1,A_2, \dots, A_N)$ denotes the block-diagonal matrix with the $A_i$ on the main diagonal. The $N\times N$ identity matrix is $\mathbb{I}_N$ and $0_{N}$ stands for the $N$-dimensional column vector with all entries equal to $0$. 
    For the vector $x\in \R^d$, we let $\|x\|_{\infty} \defeq \max_{i=1,2,\dots,d} |x_i|$. 
    $\log(\cdot)$ refers to the logarithm to base $2$, $\log^{(n)} = \log\circ \cdots \circ \log$ is the $n$-fold iterated logarithm, and $\log^\tau(\cdot)=\left(\log(\cdot)\right)^\tau$, for $\tau\in \R$. The composition of functions $f_1,f_2$ is denoted by $f_2\circ f_1$ (or $f_1\circ f_2$). 
    Let $f(\epsilon)$ and $g(\epsilon)$, in both cases for $\epsilon > 0$, be strictly positive for all small enough values of $\epsilon$. We use $f(\epsilon) = \smallo(g(\epsilon))$ to indicate that $\lim_{\epsilon \rightarrow 0} \frac{f(\epsilon)}{g(\epsilon)}=0$ and we express $\limsup_{\epsilon \rightarrow 0} \frac{f(\epsilon)}{g(\epsilon)} < \infty$ by $f(\epsilon) = \bigo(g(\epsilon))$. Moreover, we write $f(\epsilon)=\asymporder (g(\epsilon))$ when both $f(\epsilon) = \bigo(g(\epsilon))$ and $g(\epsilon) = \bigo(f(\epsilon))$. Constants are always understood to be in $\R$ unless explicitly stated otherwise. \textcolor{Iter2Color}{Finally, we say that a constant is universal if it does not depend on any of the ambient quantities.} 
\section{Problem setup and metric-entropy optimality}
\label{sec:Setup}

\subsection{\gls{relu} network approximation}\label{sec:def_nns}

We start by defining \gls{relu} networks.

\begin{definition}[\gls{relu} network \cite{deepAT2019}]\label{def:rlnn}
Let $L \in \Nplus$ and $N_0,N_1,\dots,N_L\in \Nplus$. A \gls{relu} (feedforward) neural network $\Phi$ is a mapping $\Phi~:\R^{N_0} \rightarrow \R^{N_L}$ given by
\begin{equation}
    \Phi=\left\{\begin{array}{ll}
    \hspace{-3pt}
    W_{1}, & L=1 \\ 
    \hspace{-3pt}
    W_{2} \circ \rho \circ W_{1}, & L=2 \\ 
    \hspace{-3pt}
    W_{L} \circ \rho \circ W_{L-1} \circ \rho \circ \cdots \circ \rho \circ W_{1}, & L \geq 3,\end{array}\right.
\end{equation}
where, for $\ell \in \{1,2,\dots,L\}$, $W_{\ell}:\R^{N_{\ell-1}} \rightarrow \R^{N_{\ell}}$, $W_{\ell}(x) := A_{\ell} x+b_{\ell}$, $x\in\R^{N_{\ell-1}}$, are affine transformations with (weight) matrices $A_{\ell} \in \R^{N_{\ell}\times N_{\ell-1}}$ and (bias) vectors $b_{\ell} \in \R^{N_{\ell}}$, and the \gls{relu} activation function $\relu: \R\rightarrow \R$, $\relu(x) = \max\{0,x\}$ acts component-wise, i.e., $\relu(x_1,\dots, x_N) = (\relu(x_1),\dots, \relu(x_N))$. We denote the set of all \gls{relu} networks with input dimension $N_0 =d$ and output dimension $N_L = d'$ by $\relunn_{d,d'}$. Moreover, we define the following quantities related to the notion of size of the \gls{relu} network $\Phi$:
\begin{itemize}
    \item depth $\depth(\Phi) \defeq L$,
    \item the connectivity $\nnz(\Phi)$ of the network $\Phi$ is the total number of non-zero entries in the matrices $A_{\ell}, ~\ell \in \{1,2,...,L\}$, and the vectors $b_{\ell}, ~\ell \in \{1,2,...,L\}$,
    \item width $\width(\Phi)\defeq\max _{\ell=0, \ldots, L} N_{\ell}$,
    \item the weight set $\weightsSet(\Phi)$ denotes the set of non-zero entries in the matrices $A_{\ell}, ~\ell \in \{1,2,...,L\}$, and the vectors $b_{\ell}, ~\ell \in \{1,2,...,L\}$,
    \item weight magnitude $\weightmeg(\Phi)\defeq \max _{\ell=1, \ldots, L} \max \left\{\max_{i,j}\left|(A_{\ell})_{ij}\right|,\left\|b_{\ell}\right\|_{\infty}\right\}$.
\end{itemize}
\end{definition}

We next formalize the concept of network weight quantization. 

 \begin{definition}[Quantization \cite{deepAT2019}]\label{def:quan_nn}
    Let $m\in \Nplus$ and $\epsilon\in (0,1/2)$. The network $\Phi$ is said to have $(m,\epsilon)$-quantized weights if 
    $\weightsSet(\Phi)\subset \left(2^{-m\ceil{\log(\epsilon^{-1})}} \ZZ\right) \cap \left[-\epsilon^{-m},\epsilon^{-m}\right]$. 
    Moreover, for $a\in \R$, we define the $(m,\epsilon)$-quantization mapping rounding $a$ to an integer multiple of $2^{-m\ceil{\log (\epsilon^{-1})}}$ according to
    \begin{equation}\label{eq:quantize_realNum}
        \quantize{m}{\epsilon}(a) = \ceil{\frac{a}{2^{-m\ceil{\log (\epsilon^{-1})}}}}  \cdot 2^{-m\ceil{\log (\epsilon^{-1})}}.
    \end{equation}

\end{definition}

Every quantized \gls{relu} network can be represented by a bitstring (see \ref{app:define_nn_bitstring}) specifying the topology of the network along with its quantized weights.
The mapping $\decoderNN$ taking this bitstring back to the quantized \gls{relu} network is referred to as the \emph{canonical neural network decoder}.

\begin{remark}\label{rm:bitstring}
    For every \gls{relu} network $\Phi$ with $(m, \epsilon)$-quantized weights, there is a bitstring $\bitstring$ of length no more than $C_0 m \log(\epsilon^{-1}) \nnz(\Phi) \log(\nnz(\Phi))$ such that $\decoderNN(\bitstring) = \Phi$, with $C_0>0$ a {\uc}. This follows by upper-bounding \eqref{eq:exact_bitstring_length} in \ref{app:define_nn_bitstring}.
\end{remark}

As we consider the approximation of sequence-to-sequence mappings $(\R^\ZZ \to \R^\ZZ)$,
feedforward networks are not directly applicable because they effect mappings between finite-dimensional spaces, concretely from $\R^{N_0}$ to 
$\R^{N_L}$. However, and perhaps surprisingly, simply applying feedforward networks iteratively in a judicious manner turns out to be sufficient for approximating interesting classes of nonlinear sequence-to-sequence mappings in Kolmogorov-Donoho-optimal manner. Concretely, this is effected through recurrent neural networks, defined as follows.

\begin{definition}[Recurrent neural networks \cite{pascanu2013construct,hutter2022metric}]\label{def:RNN}
For $m \in \Nplus$, let $\Phi \in \relunn_{m+1,m+1}$ be a \gls{relu} network of depth $\mathcal{L}(\Phi) \geq 2$. The recurrent neural network (RNN) associated with $\Phi$ is the operator $\mathcal{R}_{\Phi}:\R^{\Nzero} \rightarrow \R^{\Nzero}$ mapping input sequences $(x[t])_{t\in \Nzero}$ in $\R$ to output sequences $(y[t])_{t\in \Nzero}$ in $\R$ according to
\begin{equation}\label{eq:rnn}
    \begin{pmatrix}y[t] \\ h[t]\end{pmatrix}=\Phi\hspace{-2pt}\left(\begin{pmatrix} x[t] \\ h[t-1]\end{pmatrix}\right), \;\forall t\in \Nzero,
\end{equation}
where $h[t] \in \R^m$ is the hidden state vector sequence with initial state $h[-1] = 0_m$. We denote the set of all \glspl{rnn}  by $\rnnset$.
\end{definition} 
\begin{remark}
    When unfolded in time, an \gls{rnn} simply amounts to repeated application of $\Phi$.
\end{remark}

From Definition \ref{def:RNN} it is apparent that an \gls{rnn} $\rnnOp{\Phi}$ is fully specified by its associated feedforward network $\Phi$.
\begin{definition}\label{def:maptornn}
We define $\maptornn$
as the mapping that takes a \gls{relu} network $\Phi$ to its associated \gls{rnn} $\rnnOp{\Phi}$ according to Definition \ref{def:RNN}. 
\end{definition}
Together with the canonical neural network decoder $\decoderNN$, we thus obtain a decoder taking a bitstring into an RNN as follows.

\begin{definition}[Canonical RNN decoder]\label{def:canonical_rnn_decoder}
    We define the canonical RNN decoder as $\decoderRNN = \maptornn \circ \decoderNN$, where $\maptornn$ is as in Definition \ref{def:maptornn} 
    and $\decoderNN$ is the canonical neural network decoder.
\end{definition}

The main point of this paper is to show that the canonical \gls{rnn} decoder is capable of approximating a wide variety of non-linear sequence-to-sequence mappings $(\R^\ZZ \to \R^\ZZ)$ in metric-entropy-optimal manner.
Together with the results in \cite{deepAT2019}, this establishes that 
\gls{relu} networks optimally approximate a wide range of function classes and sequence-to-sequence mappings.

\subsection{Metric-entropy optimality}\label{sec:def_opt} 

In this section, we define the notion of metric-entropy optimal approximation.
Consider a metric space $(\setX, \genMetric)$ and a compact set $\entSet \subset \setX$. Together, $\entSet$ and $\genMetric$  determine an  approximation task. Specifically, we wish to approximate elements $\exampleElementOfX \in \entSet$ to within a prescribed error $\epsilon>0$ in the metric $\genMetric$ by elements $\approximationOfExample \in \setX$ which can be encoded by finite-length bitstrings $\bitstring \in \{0, 1\}^\ell$. 
To go from bitstrings to elements of $\setX$, we define decoder mappings as follows. 

\begin{definition}\label{def:decoder}
    A decoder $\decoderGen: \bitstringset \to \setX$ is a mapping from bitstrings of arbitrary length to elements of $\setX$.
\end{definition}

We shall frequently want to quantify 
how well a given decoder $\decoderGen$ performs.

\begin{definition}
\label{def:representable}
\label{def:repre_sys} 
    Given a metric space $(\setX, \genMetric)$, a compact set $\entSet \subset \setX$, and  a decoder $\decoderGen: \bitstringset \to \setX $, 
    we say that $(\entSet, \genMetric)$ is representable by  $\decoderGen$, if for every $\epsilon> 0$ and every $f \in \entSet$, there exist $\ell \in \Nplus$ and a bitstring $\bitstring \in \{0,1\}^\ell$ such that
    \begin{equation*}
        \genMetric( \decoderGen(\bitstring), f) \leq \epsilon.
    \end{equation*}
    Furthermore, we set 
    \begin{equation*}
        \Ldec{\epsilon}{\decoderGen}{\entSet}{\genMetric} 
        \defeq \min\left\{\ell'\in \Nplus \mid \forall f\in \entSet, ~\exists\ell\leq \ell', ~\exists\bitstring \in \{0,1\}^{\ell}~ \text{s.t. }~\genMetric( \decoderGen(\bitstring), f) \leq \epsilon\right\}.
    \end{equation*}
\end{definition}
\begin{remark}
    This setting allows us to fix a decoder $\decoderGen$ (e.g., the canonical neural network decoder) and then study how well $\decoderGen$ performs on different $(\entSet, \genMetric)$. That is, $\decoderGen$ does not depend on $\entSet, \genMetric, f$, or $\epsilon$. 
\end{remark}

The quantity $\Ldec{\epsilon}{\decoderGen}{\entSet}{\genMetric} $ measures how bit-efficient the decoder $\decoderGen$ is in representing $\entSet$ with respect to $\genMetric$.  It is now natural to ask what the minimum  required number of bits, independently of $\decoderGen$, is for representing
$\entSet$ with respect to $\genMetric$.
The concept of metric entropy \cite{kolmogorov1959varepsilon, wainwright2019high} allows to answer this question.
\begin{definition}\label{def:both_covering_numbers}
    Let $(\setX,\genMetric)$ be a metric space and $\entSet \subset \setX$ compact. The set $\{x_1,x_2,\dots,x_N\} \allowbreak\subset \entSet$ \emph{({\strut}respectively $\{x_1,x_2,\dots,x_N\}\subset \setX$){\strut}} is an $\epsilon$-covering \emph{(respectively $\epsilon$-net)} for $(\entSet, \genMetric)$ if, for each $x\in \entSet$, there exists an $i \in \{1,2,\dots, N\}$ so that $\genMetric(x,x_i)\leq \epsilon$.
    The $\epsilon$-covering number $\covering (\epsilon;\entSet,\genMetric)$ \emph{({\strut}respectively the exterior $\epsilon$-covering number $\covering^{\text{ext}} (\epsilon;\entSet,\genMetric)$){\strut}} is the cardinality of a smallest $\epsilon$-covering \emph{(respectively smallest $\epsilon$-net)} for $(\entSet, \genMetric)$. 
\end{definition}
In general, it is hard to obtain precise expressions for covering numbers. One therefore typically resorts to characterizations of their asymptotic behavior as $\epsilon \to 0$.
In \cite{deepAT2019}, where sets of functions are considered, this is done through the concept of optimal exponents. Here, however, we are concerned with sets of systems, which are much more massive and hence require a refined framework for quantifying the asymptotic behavior of their covering numbers. 
Thus, inspired by \cite[Section II.C]{zamesConti}, we use the following notions. 

\begin{definition}[Order, type, and generalized dimension]\label{def:order_type_dim}
 Consider a metric space $(\setX, \genMetric)$ and a compact set $\entSet \subset \setX$. Then, $(\entSet, \genMetric)$ is said to be of order $\order \in \Nplus$ and type $\type \in \Nplus$ if the quantity 
    \begin{equation} \label{eq:gen_dim}
    \genDim{}:=\limsup_{\epsilon \rightarrow 0} \frac{\log ^{(\order+1)} \covering^{\text{ext}} (\epsilon ; \entSet, \genMetric)}{\log ^{\type}\left(\epsilon^{-1}\right)} 
    \end{equation}
    is strictly positive and finite. In this case, we call $\genDim$ the generalized dimension.
\end{definition}

Order, type, and generalized dimension provide measures for the ``description complexity'' of $(\entSet, \genMetric)$
with the order $\order$ the coarsest one. For a given order, the type  $\type$ constitutes a finer measure, and for fixed order and type, the generalized dimension $\genDim$ is the finest measure \cite{zamesConti}.

Whenever the optimal exponent according to \cite[Definition IV.1]{deepAT2019}
is well-defined (i.e., strictly positive and finite), the underlying set has order and type equal to one and generalized dimension equal to the inverse of the optimal exponent (Lemma \ref{lm:gen_dimension_vs_exponent}). Based on this insight, we obtain Table \ref{tbl:nn}, which lists the generalized dimension for the sets considered in \cite[Table I]{deepAT2019}.

Returning to the previous discussion, we are now able to characterize the minimum number of bits required by any decoder to represent $(\entSet, \genMetric)$. 

\begin{lemma}\label{lm:foundamental_limit}
    Consider the metric space $(\setX, \genMetric)$, the compact set $\entSet \subset \setX$ of order $\order$, type $\type$, and generalized dimension $\genDim{}$, and assume that $(\entSet, \genMetric)$ is representable by a decoder $\decoderGen{}$. Then, it holds that 
\begin{equation}\label{eq:approximation_limit}
    \limsup_{\epsilon \rightarrow 0} \frac{\log ^{(\order)} \Ldec{\epsilon}{\decoderGen}{\entSet}{\genMetric} }{\log ^{\type}\left(\epsilon^{-1}\right)} \geq \genDim{}.
\end{equation}
\end{lemma}
\begin{proof}
    See \ref{sec:proof_fundamental_limit}.
\end{proof}
It is natural to say that a decoder $\decoderGen$ is optimal if it satisfies \eqref{eq:approximation_limit} with equality.
\begin{definition}\label{def:metric_entropy_optimal}
 Consider the metric space $(\setX, \genMetric)$ and the compact set $\entSet \subset \setX$ of order $\order$ and type $\type$ with generalized dimension $\genDim{}$. We say that $(\entSet, \genMetric)$ is optimally representable by the decoder $\decoderGen{}$, if $(\entSet, \genMetric)$ is representable by $\decoderGen{}$ and 
\begin{equation}\label{eq:def_metric_entropy_optimal}
    \limsup_{\epsilon \rightarrow 0} \frac{\log ^{(\order)} \Ldec{\epsilon}{\decoderGen}{\entSet}{\genMetric} }{\log ^{\type}(\epsilon^{-1})} = \genDim{}.
\end{equation}
\end{definition}

We now recall a remarkable universal optimality property of ReLU networks, namely 
all the function classes listed in Table \ref{tbl:nn} are optimally representable, in the sense of Definition \ref{def:metric_entropy_optimal}, by the canonical neural network decoder. This  is a simple reformulation of the results in \cite{deepAT2019}; we 
provide the details of this reformulation in \ref{app:compare_to_dennis}. 
In the present paper, we establish that \glspl{rnn} (Definition \ref{def:RNN}), with inner \gls{relu} networks, extend this universality to the approximation of a wide range of nonlinear dynamical systems.

{
\renewcommand{\arraystretch}{1.3}
\begin{table}[h]
\centering
\begin{tabular}{|ll|ll|l|l|l|}
\hline
 & Metric &                                                                         $\mathcal{C}$                          &                                & $\order$ & $\type$ & $\genDim$                                        \\ \hline
     $\{\R\to\R\}$          & $\Ltwo([0,1])$  & $L^2$-Sobolev                                                                                            & $\mathcal{U}(W_2^m([0,1]))$                           & 1     & 1    & $1/m$                                                          \\
       $\{\R\to\R\}$       & $\Ltwo([0,1])$  & H\"older                                                                                                 & $\mathcal{U}(C^\alpha([0,1]))$                        & 1     & 1    & $1/\alpha$                                                     \\
       $\{\R\to\R\}$       & $\Ltwo([0,1])$  & Bump Algebra                                                                                             & $\mathcal{U}(B_{1,1}^1([0,1]))$                       & 1     & 1    & $1$                                                          \\
       $\{\R\to\R\}$       & $\Ltwo([0,1])$  & Bounded Variation                                                                                        & $\mathcal{U}(BV([0,1]))$                              & 1     & 1    & $1$                                                          \\
       $\{\R^d\to\R\}$        & $\Ltwo(\Omega)$  & $L^p$-Sobolev                                         & $\mathcal{U}(W_p^m(\Omega))$                          & 1     & 1    & $\tfrac{d}{m}$                                               \\
        $\{\R^d\to\R\}$       & $\Ltwo(\Omega)$  & Besov                                                & $\mathcal{U}(B_{p,q}^m(\Omega))$                      & 1     & 1    & $\tfrac{d}{m}$                                               \\
        $\{\R^d\to\R\}$       & $\Ltwo(\Omega)$  & Modulation                                                          & $\mathcal{U}(M^s_{p,p}(\R^d))$                        & 1     & 1    & $\tiny{(\frac{1}{p}\!-\!\frac{1}{2}\!+\!\frac{2s}{d})}^{-1}$ \\
       $\{\R^d\to\R\}$        & $\Ltwo(\Omega)$  & Cartoon functions & $\mathcal{E}^{\beta}([-\tfrac{1}{2},\tfrac{1}{2}]^d)$ & 1     & 1    & $\frac{2(d-1)}{\beta}$                                       \\
       \hline
\end{tabular}
\vspace{0.5em}
\caption{Generalized dimension for the sets considered in \cite{deepAT2019}. Here, $\mathcal{U}(X)=\{f \in X:\|f\|_{X} \le 1\}$ denotes the unit ball in the space $X$ and
$\Omega\subseteq\R^d$ is a Lipschitz domain. 
}
\label{tbl:nn}
\end{table}
}

\subsection{Lipschitz fading-memory systems}\label{sec:def_lfms}

We proceed to characterize the class of dynamical systems considered in this paper and start by defining their domain.
\begin{definition}\label{def:signal_set}
    For fixed $\bound>0$, we denote the set of admissible input signals by $\sigSpace \defeq [-\bound, \bound]^{\ZZ}$, that is, for every $x[\cdot] \in \sigSpace$, it holds that 
   $ |x[t]| \leq \bound, \; \forall t \in\ZZ$. 
\end{definition}
The quantity $D>0$ is taken to be fixed throughout the paper and the dependence of $\sigSpace$ on $D$ is not explicitly indicated. 

First, the systems $G: \sigSpace \rightarrow \R^{\ZZ}$ under consideration are causal.
\begin{definition}[Causality]\label{def:causality}
    A system $G: \sigSpace \rightarrow \R^{\ZZ}$ is causal,  if 
    for each $T\in\ZZ$, for every pair $x, x' \in \sigSpace$ with $x[t]=x'[t], \forall t\in \ZZ$ with $t \leq T$, it holds that $(Gx)[T] = (Gx')[T]$.
\end{definition}

Second, we require time-invariance.
\begin{definition}[Time-invariance]\label{def:time_invariant}
    A system $G: \sigSpace \rightarrow \R^{\ZZ}$ is time-invariant, if for every $\timeShift\in\ZZ$, it holds that
        \[
            \shiftOp{\timeShift}Gx = G\shiftOp{\timeShift}x, \qquad \forall x \in \sigSpace,
        \]
        with the shift operator   $\shiftOp{\timeShift}: \R^{\ZZ} \rightarrow \R^{\ZZ}$  defined as
        $
            (\shiftOp{\timeShift}x)[t] = x[t - \timeShift].
        $
\end{definition}

Next, we follow Volterra, who suggested that \cite[p. 188]{voltera59} ``a first extremely natural postulate is to suppose that the influence of the [input] a long time before the given moment gradually fades out.'' This property was termed ``fading memory'' in \cite{boyd1985fading}, and here we introduce a more quantitative version thereof, namely the concept of ``Lipschitz fading memory''.
The concept is inspired by examples in \cite{lakshmikantham1995theory, alves2014superdiffusion, moura2016transient, nagel2022asset}, which will be discussed in more detail later. 

\begin{definition}[\acrlong{lfm}]
\label{def:lip_fading_memory}
    We say that $(w[t])_{t\in \Nzero}$ is a weight sequence if it is non-increasing and satisfies $ w[t] \in (0, 1], \forall t\in \Nzero$, and $\lim_{t\rightarrow\infty} w[t] = 0$.  A system $G: \sigSpace \rightarrow \R^{\ZZ}$ has \acrlong{lfm} (\glsfirst{lfm}) with respect to the weight sequence $w$ if
        \[
            |(Gx)[t] - (Gy)[t]| \leq \sup_{\timeShift\geq 0} w[\tau]|(x[t-\timeShift]-y[t-\timeShift])|, \quad \forall t \in \ZZ, \; \forall x, y \in \sigSpace.
        \]
\end{definition}

The class of \glsfirst{lfm} systems considered in the remainder of the paper can now formally be defined as follows. 

\begin{definition}[\acrlong{lfm} systems]
\label{def:lfm_sys_set}
    Given a weight sequence $w[\cdot]$, we define
        \begin{align}
        \LFMset{w} = \{G: \sigSpace \rightarrow \R^{\ZZ} \mid &G \text{ is causal, time-invariant, has \acrlong{lfm}} \notag\\
        &\text{ w.r.t. }  w,\text{ and satisfies } (G 0)[t] = 0, \; \forall t \in \ZZ \label{eq:0_in_0_out}
         \}.
        \end{align}
\end{definition}

As we will want to approximate \gls{lfm} systems $\exampleSystem \in \LFMset{w}$ by \glspl{rnn}, we need a metric that quantifies approximation quality. This metric should take into account that the RNNs we consider start running at time $t=0$ and will, moreover, be of worst-case nature.

\begin{definition}\label{def:metric}
    Let $\sigSpaceplus \defeq \{s\in\sigSpace \hspace{-2pt} \mid \hspace{-2pt}  s[t]=0, \forall t < 0\}$. 
    For $G, G' \in \systemsAmbientSpace$,
    we define the metric
    \[
            \opMetric(G, G') = \sup_{x \in \sigSpaceplus} \sup_{t\in \Nzero} |(Gx)[t] - (G'x)[t]|.
    \]
\end{definition}
We hasten to add that the restriction to one-sided input signals in Definition \ref{def:metric} and to taking the supremum over $t\in \Nzero$ 
in the output signals does not impact the hardness of the approximation task, as shown by the next result.
\begin{lemma}\label{lem:sup_over_N_same}
Let $(w[t])_{t\in \Nzero}$ be a weight sequence.
    For $G, G' \in \LFMset{w}$, we have 
    \[
        \opMetric({G, G'}) = \sup_{x \in \sigSpace} \sup_{t\in\ZZ} |(Gx)[t] - (G'x)[t]|.
    \]
\end{lemma}
\begin{proof}
    See \ref{app:proof_sub_over_N_same}.
\end{proof}

We are now ready to formally state the main goal of this paper, which is to prove that $(\LFMset{w}, \opMetric)$ is optimally representable by the canonical \gls{rnn} decoder in Definition \ref{def:canonical_rnn_decoder}. In fact, we will be seeking a quantitative version of this statement comparing the description complexity of the class $\LFMset{w}$ to that of the RNNs approximating it.

\section{Metric entropy of LFM systems}\label{sec:app_rate_lfms}
In this section, we study the ($\epsilon$-)scaling behavior of $\coveringExt(\epsilon; \LFMset{w}, \opMetric )$ for general weight sequences $w$. This will be effected by deriving an upper bound on $\coveringExt(\epsilon; \LFMset{w}, \opMetric )$ through the construction of a covering and a matching (in terms of scaling behavior) lower bound by identifying an explicit packing. We first define the concept of packings.

\begin{definition}
\label{def:packing}
Let $(\setX,\genMetric)$ be a metric space and $\entSet \subset \setX$ compact. An $\epsilon$-packing for $(\entSet, \genMetric)$ is a set $\{x_1,x_2,\dots,x_N\}\subset \entSet$ such that $\genMetric(x_i,x_j) > \epsilon$, for all distinct $i,j$. The $\epsilon$-packing number $\packing (\epsilon;\entSet,\genMetric)$ is the cardinality of a largest $\epsilon$-packing for $(\entSet, \genMetric)$.
\end{definition}
We shall frequently make use of the following two results relating the packing, covering, and exterior covering numbers.
\begin{lemma}
[\hspace{1sp}\cite{kolmogorov1959varepsilon}, Theorem IV] 
\label{lm:number_relation}
Let $(\setX, \genMetric)$ be a metric space and $\entSet \subset \setX$ compact. For all $\epsilon >0$, we have
\begin{equation}\label{eq:number_relation}
    \packing (2\epsilon;\entSet,\genMetric) \leq \covering^{\text{ext}} (\epsilon;\entSet,\genMetric) \leq \covering (\epsilon;\entSet,\genMetric) \leq \packing (\epsilon;\entSet,\genMetric).
\end{equation}
\end{lemma}
\begin{lemma}[\hspace{1sp}\cite{kolmogorov1959varepsilon}, p. 93]\label{lm:ent_inv} 
Let $(\setX, \genMetric_{\setX})$ and $(\mathcal{Y}, \genMetric_{\mathcal{Y}})$ be metric spaces and consider the compact sets $\entSet_{\setX}\subset \setX$ and $\entSet_{\mathcal{Y}}\subset \mathcal{Y}$. Assume that there exists an isometric isomorphism $f : \entSet_{\setX} \rightarrow \entSet_{\mathcal{Y}}$, i.e., $f$ is bijective and for every pair $a,b\in \entSet_{\setX}$, one has $\genMetric_{\mathcal{Y}}(f(a),f(b))=\genMetric_{\setX}(a,b)$. Then, 
\begin{align}
    \covering(\epsilon;\entSet_{\setX},\genMetric_{\setX}) =\covering(\epsilon;\entSet_{\mathcal{Y}},\genMetric_{\mathcal{Y}}) \quad \textrm{and} \quad
    \packing(\epsilon;\entSet_{\setX},\genMetric_{\setX}) =\packing(\epsilon;\entSet_{\mathcal{Y}},\genMetric_{\mathcal{Y}}).
\end{align}
\end{lemma}

Lemma \ref{lm:ent_inv} will allow us to work with a simplified metric space $(\LFMsetZero{w}, \metricZero)$ instead of the original one $(\LFMset{w}, \opMetric)$. Concretely, we exploit the properties of \gls{lfm} systems to effect this reduction as follows. 
First, as \gls{lfm} systems are causal, their output at time $t$ depends on the history of inputs up to and including time $t$ only. Second, time-invariance implies that the mapping taking the history of the input signal to the current output at time $t$ does not change with $t$ and we can therefore restrict ourselves to $t=0$ w.l.o.g. Thus, the mapping
realized by an \gls{lfm} system is completely characterized by the response to signals in the set
\begin{equation}
\label{eq:sig_space_minus}
\sigSpaceminus \defeq \{s\in\sigSpace \mid \forall \ell \in \Nplus: s[\ell]=0\}.
\end{equation}
Using this insight, the simplified metric space can be defined as follows.
\begin{definition}
Let $(w[t])_{t\in \Nzero}$ be a weight sequence. We define the metric space
$(\LFMsetZero{w}, \metricZero)$  according to 
\begin{align}
    \LFMsetZero{w} = & \, \{g: \sigSpaceminus \rightarrow \R \mid 
     |g(x) - g(x')| \leq \norm{x-x'}_w, \; \forall x,x' \in \sigSpaceminus, 
    g(0) = 0 
     \},\label{eq:nonLinSpace} \\
    &\textrm{where } \quad \norm{x-x'}_w \coloneqq \sup_{t\in \Nzero}  w[t]|(x[-t] -x'[-t])| \label{eq:w_norm}.
\end{align}
The metric $\metricZero$ is given by
\begin{equation}\label{eq:def_zero_metric}
\metricZero(g, g') = \sup_{x\in\sigSpaceminus} |g(x) - g'(x)|, \qquad
\textrm{for }g, g' \in \{\sigSpaceminus \to \R\}.
\end{equation} 
\end{definition}
 Next, we define the projection operator $\projLeftSide: \sigSpace \rightarrow \sigSpaceminus$ as
    \[
        (\projLeftSide x)[t] = x[t]\cdot \indArg{t \leq 0}
    \]
and formalize the isometric isomorphism between functionals $g \in \LFMsetZero{w}$ and systems $G \in \LFMset{w}$ as follows.
\begin{lemma}\label{lem:g0_g_isomorphism}
\letWbeWeightSequence{}
    The map
\begin{align}\label{eq:isometry_G0_2_G1}
    \isometry : \LFMsetZero{w} &\rightarrow \LFMset{w} \\
    g & \rightarrow G \coloneqq (x \rightarrow \{g(\projLeftSide\shiftOp{-t} x)\}_{t\in\ZZ})\label{eq:isometry}
\end{align}
    is an isometric isomorphism between $(\LFMsetZero{w}, \metricZero)$ and $(\LFMset{w}, \opMetric)$.
Furthermore, $\covering(\epsilon; \LFMsetZero{w}, \metricZero) = \covering(\epsilon; \LFMset{w}, \opMetric)$ and $\packing(\epsilon; \LFMsetZero{w}, \metricZero) = \packing(\epsilon; \LFMset{w}, \opMetric)$, for all $\epsilon>0$.
\end{lemma}
\begin{proof}See \ref{app:proof_g0_g_isomorphism}.\end{proof}

In the remainder of this section, we first lower-bound $\packing(\epsilon; \LFMsetZero{w}, \metricZero)$, then upper-bound $\covering(\epsilon; \LFMsetZero{w}, \metricZero)$, and finally use Lemmata \ref{lm:ent_inv} and \ref{lem:g0_g_isomorphism} to translate these bounds into bounds on $\coveringExt(\epsilon; \LFMset{w}, \opMetric)$. The lower bound is established as follows.
\begin{lemma}\label{lm:packing_number}
\letWbeWeightSequence{}
The $\epsilon$-packing number of $(\LFMsetZero{w}, \metricZero)$ satisfies
    \[
     \log \packing(\epsilon; \LFMsetZero{w}, \metricZero) \geq
        \left (\prod_{\ell=0}^T \ceil{\frac{2\bound w[\ell]}{{\epsilon}}} \right ) - 1 ,
    \]
    with $T := \max \{ T' \in \Nzero \mid w[T']> \frac{\epsilon}{2\bound }\}$.
\end{lemma}
\begin{proof}
    The proof is taken from \cite{kolmogorov1959varepsilon} and is detailed, for completeness, in \ref{app:proof_packing_number}.
\end{proof}

To upper-bound $\covering(\epsilon; \LFMsetZero{w}, \metricZero)$, we construct an $\epsilon$-net for $(\LFMsetZero{w}, \metricZero)$. This construction is again inspired by \cite{kolmogorov1959varepsilon} but we need to modify it to ensure that the elements of the $\epsilon$-net can efficiently be realized by \gls{relu} networks. To be specific, we employ piecewise linear mappings to approximate \gls{lfm} systems instead of piecewise constant mappings as considered in \cite{kolmogorov1959varepsilon}; this requires significant adjustments to the proof in \cite[Section 7.2]{kolmogorov1959varepsilon}.

\textcolor{Iter1Color}{We start by introducing the ``spike'' function $\phi: \R^{d} \rightarrow \R$  considered in \cite{karnik2024neural}, and defined as
\begin{equation}\label{eq:phi_def}
    \phi(z) = \max \{1 + \min\{z_1,\dots,z_d,0\} - \max\{z_1,\dots,z_d,0\}, 0\}.
\end{equation}
An illustration of spike functions for $d=1$ and $d=2$ is provided in Figure \ref{fig:spike}. }The idea underlying the construction of such ``spike'' functions can be traced back to \cite{yarotsky2018optimal}, where the convex set
\begin{equation}\label{eq:simplexes}
\mathcal{T}\defeq\{z\in\R^d\mid \max\{z_1,\dots,z_d,0\} - \min\{z_1,\dots,z_d,0\} \leq 1\}
\end{equation}
is considered and shown to be the union of $(d+1)!$ simplices, each having $0$ as a vertex, given by
\begin{equation}\label{eq:simplexes_exs}
    \{z\in \mathcal{T}\mid~ z_{\permutation(0)}\leq z_{\permutation(1)}\leq \dots\leq z_{\permutation(d)}\},
\end{equation}
where $\permutation$ is a permutation 
of the integers $0,1,\dots, d$ and $z_0\defeq 0$. In \cite{yarotsky2018optimal} this result is employed to approximate continuous functions mapping $\R^d$ to $\R$ by functions that are piecewise linear on the simplices (\ref{eq:simplexes_exs}).

We remark that the spike function \eqref{eq:phi_def} is a composition of affine functions and min/max functions, which, as shown in Section \ref{sec:RNN2G}, renders it uniquely
suitable for realization through \gls{relu} networks. 
To be specific, in Lemmata \ref{lm:NN2minmax} and \ref{lm:NN2phi}, we provide concrete constructions of spike functions using \gls{relu} networks. 
To the best of our knowledge, these constructions are novel.

\begin{figure}[ht]
    \centering
    \begin{subfigure}[b]{0.45\textwidth}
        \centering
        \includegraphics[width=\textwidth]{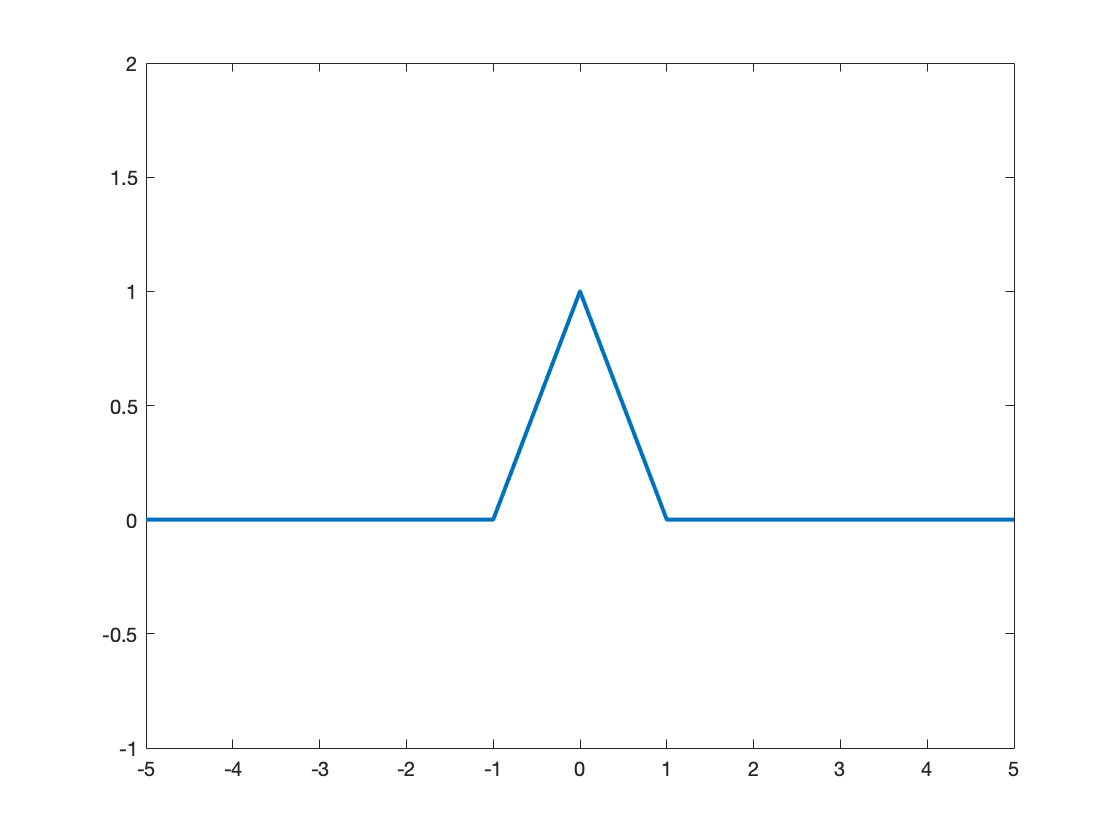} 
        \caption{$d=1$}
        \label{fig:subfig1}
    \end{subfigure}
    \hfill
    \begin{subfigure}[b]{0.45\textwidth}
        \centering
        \includegraphics[width=\textwidth]{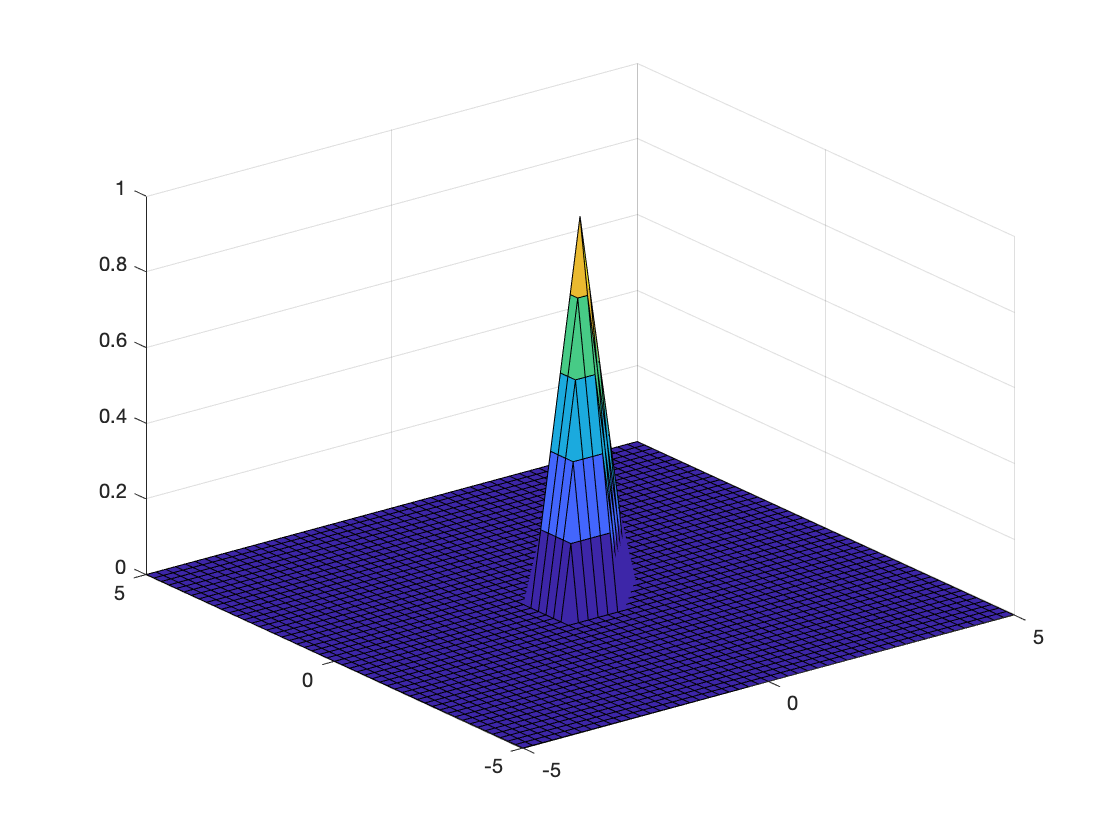}
        \caption{$d=2$}
        \label{fig:subfig2}
    \end{subfigure}
    \caption{Spike functions in dimensions $1$ and $2$.}
    \label{fig:spike}
\end{figure}

We proceed to show how a \gls{pou} can be realized as a weighted linear combination of shifted spike functions, a property that will be of key importance in the RNN constructions described in Section \ref{sec:RNN2G}.


\begin{lemma}\label{lm:partition_of_unity}
Consider the spike function 
\begin{equation}\label{eq:phi_def_lm}
    \phi(z) = \max \{1 + \min\{z_1,\dots,z_d,0\} - \max\{z_1,\dots,z_d,0\}, 0\},
\end{equation}
the lattice 
\begin{equation}\label{eq:lattice_def_lm}
    \lattice = \dsidedEnum{M_1}\times \cdots \times \dsidedEnum{M_d} \subset \R^{d}, \text{ where }M_{\ell}\in \Nplus \text{, for  }\ell\in \{1,2,\dots,d\},
\end{equation}
and the set 
\begin{equation}\label{eq:POU_lm}
    \Xi \defeq\{\phi(\cdot-\vecindex)\}_{\vecindex\in \lattice}.
\end{equation}
Then, $\Xi$ forms a \gls{pou} on $\prod_{\ell =1}^{d}[-M_{\ell},M_{\ell}]$, i.e.,
\begin{enumerate}
    \item[(i)] $0\leq \phi(z-\vecindex)\leq 1$, for $z\in \R^d$ and $\vecindex\in \lattice$;
    \item[(ii)] $\phi(\cdot-\vecindex)$ is compactly supported, specifically $\operatorname{supp}(\phi(\cdot-\vecindex)) \subset \vecindex+[-1,1]^{d}$; 
    \item[(iii)] it holds that
    \begin{equation*}
        \sum_{\vecindex\in \lattice} \phi(z-\vecindex) = 1,\quad \text{for }z\in \prod_{\ell =1}^{d}[-M_{\ell},M_{\ell}].
    \end{equation*}
\end{enumerate}
\end{lemma}
\begin{proof}
To prove (i), we note that $\phi(z) \geq 0$ by definition and 
$$1 + \min\{z_1,\dots,z_d,0\} - \max\{z_1,\dots,z_d,0\}\leq 1.$$

For (ii), it suffices to show that
\begin{equation*}
    \phi(z) =0, \quad \text{for } z\in \R^d \setminus [-1,1]^{d}.
\end{equation*}
To this end, we pick $z\in \R^d \setminus [-1,1]^{d}$ arbitrarily and fix an arbitrary $\ell \in \{1,2,\dots,d\}$ such that $|z_{\ell}|>1$. Assume that $z_{\ell}>1$. (The case $z_{\ell}<-1$ follows similarly.) Then,
$$1 + \min\{z_1,\dots,z_d,0\} - \max\{z_1,\dots,z_d,0\}\leq 1+0-z_{\ell} <0,$$
which by \eqref{eq:phi_def_lm} implies $\phi(z)=0$.

We proceed to prove (iii). Since $\operatorname{supp}(\phi(\cdot-\vecindex)) \subset \vecindex+[-1,1]^{d}$, we have
\begin{equation*}
        \sum_{\vecindex\in \lattice} \phi(z-\vecindex) = \sum_{\vecindex\in\prod_{\ell=1}^{d} \{\lfloor z_{\ell}\rfloor, \lceil z_{\ell} \rceil\}} \phi(z-\vecindex), \quad \text{for }z\in \prod_{\ell =1}^{d}[-M_{\ell},M_{\ell}]. 
\end{equation*}
Defining $\bar{z}\in \R^d$ according to $\bar{z}_\ell = \floor{z_\ell}$, for $\ell=1,2,\dots, d$, and noting that $\phi(z-\vecindex) = \phi((z-\bar{z})-(\vecindex-\bar{z}))$, it suffices to show that 
\begin{equation*}
        \sum_{\vecindex\in\{0,1\}^d} \phi(z-\vecindex) = 1,\quad \text{for }  z\in [0,1]^d \text{ and } \vecindex\in \{0,1\}^d.
\end{equation*}
As $\min\{x_1,\dots,x_d\}$ and $\max\{x_1,\dots,x_d\}$ are permutation-invariant, so is $\phi$ by \eqref{eq:phi_def_lm}. In what follows, we can therefore assume, w.l.o.g., that $z_1\geq z_2 \geq \dots\geq z_d$.

Now, with $e_k$ the $k$-th unit vector in $\R^d$, let
\begin{equation*}
    A \defeq \left\{0,e_1,e_1+e_2,\dots, \sum_{i=1}^{k}e_i,\dots, \sum_{i=1}^{d}e_i\right\}\subset \{0,1\}^d.
\end{equation*}
We claim that 
\begin{equation}
    \phi(z-\vecindex)=0,\quad \text{for } z\in[0,1]^d \text{ and }  \vecindex\in \{0,1\}^d\setminus A.
\end{equation}
This can be verified as follows. First, thanks to
\begin{equation*}
\vecindex \in \{0,1\}^d\setminus A ~\Leftrightarrow ~\exists~ i,j\in\{1,2,\dots, d\},  i<j, \text{ s.t. }\vecindex_i=0, \vecindex_j=1,
\end{equation*}
we get for $z\in[0,1]^d $ and $\vecindex\in \{0,1\}^d\setminus A$, 
\begin{align*}
    &\min\{z_1-\vecindex_1, z_2-\vecindex_2, \dots,z_d-\vecindex_d, 0\} \leq z_j-\vecindex_j=z_j-1,\\
    &\max\{z_1-\vecindex_1, z_2-\vecindex_2, \dots,z_d-\vecindex_d, 0\} \geq z_i-\vecindex_i=z_i,\\
    &\Rightarrow  1+\min\{z_1-\vecindex_1, z_2-\vecindex_2, \dots,z_d-\vecindex_d, 0\} -\max\{z_1-\vecindex_1, z_2-\vecindex_2, \dots,z_d-\vecindex_d, 0\}\leq 0,
\end{align*}
and hence $\phi(z-\vecindex)=0$.
It thus suffices to show that 
\begin{equation}\label{eq:spike_sum_proof}
        \sum_{\vecindex\in A} \phi(z-\vecindex) = 1, \quad \text{for }z\in [0,1]^d.
\end{equation}
Now, a direct calculation yields, for $z\in [0,1]^d$, \begin{align}
   & \phi(z) = 1-z_1,\label{eq:spike_sum1}\\
    &\phi\left(z-\sum_{i=1}^k e_i\right) = z_{k}-z_{k+1},\quad \text{for } k=1,\dots,d-1,\label{eq:spike_sum2}\\
    &\phi\left(z-\sum_{i=1}^d e_i\right) = z_{d}.\label{eq:spike_sum3}
\end{align}
Summing \eqref{eq:spike_sum1}-\eqref{eq:spike_sum3} results in \eqref{eq:spike_sum_proof}.
\end{proof}

Next, we establish a traversal property of the lattice $\lattice$.
\begin{definition}[Regular path]
    For every $d\in \Nplus$, $M_{\ell} \in \Nplus,~\ell\in\{1,\dots,d\}$, and corresponding lattice $\lattice = \dsidedEnum{M_1}\times \cdots \times \dsidedEnum{M_d}$, we call a path of lattice points $\vecindex_1\leftrightarrow \vecindex_2 \leftrightarrow \dots \leftrightarrow \vecindex_{|\lattice|}$ regular for $\lattice$ if 
    \begin{enumerate}
        \item[(i)] the path visits each lattice point in $\lattice$ exactly once, 
        \item[(ii)] $\vecindex_{i+1}$ and $\vecindex_i$, for each $i=1,\dots, |\lattice|-1$, differ in exactly one position, specifically by $+1$ or $-1$.
    \end{enumerate}
\end{definition}

\begin{lemma}\label{lm:lattice_traversal}
    For every $d\in \Nplus$, $M_{\ell} \in \Nplus,~\ell\in\{1,\dots,d\}$, and corresponding lattice $\lattice = \dsidedEnum{M_1}\times \cdots \times \dsidedEnum{M_d}$, there exists a regular path $\vecindex_1\leftrightarrow \vecindex_2 \leftrightarrow \dots \leftrightarrow \vecindex_{|\lattice|}$ for $|\lattice|$.
\end{lemma}

\begin{proof}
    We write $\lattice_d$ for the $d$-dimensional lattice $\dsidedEnum{M_1}\times \cdots \times \dsidedEnum{M_d}$ to emphasize the dependence on the dimension $d$ and prove the statement by induction over $d$. The base case $d=1$ follows by simply considering the path $-M_1 \leftrightarrow-M_1+1\leftrightarrow \dots \leftrightarrow 0 \leftrightarrow \dots\leftrightarrow M_1-1\leftrightarrow M_1$ for lattice $\lattice_1$.
    Assume now that the statement holds for $d=k$, i.e., there exists a path 
    $\vecindex_1 \leftrightarrow \dots \leftrightarrow \vecindex_{|\lattice_k|}$ that is regular for lattice $\lattice_k$.
    Then, for $d=k+1$, 
    the path $(\vecindex_1,-M_{k+1})\leftrightarrow \dots \leftrightarrow (\vecindex_{|\lattice_k|},-M_{k+1}) \leftrightarrow(\vecindex_{|\lattice_k|},-M_{k+1}+1)\leftrightarrow \dots \leftrightarrow (\vecindex_1,-M_{k+1}+1) \leftrightarrow(\vecindex_1,-M_{k+1}+2)\leftrightarrow \dots \leftrightarrow (\vecindex_{|\lattice_k|},-M_{k+1}+2) \leftrightarrow \dots$
    is regular for lattice $\lattice_{k+1}$.
\end{proof}

We are now ready to describe our construction of the $\epsilon$-net for $(\LFMsetZero{w},\metricZero)$. 
In fact, we shall specify not only the elements of the $\epsilon$-net, but also the mapping taking a given
functional $g \in \LFMsetZero{w}$ to an $\epsilon$-close element of the net. Counting the number of ball centers needed to
ensure that every $g \in \LFMsetZero{w}$ is in the $\epsilon$-vicinity of some ball center, then yields the cardinality of the $\epsilon$-net. 
The mapping proceeds by constructing a functional $\tilde{g}$ which is approximately faithful with respect to $g$ on amplitude-discretized and time-truncated input signals. 
The formal result is as follows.

\begin{lemma}\label{lm:covering_spike}
For every \textcolor{Iter2Color}{$\epsilon>0$ and $s\geq 1$}, let
\begin{align*}
    T&\defeq \max\left\{ \ell \in \Nzero \Bigm| w[\ell] > \frac{\epsilon}{\bound} \frac{s}{s+1} \right\},\\
    \delta_{\ell} &\defeq \frac{s}{s+1} \frac{\epsilon}{w[\ell]},\qquad\forall \ell \in \ssidedEnum{T},\\
    \latticeBound_{\ell} & \defeq \ceil{\frac{\bound}{\delta_{\ell}}} ,\quad \forall\ell \in \ssidedEnum{T},
\end{align*}
and define the mapping 
\begin{align*}
    k&: S_{-} \rightarrow \R^{T+1}\\
    k_{\ell}(x)& = \frac{x[-\ell]}{\delta_{\ell}}, \qquad \forall \ell \in \ssidedEnum{T}.
\end{align*}
Furthermore, consider the lattice 
$$\latticeN \defeq \dsidedEnum{N_0}\times \cdots\times \dsidedEnum{N_T} \subset \R^{T+1}.$$
Then, there exists a set $\coveringset_0 \subset \R^{|\latticeN |}$ with $|\coveringset_0 | = \left(2\left\lceil\frac{s+1}{2}\right\rceil+1\right)^{|\latticeN |-1}$
such that
$$\coveringset \defeq 
\left\{\tilde{g}: S_- \rightarrow \R \biggm| \tilde{g}(x) = \sum_{\vecindex\in\latticeN } \hat{g}_{\vecindex} \phi(k(x)-\vecindex), ~\{\hat{g}_{\vecindex}\}_{\vecindex\in\latticeN } \in \coveringset_0 \right\}$$
is an $\epsilon$-net
for $(\LFMsetZero{w},\metricZero)$. Furthermore, it holds that
$$|\coveringset| = \left(2\left\lceil\frac{s+1}{2}\right\rceil+1\right)^{|\latticeN |-1} = \left(2\left\lceil\frac{s+1}{2}\right\rceil+1\right)^{\left[\prod_{\ell=0}^{T}\left(2\ceil{\frac{\bound w[\ell]}{\epsilon}  \frac{s+1}{s}} + 1\right)\right]-1}.$$
\end{lemma}
\begin{proof}
The proof proceeds in three steps.

\textit{Step 1:} We pick amplitude-discretized and time-truncated signals from the set $\sigSpaceminus$ on the lattice $\latticeN$ and approximately interpolate the functional $g \in \LFMsetZero{w}$ on these signals, using the elements of the \gls{pou} $\Xi=\{\phi(\cdot-\vecindex)\}_{\vecindex\in \latticeN}$ as interpolation basis functions, to get a new functional $\hat{g}$.

Fix $g \in \LFMsetZero{w}$ arbitrarily. For each $\vecindex\in\latticeN $, define $\xhat_{\vecindex} \in \sigSpaceminus$ according to
\begin{equation}\label{eq:sig_lattice}
\xhat_{\vecindex}[-\ell] =
\begin{cases}
    \delta_\ell \vecindex_{\ell}, & \text{ if } \ell \in \ssidedEnum{T} \\
    0, & \text{ if } \ell \in \ZZ\setminus\ssidedEnum{T}
\end{cases}
\end{equation}
and let $\hat{g}: S_- \rightarrow \R$ be given by
\begin{equation}\label{eq:hat_g}
    \hat{g}(x) = \sum_{\vecindex\in\latticeN } g(\xhat_{\vecindex}) \phi(k(x)-\vecindex).
\end{equation}
We first prove that 
\begin{equation}\label{eq:hat_g2g}
    |\hat{g}(x) - g(x)|\leq \frac{s\epsilon}{s+1}, \quad \text{for }x\in \sigSpaceminus.
\end{equation}
By Lemma \ref{lm:partition_of_unity}, $\{\phi(\cdot-\vecindex)\}_{\vecindex\in\latticeN }$ constitutes a \gls{pou} on $\prod_{\ell =0}^{T}[-N_{\ell},N_{\ell}]$, and hence
\begin{equation}\label{eq:kx_vanish}
    \phi(k(x)-\vecindex)=0, \quad \text{for all } x\in \sigSpaceminus \text{ s.t. }\|k(x)-\vecindex\|_{\infty}> 1.
\end{equation}
Furthermore, for $x\in S_-$ and $\vecindex \in \latticeN$ such that $\|k(x)-\vecindex\|_{\infty}\leq 1 $, we have
\begin{equation*}
     \left|\frac{x[-\ell]}{\delta_{\ell}} - \vecindex_{\ell}\right| \leq 1 ,\quad \text{for } \ell \in \ssidedEnum{T},\nonumber
\end{equation*}
which gives
\begin{align*}
    w[\ell]\left|x[-\ell] - \xhat_{\vecindex}[-\ell]\right| &\leq w[\ell] \delta_{\ell} = \frac{s\epsilon}{s+1}, \qquad ~ \text{for } \ell \in \ssidedEnum{T}, \\
    w[\ell]\left|x[-\ell] - \xhat_{\vecindex}[-\ell]\right| &=  w[\ell]\left|x[-\ell]\right|\\
    &\leq \frac{\epsilon}{D}\frac{s}{s+1} D = \frac{s\epsilon}{s+1}, \quad\text{for } \ell > T,
\end{align*}
and therefore
\begin{equation} \label{eq:x_xn_distance}
    \|x-\xhat_{\vecindex}\|_w \leq \frac{s\epsilon}{s+1}.
\end{equation}

Hence, we can bound the approximation error $|g(x) - \hat{g}(x)|$, for $x\in \sigSpaceminus$, according to
\begin{align}
    |g(x) - \hat{g}(x)| 
    \overset{\eqref{eq:hat_g}, \text{ \gls{pou}}}&{=} \left|\sum_{\vecindex\in\latticeN } (g(x)- g(\xhat_{\vecindex})) \phi(k(x)-\vecindex )\right| \label{eq:ghat2g_2}\\
    \overset{\eqref{eq:kx_vanish}}&{=} \left|\sum_{\|k(x)-\vecindex \|_{\infty}\leq 1} (g(x)- g(\xhat_{\vecindex})) \phi(k(x)-\vecindex )\right|\label{eq:ghat2g_3}\\
     \overset{(\phi\geq 0)}&{\leq} \sum_{\|k(x)-\vecindex \|_{\infty}\leq 1}\left| (g(x)- g(\xhat_{\vecindex}))\right| \phi(k(x)-\vecindex )\label{eq:ghat2g_4}\\
     \overset{\text{Lipschitz}}&{\leq} \sum_{\|k(x)-\vecindex \|_{\infty}\leq 1}\left\| x- \xhat_{\vecindex}\right\|_w \phi(k(x)-\vecindex ) \label{eq:ghat2g_5}\\
     \overset{\eqref{eq:x_xn_distance}}&{\leq} \frac{s\epsilon}{s+1}\sum_{\|k(x)-\vecindex \|_{\infty}\leq 1} \phi(k(x)-\vecindex ) \label{eq:ghat2g_6}\\
     \overset{\eqref{eq:kx_vanish}}&{=} \frac{s\epsilon}{s+1}\sum_{\vecindex\in\latticeN } \phi(k(x)-\vecindex ) \label{eq:ghat2g_7}\\
    \overset{\text{\gls{pou}}}&{=} \frac{s\epsilon}{s+1}.\label{eq:ghat2g_8}
\end{align}

\textit{Step 2:} We modify the interpolation weights $g(\xhat_{\vecindex})$ in $\hat{g}$ to weights $\hat{g}_{\vecindex}$, for every $\vecindex\in\latticeN$, along a fixed path of $\latticeN$, to get a new functional $\tilde{g}$ that approximates $g$ to within an error of at most $\epsilon$.

Specifically, we construct, by induction, a functional $\tilde{g}$ such that 
\begin{equation}
    |\tilde{g}(x)-\hat{g}(x)|\leq \Delta \defeq \frac{\epsilon}{s+1}.
\end{equation}
To this end, we let
\begin{equation}\label{eq:tilde_g_def}
    \tilde{g}(x) \defeq \sum_{\vecindex\in\latticeN } \hat{g}_{\vecindex} \phi(k(x)-\vecindex )
\end{equation}
and then find, for every $\vecindex\in\latticeN$, a value $\hat{g}_{\vecindex}$ so that 
\begin{equation}\label{eq:store_value}
    |g(\xhat_{\vecindex})- \hat{g}_{\vecindex}| \leq \Delta,
\end{equation}
which, in turn, yields
\begin{align*}
    |\hat{g}(x) - \tilde{g}(x)| &= \left| \sum_{\vecindex\in\latticeN } (g(\xhat_{\vecindex})- \hat{g}_{\vecindex}) \phi(k(x)-\vecindex ) \right|\\
    & \leq \sum_{\vecindex\in\latticeN } \left|g(\xhat_{\vecindex})- \hat{g}_{\vecindex}\right| \phi(k(x)-\vecindex ) \\
    & \leq \Delta \sum_{\vecindex\in\latticeN }  \phi(k(x)-\vecindex )\\
     \overset{\text{\gls{pou}}}&{=} \Delta.
\end{align*}
Therefore,
\begin{equation}\label{eq:covering_gtilde2g}
    |\tilde{g}(x)-g(x)| \leq |\tilde{g}(x)-\hat{g}(x)| + |\hat{g}(x)-g(x)| \leq \Delta + \frac{s\epsilon}{s+1} = \epsilon.
\end{equation}
It remains to specify how the values $\hat{g}_{\vecindex}$, $\vecindex\in\latticeN$, can be obtained such that \eqref{eq:store_value} holds. This will be done by performing the mapping from $g(\xhat_{\vecindex})$ to $\hat{g}_{\vecindex}$ along a judiciously chosen path $\vecindex_1\leftrightarrow \vecindex_2 \leftrightarrow \dots \leftrightarrow \vecindex_{|\latticeN|}$ that is regular for $\latticeN$ (Lemma \ref{lm:lattice_traversal}). To this end, we distinguish two cases.
\begin{itemize}
    \item \textit{Case 1: The path $\vecindex_1\leftrightarrow \vecindex_2 \leftrightarrow \dots \leftrightarrow \vecindex_{|\latticeN|}$ starts or ends at $0\in \latticeN$.} \\In particular, we assume, w.l.o.g., that $\vecindex_1=0$. Next, we find the values $\hat{g}_{\vecindex_k}$ satisfying \eqref{eq:store_value} inductively over the index $k=1,2\dots,|\latticeN|$. The base case $k=1$ is immediate as we can simply set $\hat{g}_{\vecindex_{1}} = 0$ and thereby obtain $$|g(\xhat_{\vecindex_{1}})- \hat{g}_{\vecindex_{1}}| \overset{\eqref{eq:nonLinSpace}}{=} |g(0)-0| = 0 \leq \Delta.$$
Now, assume that \eqref{eq:store_value} holds for some $\vecindex_k$, $k\in \{1,2,\dots,|\latticeN|\}$, on the path $\vecindex_{1} \rightarrow \dots \rightarrow \vecindex_{|\latticeN|}$. Then, by Lemma \ref{lm:lattice_traversal}, $\vecindex_{k+1}$ differs in exactly one position from $\vecindex_k$, namely by $+1$ or $-1$, which, thanks to \eqref{eq:w_norm} and \eqref{eq:sig_lattice}, yields $\left\|\widehat{x}_{\vecindex_{k+1}}-\widehat{x}_{\vecindex_k}\right\|_{w} \leq \frac{\epsilon s}{s+1}$. Upon noting that 
\begin{align*}
\left|g\left(\xhat_{\vecindex_{k+1}}\right)-\hat{g}_{\vecindex_k}\right| &=\left|\underbrace{g\left(\xhat_{\vecindex_{k+1}}\right)-g\left(\widehat{x}_{\vecindex_k}\right)}_{\text{Lipschitz}}+\underbrace{g\left(\widehat{x}_{\vecindex_k}\right)-\hat{g}_{\vecindex_k}}_{\eqref{eq:store_value}}\right| \\ & \leq\left\|\widehat{x}_{\vecindex_{k+1}}-\widehat{x}_{\vecindex_k}\right\|_{w}+\Delta \\ & \leq  \frac{\epsilon s}{s+1}+\Delta=(s+1) \Delta,
\end{align*}
we can conclude that there exists an 
\begin{equation}\label{eq:deviation_store_g}
\begin{aligned}
    &~m\in\left\{ -2\left\lceil\frac{s+1}{2}\right\rceil, -2\left\lceil\frac{s+1}{2}\right\rceil+2, \dots, 2\left\lceil\frac{s+1}{2}\right\rceil \right\},\\
    &\text{s.t. }~ \hat{g}_{\vecindex_{k+1}} = \hat{g}_{\vecindex_k}+ m\Delta  ~\text{ and }~ \left|\hat{g}_{\vecindex_{k+1}}-g(\xhat_{\vecindex_{k+1}})\right|\leq \Delta.
\end{aligned}
\end{equation}
This completes the induction. 
    \item \textit{Case 2: The path $\vecindex_1\leftrightarrow \vecindex_2 \leftrightarrow \dots \leftrightarrow \vecindex_{|\latticeN|}$ does not start or end at $0\in \latticeN$.} \\
    With $\vecindex_i=0$, for some $i\in \{2,3,\dots, |\latticeN|-1\}$, and following the spirit of \textit{Case 1}, we split the path 
    $\vecindex_1\leftrightarrow \vecindex_2 \leftrightarrow \dots \leftrightarrow \vecindex_{|\latticeN|}$ into $\operatorname{path}_{\leftarrow}\defeq\vecindex_{i} \rightarrow \vecindex_{i-1}\rightarrow  \dots \rightarrow \vecindex_{1}$ and $\operatorname{path}_{\rightarrow}\defeq\vecindex_{i} \rightarrow \vecindex_{i+1}\rightarrow  \dots \rightarrow \vecindex_{|\latticeN|}$. The idea is to prove \eqref{eq:store_value} by performing induction across $\operatorname{path}_{\leftarrow}$ and $\operatorname{path}_{\rightarrow}$ separately. This can be done following the same procedure as in \textit{Case 1}. 
\end{itemize}

\textit{Step 3:} Repeat \textit{Steps 1} and \textit{2} for all $g\in\LFMsetZero{w}$ and collect all the resulting $\{\hat{g}_{\vecindex}\}_{\vecindex\in\latticeN}$ in the set $\coveringset_0$. For \textit{Case 1} in \textit{Step 2}, we need to store one value for $\hat{g}_{\vecindex_{1}}$ and, according to \eqref{eq:deviation_store_g}, $\left(2\ceil{\frac{s+1}{2}}+1\right)$ increments or decrements from $\hat{g}_{\vecindex_{i}}$ to $\hat{g}_{\vecindex_{i+1}}$, for $i=0,1,\dots,|\latticeN|-1$. 
This yields a total of
$$\left(2\left\lceil\frac{s+1}{2}\right\rceil+1\right)^{|\latticeN|-1}$$
values.
For \textit{Case 2}, noting that $\operatorname{path}_{\leftarrow}$ and $\operatorname{path}_{\rightarrow}$ are of length $i-1$ and $|\latticeN|-i$, respectively, and applying the same argument as above, there is again a total of
$$\left(2\left\lceil\frac{s+1}{2}\right\rceil+1\right)^{i-1} \left(2\left\lceil\frac{s+1}{2}\right\rceil+1\right)^{|\latticeN|-i}=\left(2\left\lceil\frac{s+1}{2}\right\rceil+1\right)^{|\latticeN|-1}$$
increments or decrements. The set $\coveringset_0$ 
is hence of cardinality
\begin{equation}\label{eq:count_round_combi}
   |\coveringset_0| = \left(2\left\lceil\frac{s+1}{2}\right\rceil+1\right)^{|\latticeN |-1} = \left(2\left\lceil\frac{s+1}{2}\right\rceil+1\right)^{\left[\prod_{\ell=0}^{T}\left(2\ceil{\frac{\bound w[\ell]}{\epsilon}  \frac{s+1}{s}} + 1\right)\right]-1}.
\end{equation}

Finally, setting
\begin{equation*}
    \coveringset = 
\left\{\tilde{g}: S_- \rightarrow \R \biggm| \tilde{g}(x) = \sum_{\vecindex\in\latticeN } \hat{g}_{\vecindex} \phi(k(x)-\vecindex), ~\{\hat{g}_{\vecindex}\}_{\vecindex\in\latticeN } \in \coveringset_0 \right\},
\end{equation*}
concludes the proof.
\end{proof}

Based on the $\epsilon$-net constructed in Lemma \ref{lm:covering_spike}, we can now upper-bound the metric entropy of $(\LFMsetZero{w},\metricZero)$ as follows.

\begin{corollary}\label{cor:upper_bound_g0}
    \label{col:covering_number}
    For every \textcolor{Iter2Color}{$s\geq 1$}, the covering number of $(\LFMsetZero{w},\metricZero)$ satisfies
    \[
        \log \covering(\epsilon; \LFMsetZero{w}, \metricZero)\leq \log \left(2\left\lceil\frac{s+1}{2}\right\rceil+1\right)\prod_{\ell=0}^{T}\left(2\ceil{\frac{2\bound w[\ell]}{\epsilon}  \frac{s+1}{s}} + 1\right),     
    \]
    where $T\defeq \max\left\{ \ell \in \Nzero \mid w[\ell] > \frac{\epsilon}{2\bound} \frac{s}{s+1} \right\}$.
\end{corollary}
\begin{proof}
    For every $\epsilon> 0$, Lemma \ref{lm:covering_spike} delivers an $\frac{\epsilon}{2}$-net for $(\LFMsetZero{w},\metricZero)$ with $$\left(2\left\lceil\frac{s+1}{2}\right\rceil+1\right)^{\left[\prod_{\ell=0}^{T}\left(2\ceil{\frac{2\bound w[\ell]}{\epsilon}  \frac{s+1}{s}} + 1\right)\right]-1}$$ elements. By Definition \ref{def:both_covering_numbers}, it hence follows that
    \begin{align*}
        \covering^{\text{ext}}(\epsilon/2; \LFMsetZero{w}, \metricZero) 
        & \leq \left(2\left\lceil\frac{s+1}{2}\right\rceil+1\right)^{\left[\prod_{\ell=0}^{T}\left(2\ceil{\frac{2\bound w[\ell]}{\epsilon}  \frac{s+1}{s}} + 1\right)\right]}.
    \end{align*}
    Application of Lemma \ref{lm:number_relation} then yields an upper bound on the covering number according to
    $$\covering(\epsilon; \LFMsetZero{w}, \metricZero) \leq \covering^{\text{ext}}(\epsilon/2; \LFMsetZero{w}, \metricZero) \leq \left(2\left\lceil\frac{s+1}{2}\right\rceil+1\right)^{\left[\prod_{\ell=0}^{T}\left(2\ceil{\frac{2\bound w[\ell]}{\epsilon}  \frac{s+1}{s}} + 1\right)\right]},$$
    which finishes the proof.
\end{proof}

We conclude the developments in this section by deriving a lower bound and an upper bound on the exterior covering number of $(\LFMset{w}, \opMetric)$.

\begin{theorem}\label{th:main_entropy_result}
    The exterior covering number of $(\LFMset{w}, \opMetric)$ satisfies
    \begin{equation}
        \left (\prod_{\ell=0}^{T'} \ceil{\frac{\bound w[\ell]}{{\epsilon}}} \right ) - 1 
        \leq \log \covering^{\text{ext}}(\epsilon; \LFMset{w}, \opMetric) 
        \leq \log \left(3\right)\prod_{\ell=0}^{T''}\left(2\ceil{\frac{4\bound w[\ell]}{\epsilon} } + 1\right),
    \end{equation}
    where $T' \defeq \max \left\{ \ell \in \Nzero \mid w[\ell]> \frac{\epsilon}{\bound}\right\}$ and $T''  \defeq  \max \left\{ \ell \in \Nzero \mid w[\ell]> \frac{\epsilon}{4\bound } \right\}$.
\end{theorem}
\begin{proof}
The result follows by noting that
\begin{equation}
\begin{aligned}
    \left (\prod_{\ell=0}^{T'} \ceil{\frac{\bound w[\ell]}{{\epsilon}}} \right ) - 1 \overset{\text{Lemma \ref{lm:packing_number}}}&{\leq}\log\packing(2\epsilon; \LFMsetZero{w}, \metricZero) \\
    \overset{\text{Lemma \ref{lem:g0_g_isomorphism}}}&{=} \log\packing(2\epsilon; \LFMset{w}, \opMetric) \\
    \overset{\text{Lemma \ref{lm:number_relation}}}&{\leq} \log\covering^{\text{ext}}(\epsilon; \LFMset{w}, \opMetric) \leq \log\covering(\epsilon; \LFMset{w}, \opMetric) \\
    \overset{\text{Lemma \ref{lem:g0_g_isomorphism}}}&{=} \log \covering(\epsilon; \LFMsetZero{w}, \metricZero) \\
    \overset{\overset{\text{Corollary \ref{cor:upper_bound_g0}}}{\text{for } s=1}}&{\leq} \log\left(3\right)\prod_{\ell=0}^{T''}\left(2\ceil{\frac{4\bound w[\ell]}{\epsilon} } + 1\right). 
\end{aligned}  
\end{equation}
\end{proof}
\section{Metric entropy of \gls{elfm} and \gls{plfm} systems}\label{sec:rate_elfm_plfm}

We now discuss two specific classes of LFM systems, namely \glsfirst{elfm} and \glsfirst{plfm} systems. Specifically, we characterize the \textcolor{Iter1Color}{description complexity} of these two classes by computing their type, order, and generalized dimension. \textcolor{Iter1Color}{The corresponding results will then serve as a reference for the RNN approximations in Section \ref{sec:RNN2G}. Specifically, we will establish, in Section \ref{sec:optimality_RNN}, that RNNs can \textcolor{black}{approximate} \gls{elfm} and \gls{plfm} systems in metric-entropy-optimal manner.}

\subsection{Metric entropy of \gls{elfm} systems}\label{sec:ELFM_rate}

The concept of ELFM systems is inspired, inter alia, by applications in finance, such as those discussed in \cite{nagel2022asset}, where asset pricing decisions are influenced by past observations. Instead of relying solely on finite-length observations, the model integrates infinite past observations by endowing them with exponentially decaying memory. Similar settings are also considered in random walk models \cite{alves2014superdiffusion,moura2016transient}. 
These examples, when appropriately adapted to the setup in the present paper, fit into the setting of \gls{lfm} systems according to Definition \ref{def:lfm_sys_set}, with respect to exponentially decaying weight sequences. 
We formally define exponentially decaying memory and the corresponding 
ELFM systems as follows.

\begin{definition}\label{def:weight_exp}
For $a\in (0,1]$ and $b>0$, let
$$\weightSeqExp{a}{b}[t] \defeq ae^{-bt}, \qquad \text{for all } t\in \Nzero.$$
An LFM system with respect to $\{\weightSeqExp{a}{b}[t]\}_{t\in \Nzero}$ is said to be \glsfirst{elfm}.
We write $\LFMset{\weightSeqExp{a}{b}}$ for the class of systems that are ELFM with respect to the weight sequence $\weightSeqExp{a}{b}[t]$.
\end{definition}

The remainder of this section is devoted to computing the order, type, and generalized dimension of $(\LFMset{\weightSeqExp{a}{b}},\opMetric)$.
To this end, we first establish an auxiliary result.

\begin{lemma}\label{lm:estimation} Let $a\in (0,1]$, $b,c,d>0$ and consider the weight sequence $\weightSeqExp{a}{b}$ as per Definition \ref{def:weight_exp}. Let
$$T \defeq \max \left\{ t \in \Nzero \;\middle|\; \weightSeqExp{a}{b}[t]> \frac{\epsilon}{d}\right\}.$$
Then,
\begin{equation}\label{eq:estimation}
    \log\left( \prod_{\ell=0}^T \frac{c \weightSeqExp{a}{b}[\ell]}{\epsilon}\right) = \frac{1}{2b\log (e)} \log^2(\epsilon^{-1}) + \smallo\left(\log^2(\epsilon^{-1})\right).
\end{equation}
\end{lemma}
\begin{proof}
    See \ref{sec:ELFM_est}.
\end{proof}

We are now ready to state the main result of this section quantifying the massiveness of the class of \gls{elfm} systems.

\begin{lemma}\label{lm:dim_ELFM}
Let $a\in (0,1]$ and $b>0$. The class of \gls{elfm} systems $(\LFMset{\weightSeqExp{a}{b}},\opMetric)$ is of order $1$ and type $2$, with generalized dimension $$\genDim{} = \frac{1}{2b \log (e)}.$$
\end{lemma}
\begin{proof}
Consider $\epsilon \in (0,\epsilon_0)$ with $\epsilon_0=\frac{D\weightSeqExp{a}{b}[0]}{2} = \frac{aD}{2}$. By Theorem \ref{th:main_entropy_result}, the exterior covering number of $(\LFMset{\weightSeqExp{a}{b}},\opMetric)$ satisfies
    \begin{equation}\label{eq:bound_covering}
        \left (\prod_{\ell=0}^{T'} \ceil{\frac{\bound \weightSeqExp{a}{b}[\ell]}{{\epsilon}}} \right ) - 1 
        \leq \log \covering^{\text{ext}}(\epsilon; \LFMset{\weightSeqExp{a}{b}}, \opMetric) 
        \leq \log \left(3\right)\prod_{\ell=0}^{T''}\left(2\ceil{\frac{4\bound \weightSeqExp{a}{b}[\ell]}{\epsilon}  } + 1\right),
    \end{equation}
where 
\begin{equation*}
    T^{\prime} \defeq \max \left\{ \ell \in \Nzero \Bigm| \weightSeqExp{a}{b}[\ell]> \frac{\epsilon}{\bound}\right\} \text{ and } T''  \defeq  \max \left\{ \ell \in \Nzero \Bigm| \weightSeqExp{a}{b}[\ell]> \frac{\epsilon}{4\bound } \right\} .
\end{equation*}
We can further lower-bound the left-most term in \eqref{eq:bound_covering} according to
\begin{align}
    \left (\prod_{\ell=0}^{T^{\prime}} \left\lceil\frac{\bound \weightSeqExp{a}{b}[\ell]}{{\epsilon}}\right\rceil \right ) - 1 &\geq \left(\prod_{\ell=0}^{T^{\prime}} \frac{\bound \weightSeqExp{a}{b}[\ell]}{{\epsilon}}\right) -1 \nonumber \\
    & \geq \frac{1}{2} \prod_{\ell=0}^{T^{\prime}} \frac{\bound \weightSeqExp{a}{b}[\ell]}{{\epsilon}} \label{eq:lower_bound_covering},
\end{align}
where \eqref{eq:lower_bound_covering} follows from
\begin{equation*}
    \frac{1}{2} \prod_{\ell=0}^{T^{\prime}} \frac{\bound \weightSeqExp{a}{b}[\ell]}{{\epsilon}}\geq 1,
\end{equation*}
which, in turn, is a consequence of 
\begin{align*}
   & \frac{\bound \weightSeqExp{a}{b}[0]}{{\epsilon}} \geq \frac{\bound \weightSeqExp{a}{b}[0]}{{\epsilon_0}} = 2,\\
   & \frac{\bound \weightSeqExp{a}{b}[\ell]}{{\epsilon}} \geq \frac{\bound \weightSeqExp{a}{b}[T']}{{\epsilon}} > 1, \quad \text{for}~~\ell \in \ssidedEnum{T'}\setminus\{0\}.
\end{align*}
Similarly, we can further upper-bound the right-most term in (\ref{eq:bound_covering}) according to
\begin{align}
    \log(3)\prod_{\ell=0}^{T''} \left(2\ceil{\frac{4\bound \weightSeqExp{a}{b}[\ell]}{\epsilon}  } + 1\right)  &\leq \log(3)\prod_{\ell=0}^{T''} \left(\frac{8\bound \weightSeqExp{a}{b}[\ell]}{{\epsilon}}+3\right)  \nonumber \\
    & \leq \log(3)\prod_{\ell=0}^{T''} \frac{20\bound \weightSeqExp{a}{b}[\ell]}{{\epsilon}} \label{eq:upper_bound_covering},
\end{align}
where \eqref{eq:upper_bound_covering} follows from 
$$\frac{4\bound \weightSeqExp{a}{b}[\ell]}{{\epsilon}} \geq \frac{4\bound \weightSeqExp{a}{b}[T'']}{{\epsilon}} > 1, \quad \text{for}~~\ell \in \ssidedEnum{T''} .$$
Combining 
(\ref{eq:bound_covering})--(\ref{eq:upper_bound_covering}), 
taking logarithms one more time, and dividing by $\log^2 (\epsilon^{-1})$, yields
\begin{equation} \label{eq:covering_final_bound}
\begin{aligned}
    \frac{\log\left( \prod_{\ell=0}^{T^{\prime}} \frac{\bound \weightSeqExp{a}{b}[\ell]}{{\epsilon}} \right)-1}{\log^2 (\epsilon^{-1})} &\leq \frac{\log^{(2)} \covering^{\text{ext}}(\epsilon; \LFMset{\weightSeqExp{a}{b}},\opMetric)}{\log^2 (\epsilon^{-1})} \\
    &\leq \frac{\log\left (\prod_{\ell=0}^{T''} \frac{20\bound \weightSeqExp{a}{b}[\ell]}{{\epsilon}} \right )+\log^{(2)}(3)}{\log^2 (\epsilon^{-1})}.    
\end{aligned}
\end{equation}
Taking the limit $\epsilon \rightarrow 0$ and applying Lemma \ref{lm:estimation} to the lower and the upper bound in (\ref{eq:covering_final_bound}), we obtain 
$$\lim_{\epsilon \rightarrow 0} \frac{\log^{(2)} \covering^{\text{ext}}(\epsilon; \LFMset{\weightSeqExp{a}{b}}
,\opMetric)}{\log^2 (\epsilon^{-1})} = \frac{1}{2b \log (e)},$$
which implies that $(\LFMset{\weightSeqExp{a}{b}},\opMetric)$ is of order 1 and type 2, with generalized dimension $$\genDim{} = \frac{1}{2b \log (e)}.$$
This concludes the proof.
\end{proof}
The generalized dimension being inversely proportional to $b$ reflects that faster memory decay rates, i.e., larger $b$, result in system classes that are less massive. Additionally, we note that the description complexity of \gls{elfm} systems is primarily determined by the order being equal to 1 and the type equal to 2. Compared to the function classes in Table \ref{tbl:nn}, which are all of order 1 and type 1, this shows that the class of \gls{elfm} systems is significantly more massive than unit balls in function spaces.

\subsection{Metric entropy of \gls{plfm} systems}\label{sec:PLFM_rate}

Next, we consider systems with polynomially decaying memory, a concept used, e.g., in the context of PDEs \cite{mainini2009exponential, fabrizio2002asymptotic}. Specifically, the references \cite{mainini2009exponential, fabrizio2002asymptotic} are concerned with Volterra integro-differential equations \cite{lakshmikantham1995theory} of polynomially decaying memory kernels. These examples, when suitably adjusted to the framework in the present paper,
align with dynamical systems that are \gls{lfm} 
with respect to polynomially decaying weight sequences. 
We formalize the concept of polynomially decaying memory and \glsfirst{plfm} systems as follows.

\begin{definition}\label{def:weight_poly}
    For $q\in (0,1]$ and $p>0$, let 
$$\weightSeqPoly{p}{q}[t] \defeq \frac{q}{(1+t)^p}, \qquad \text{for all }  t\in \Nzero.$$
An LFM system with respect to $\{\weightSeqPoly{p}{q}[t]\}_{t\in \Nzero}$ is said to be \glsfirst{plfm}.
We write $\LFMset{\weightSeqPoly{p}{q}}$ for the class of systems that are PLFM with respect to the weight sequence $\weightSeqPoly{p}{q}[t]$.
\end{definition}

As in the previous section, we first need an auxiliary result.

\begin{lemma}\label{lm:estimation_poly}
    Let $q\in (0,1]$, $p,c,d>0$ and consider $\weightSeqPoly{p}{q}$ as per Definition \ref{def:weight_poly}. Let 
$$T \defeq \max \left\{ t \in \Nzero \;\middle|\; \weightSeqPoly{p}{q}[t]> \frac{\epsilon}{d}\right\}.$$
Then,
\begin{equation}\label{eq:estimation_poly}
    \log\left( \prod_{\ell=0}^T \frac{c \weightSeqPoly{p}{q}[\ell]}{\epsilon}\right) = \asymporder(\epsilon^{-1/p}).
\end{equation}
\end{lemma}
\begin{proof}
    See \ref{sec:PLFM_est}.
\end{proof}

We obtain the order, type, and generalized dimension of \gls{plfm} systems as follows.

\begin{lemma}\label{lm:dim_PLFM}
    Let $q\in (0,1]$ and $p>0$. The class of \gls{plfm} systems $(\LFMset{\weightSeqPoly{p}{q}},\opMetric)$ is of order $2$ and type $1$, with generalized dimension 
    $$\genDim{} = \frac{1}{p}.$$
\end{lemma}
\begin{proof}
    See \ref{sec:app_rate_plfm}.
\end{proof}
Compared to the class of \gls{elfm} systems, which exhibits order $1$, the class of \gls{plfm} systems is more massive
as it has order $2$.
This is intuitively meaningful as polynomial decay is significantly slower than exponential decay, rendering \gls{plfm} systems to depend more strongly on past inputs. Additionally, the generalized dimension exhibiting inverse proportionality to $p$, reflects that faster (polynomial) decay, i.e., larger $p$, leads to smaller description complexity.

\section{Approximating LFM systems through RNNs} \label{sec:RNN2G}
\textcolor{Iter1Color}{ 
We now proceed to realize the elements in the $\epsilon$-net for $(\LFMsetZero{w},\metricZero)$ constructed in Lemma \ref{lm:covering_spike} by \gls{relu} networks. Based on the connection between $(\LFMsetZero{w},\metricZero)$ and $(\LFMset{w},\opMetric)$ identified in Section \ref{sec:app_rate_lfms}, this will then allow us to construct \glspl{rnn} approximating general \gls{lfm} systems.
}
\subsection{Approximating $(\LFMsetZero{w},\metricZero)$ with \gls{relu} networks}\label{sec:nn2g0}
The building block in Lemma \ref{lm:covering_spike} for approximating functionals in $\LFMsetZero{w}$ is the \gls{pou} $\Xi =\{\phi(\cdot-\vecindex)\}_{\vecindex\in \lattice}$, which is why we first focus on constructing \gls{relu} networks realizing $\Xi$. 
As the elements of $\Xi$ are shifted versions of the spike function $\phi$ and shifts, by virtue of being affine transformations, are trivially realized by single-layer \gls{relu} networks, it suffices to find a \gls{relu} network realization of $\phi$. Note that this argument made use of the fact that compositions of \gls{relu} networks are again \gls{relu} networks (see Lemma \ref{lm:compo_nn}).

\begin{lemma}\label{lm:NN2phi}
For $d \in \Nplus$, consider the spike function $\phi: \R^{d} \rightarrow \R$,
\begin{equation}
    \phi(z) = \max \{1 + \min\{z_1,\dots,z_d,0\} - \max\{z_1,\dots,z_d,0\}, 0\}.
\end{equation}
There exists a \gls{relu} network $\Phi \in \relunn_{d,1}$ with $\depth(\Phi)=\ceil{\log(d+1)}+4$, $\nnz(\Phi)\leq 60d-28$, $\width(\Phi)\leq 6d$, $\weightsSet(\Phi)=\{1,-1\}$, and $\weightmeg(\Phi)=1$,
such that 
\begin{equation}
    \Phi(z) = \phi(z), \qquad \text{for all} \quad z \in \R^{d}.
\end{equation}
\end{lemma}
\begin{proof}
    Let $\Phi_d^{max}$ be the \gls{relu} network realization of $\max\{z_1,z_2,\dots,z_d\}$ according to Lemma \ref{lm:NN2minmax}. Then, using $\min\{z_1,z_2,\dots,z_d\} = -\max\{-z_1,-z_2,\dots,-z_d\}$, we obtain 
    \begin{equation}\label{eq:spike_writeout}
    \begin{aligned}
        \phi(z) &= \relu (1-\relu(\Phi_d^{max}(-z)) - \relu(\Phi_d^{max}(z)))\\
        &= (\underbrace{(W_3\circ \relu\circ W_2\circ \relu\circ P(\Phi_d^{max},\Phi_d^{max}))}_{\eqdef\Phi_2}\circ \underbrace{W_1}_{\eqdef\Phi_1}) (z),
    \end{aligned}
    \end{equation}
    where $W_1(z)=\begin{pmatrix}
        \Imat{d}&-\Imat{d}
    \end{pmatrix}^Tz$,
    $W_2(z)=\begin{pmatrix}
        -1&-1
    \end{pmatrix}z+1$, $W_3(z)=z$, and $P(\Phi_d^{max},\Phi_d^{max})$ is the parallelization of two $\Phi_d^{max}$-networks according to Lemma \ref{lm:para_nn} such that $P(\Phi_d^{max},\Phi_d^{max})(z)=(\Phi_d^{max}(z),\Phi_d^{max}(z))^T$. Now, by Lemma \ref{lm:compo_nn}, there exists a \gls{relu} network $\Phi$ realizing $\Phi_2\circ \Phi_1$ in \eqref{eq:spike_writeout}, with 
    \begin{equation*}
        \begin{aligned}
            \depth(\Phi) \overset{\text{Lemma \ref{lm:compo_nn}}}&{=} \depth(\Phi_2) + \depth(\Phi_1)\\
            \overset{\text{Lemma \ref{lm:NN2minmax}, \ref{lm:para_nn}}}&{=} \ceil{\log(d+1)}+4,\\
            \nnz(\Phi) \overset{\text{Lemma \ref{lm:compo_nn}}}&{\leq} 2\nnz(\Phi_1)+2\nnz(\Phi_2)\\
            \overset{\text{Lemma \ref{lm:para_nn}}}&{=} 2\nnz(W_3)+ 2\nnz(W_2) + 2\nnz(W_1) + 4\nnz(\Phi_d^{max})\\
            \overset{\text{Lemma \ref{lm:NN2minmax}}}&{\leq} 60d-28,\\
            \width(\Phi) \overset{\text{Lemma \ref{lm:compo_nn}}}&{\leq} \max\left\{4d,\max\{\width(\Phi_2),\width(\Phi_1)\}\right\}\\
            \overset{\text{Lemma \ref{lm:para_nn}}}&{\leq} \max\left\{4d,2\width(\Phi_d^{max}),\width(W_3),\width(W_2),\width(W_1)\right\}\\
            \overset{\text{Lemma \ref{lm:NN2minmax}}}&{\leq} 6d,\\
            \weightsSet(\Phi) \overset{\text{Lemma \ref{lm:compo_nn}}}&{\subset} \weightsSet(\Phi_2)\cup (-\weightsSet(\Phi_2))\cup\weightsSet(\Phi_1)\cup (-\weightsSet(\Phi_1))\\
            \overset{\text{Lemma \ref{lm:compo_nn},\ref{lm:para_nn}}}&{=} 
            \left(\bigcup_{i=1}^3\weightsSet(W_i)\cup(-\weightsSet(W_i))\right) \cup \weightsSet(\Phi_d^{max})\cup(-\weightsSet(\Phi_d^{max}))\\
            \overset{\text{Lemma \ref{lm:NN2minmax}}}&{=}\{1,-1\},\\
            \weightmeg(\Phi) &= \max_{b\in \weightsSet(\Phi)}|b|=1.
        \end{aligned}
    \end{equation*}
    This concludes the proof.
\end{proof}

We now show how the elements of the $\epsilon$-net constructed in Lemma \ref{lm:covering_spike} can be realized by \gls{relu} networks. In particular, our construction will be found, in Section \ref{sec:optimality_RNN}, to be encodable by bitstrings of length scaling (in $\epsilon$) in a metric-entropy optimal manner.

\begin{lemma}\label{lm:NN2G_0}
For every \textcolor{Iter2Color}{$s\geq 1$}, there exists $\epsilon_0>0$, such that for $\epsilon\in (0,\epsilon_0)$, with 
$$T\defeq \max\left\{ \ell \in \Nzero \Bigm| w[\ell] > \frac{\epsilon}{\bound} \frac{s}{s+1} \right\},$$
for every $g \in \LFMsetZero{w}$, there exists a $\Phi \in \relunn_{T+1,1}$ with $(2,\epsilon)$-quantized weights (Definition \ref{def:quan_nn}) satisfying
\begin{equation*}
    \left|\Phi\left(\left\{x[-\ell]\right\}_{\ell=0}^{T}\right)-g(x)\right| \leq \epsilon, \qquad \text{for all }  x\in \sigSpaceminus.
\end{equation*}
Moreover, 
\begin{equation}
    \depth(\Phi) = \ceil{\log(T+2)}+6 ~\text{ and }  ~  \nnz(\Phi)\leq 244(T+1)  \prod_{\ell=0}^{T} \left(\frac{2Dw[\ell]}{\epsilon}\frac{s+1}{s} +4\right).
\end{equation}
\end{lemma}
\begin{proof}
Fix $g\in \LFMsetZero{w}$. To construct a $(2, \epsilon)$-quantized \gls{relu} network approximating $g$, we follow the spirit of the proof of Lemma \ref{lm:covering_spike} and first consider
\begin{equation}\label{eq:hat_g_2}
    \hat{g}(x) = 
    \sum_{\vecindex\in \latticeN} g(\xhat_{\vecindex}) \phi\left( \left\{\frac{x[-\ell]}{\delta_{\ell}}-\vecindex_{\ell} \right\}_{\ell=0}^{T}\right),
\end{equation}
with 
\begin{equation}\label{eq:lattice_parameters}
\begin{aligned}
    \delta_{\ell} &\defeq \frac{s}{s+1} \frac{\epsilon}{w[\ell]},\qquad\forall \ell \in \ssidedEnum{T},\\
    \latticeBound_{\ell} & \defeq \ceil{\frac{\bound}{\delta_{\ell}}} ,\quad \forall\ell \in \ssidedEnum{T},\\
    \latticeN &\defeq \dsidedEnum{N_0}\times \cdots\times \dsidedEnum{N_T} \subset \R^{T+1},\\
    \xhat_{\vecindex}[-\ell] &\defeq 
    \begin{cases*}
    \delta_\ell \vecindex_{\ell}, & if $\ell \in \ssidedEnum{T}$,\\
    0, & else.
    \end{cases*}
\end{aligned}
\end{equation}
It was shown in \eqref{eq:hat_g2g} that 
\begin{equation}
    |\hat{g}(x) - g(x)| \leq \frac{s\epsilon}{s+1}, \quad \text{for }x\in \sigSpaceminus.
\end{equation}
We next quantize the parameters $\delta_\ell^{-1}$ and the lattice $\latticeN$ in \eqref{eq:lattice_parameters} according to
\begin{equation}\label{eq:quantize_delta}
\begin{aligned}
    \deltaQuan_{\ell}^{-1} &\defeq \quantize{2}{\epsilon}(\delta_\ell^{-1}),\\
    \latticeBoundQuan_{\ell} &= \ceil{\frac{D}{\deltaQuan_{\ell}}}, \quad \ell\in \ssidedEnum{T},\\
    \latticeNQuan &= \dsidedEnum{\latticeBoundQuan_0}\times \cdots\times \dsidedEnum{\latticeBoundQuan_T},
\end{aligned}
\end{equation}
and adjust the grid points $\xhat_{\vecindex}$ to
\begin{equation}
    \xtilde_{\vecindex}[-\ell] = 
    \begin{cases*}
    \deltaQuan_\ell \vecindex_{\ell}, & if $\ell \in \ssidedEnum{T}$, \\
                0, & else.
    \end{cases*}
\end{equation}
Furthermore, we quantize $g(\xtilde_{\vecindex})$ according to
\begin{equation}\label{eq:quantize_g}
    \gQuan_{\vecindex}  \defeq \quantize{2}{\epsilon}(g(\xtilde_{\vecindex}))
\end{equation}
and consider the function 
\begin{equation}\label{eq:def_gtilde}
    \tilde{g}(x) = \sum_{\vecindex\in\latticeNQuan } \gQuan_{\vecindex} \phi\left( \left\{\frac{x[-\ell]}{\deltaQuan_{\ell}}-\vecindex_{\ell} \right\}_{\ell=0}^{T}\right) \eqdef f\left(\left\{x[-\ell]\right\}_{\ell=0}^{T}\right).
\end{equation}
For ease of notation, we define $\tilde{k}: S_{-} \rightarrow \R^{T+1}$ as $\tilde{k}_{\ell}(x) = \deltaQuan_{\ell}^{-1}x[-\ell],$ for $\ell \in \ssidedEnum{T}$. Next, we show that 
\begin{equation}\label{eq:gtilde2g}
    |\tilde{g}(x)-g(x)|\leq \epsilon.
\end{equation}
This follows from
\begin{align}
    |\tilde{g}(x) - g(x)| 
    =~&\left| \sum_{ \|\tilde{k}(x)-\vecindex \|_\infty \leq 1} (g(x)- \gQuan_{\vecindex}) \phi(\tilde{k}(x)-\vecindex)\right| \qquad \label{eq:lemma_nn2G0_app_proof_pou1}\\
     \leq~& \sum_{\|\tilde{k}(x)-\vecindex \|_{\infty}\leq 1}\left(\left| (g(x)- g(\xtilde_{\vecindex}))\right|+\left| (g(\xtilde_{\vecindex})- \gQuan_{\vecindex})\right| \right)\phi(\tilde{k}(x)-\vecindex )\\
    \leq ~& \sum_{\|\tilde{k}(x)-\vecindex \|_{\infty}\leq 1}\left(\left\| x- \xtilde_{\vecindex}\right\|_w +\frac{\epsilon}{s+1}\right)\phi(\tilde{k}(x)-\vecindex )\label{eq:lemma_nn2G0_app_proof_lip}\\
    \leq~& \epsilon\sum_{\|\tilde{k}(x)-\vecindex \|_{\infty}\leq 1} \phi(\tilde{k}(x)-\vecindex )\label{eq:lemma_nn2G0_app_proof_grid_size}\\
    \leq ~& \epsilon \label{eq:lemma_nn2G0_app_proof_pou2}.
\end{align}
Here, \eqref{eq:lemma_nn2G0_app_proof_pou1} and \eqref{eq:lemma_nn2G0_app_proof_pou2} are by the \gls{pou} property of $\phi$ and \eqref{eq:lemma_nn2G0_app_proof_lip} is a consequence of the Lipschitz property of $g$ according to \eqref{eq:nonLinSpace} and 
\begin{equation}
    \begin{aligned}
    \left| g(\xtilde_{\vecindex}) -\gQuan_{\vecindex} \right| & = \left| g(\xtilde_{\vecindex}) -\quantize{2}{\epsilon}(g(\xtilde_{\vecindex}))\right|\\
    & \leq 2^{-2\ceil{\log (\epsilon^{-1})}} \overset{\epsilon\in (0,\epsilon_0), \eqref{eq:eps0}}{\leq} \frac{\epsilon}{s+1}.
\end{aligned}
\end{equation}
Further, \eqref{eq:lemma_nn2G0_app_proof_grid_size} follows from the fact that for $\|\tilde{k}(x)-\vecindex\|_{\infty}\leq 1 $,
\begin{equation*}
    \left|\frac{x[-\ell]}{\deltaQuan_{\ell}} - \vecindex_{\ell}\right| \leq 1, \quad \text{for }\ell \in \ssidedEnum{T},
\end{equation*}
and hence 
\begin{equation}\label{eq:x2xtilde}
    \|x-\xtilde_{\vecindex}\|_w  \leq\max_{\ell\in \ssidedEnum{T}} \deltaQuan_{\ell}w[\ell]\leq\frac{s\epsilon}{s+1},
\end{equation}
where the second inequality in \eqref{eq:x2xtilde} is by $$\deltaQuan_{\ell} \overset{\eqref{eq:quantize_delta}}{=} \left(\quantize{2}{\epsilon} \left(\delta_{\ell}^{-1}\right)\right)^{-1}\leq \delta_{\ell}=\frac{s}{s+1}\frac{\epsilon}{w[\ell]}.$$
Based on \eqref{eq:def_gtilde}, we can rewrite \eqref{eq:gtilde2g} according to
\begin{equation}\label{eq:Phi2g}
    \left|f\left(\left\{x[-\ell]\right\}_{\ell=0}^{T}\right)-g(x)\right| \leq \epsilon, \qquad \text{for }  x\in S_-.
\end{equation}
It hence suffices to construct a \gls{relu} network $\Phi$ that realizes $f$, which then, thanks to \eqref{eq:Phi2g}, approximates $g$ to within an error of at most $\epsilon$. To this end, assume that $\vecindex^1,\vecindex^2,\dots,\vecindex^{|\latticeNQuan |}$ is an arbitrary, but fixed, enumeration of the elements of $\latticeNQuan$. Set $\widetilde{W}^{\Sigma} (x)= \Lambda x$ and $\widehat{W}_{\vecindex^i}(x) = \widehat{B}x + \hat{b}_{\vecindex^i}$, with 
\begin{equation}\label{eq:weights_coe_shift}
    \begin{aligned}
    \Lambda &= \begin{pmatrix}
        \gQuan_{\vecindex^1}& \gQuan_{\vecindex^2}&\dots&\gQuan_{\vecindex^{|\latticeNQuan |}}
    \end{pmatrix},
    \\
    \widehat{B} &= \operatorname{diag}\left(\deltaQuan_0^{-1},\deltaQuan_1^{-1},\dots,\deltaQuan_T^{-1}\right),\\
    \hat{b}_{\vecindex^i} &= \begin{pmatrix}
        -\vecindex^i_0&-\vecindex^i_1& \dots & -\vecindex^i_T
    \end{pmatrix}^T.
    \end{aligned}
\end{equation}
Moreover, let $\Psi$ be a \gls{relu} network realizing the spike function $\phi$ according to Lemma \ref{lm:NN2phi} and define the following \gls{relu} networks
\begin{equation}\label{eq:structure_NN_parts}
    \begin{aligned}
        \Phi_2 & = \widetilde{W}^{\Sigma},\\
        \Phi^{\vecindex^i}_{1,2} & = \Psi,\hspace{0.9cm}\text{for }i=1,\dots,|\latticeNQuan|,\\
        \Phi^{\vecindex^i}_{1,1} & = \widehat{W}_{\vecindex^i}, \hspace{0.6cm} \text{for }i=1,\dots,|\latticeNQuan|.
    \end{aligned}
\end{equation}
Next, we use Lemma \ref{lm:compo_nn} to compose $\Phi^{\vecindex^i}_{1,2}$ and $\Phi^{\vecindex^i}_{1,1}$ in order to realize the shifted versions of the spike function according to
\begin{equation*}
    \left(\Phi^{\vecindex^i}_{1,2}\circ \Phi^{\vecindex^i}_{1,1}\right) \left(\left\{x[-\ell]\right\}_{\ell=0}^{T}\right)=\phi\left( \left\{\frac{x[-\ell]}{\deltaQuan_{\ell}}-\vecindex^i_{\ell} \right\}_{\ell=0}^{T}\right), \quad\text{for }i=1,2,\dots, |\latticeNQuan |.
\end{equation*}
Then, we apply Lemma \ref{lm:para_nn} to construct a \gls{relu} network as the parallelization of the compositions $\Phi^{\vecindex^i}_{1,2}\circ \Phi^{\vecindex^i}_{1,1}$, for $i=1,2,\dots, |\latticeNQuan |$,
\begin{equation}\label{eq:para_shifted_spike}
    \Phi_1\defeq P\left(\left(\Phi^{\vecindex^1}_{1,2}\circ \Phi^{\vecindex^1}_{1,1}\right),\left(\Phi^{\vecindex^2}_{1,2}\circ \Phi^{\vecindex^2}_{1,1}\right),\dots,\left(\Phi^{\vecindex^{|\latticeNQuan |}}_{1,2}\circ \Phi^{|\latticeNQuan |}_{1,1}\right) \right).
\end{equation}
Finally, we use Lemma \ref{lm:compo_nn} again to get a \gls{relu} network that composes $\Phi_2$ and $\Phi_1$ according to
\begin{equation}\label{eq:structure_NN}
    \Phi = \Phi_2 \circ \Phi_1 = \Phi_2 \circ P\left(\left(\Phi^{\vecindex^1}_{1,2}\circ \Phi^{\vecindex^1}_{1,1}\right),\left(\Phi^{\vecindex^2}_{1,2}\circ \Phi^{\vecindex^2}_{1,1}\right),\dots,\left(\Phi^{\vecindex^{|\latticeNQuan |}}_{1,2}\circ \Phi^{|\latticeNQuan |}_{1,1}\right) \right),
\end{equation}
thereby realizing the linear combination 
\begin{equation*}
    \sum_{\vecindex\in\latticeNQuan } \gQuan_{\vecindex} \phi\left( \left\{\frac{x[-\ell]}{\deltaQuan_{\ell}}-\vecindex_{\ell} \right\}_{\ell=0}^{T}\right) \overset{\eqref{eq:def_gtilde}}{=}  f\left(\left\{x[-\ell]\right\}_{\ell=0}^{T}\right).
\end{equation*}

\textcolor{IterColor}{To conclude the proof, we verify that $\Phi$, indeed, has $(2,\epsilon)$-quantized weights, compute $\depth(\Phi)$, and derive an upper bound on $\nnz(\Phi)$. We defer the corresponding details to \ref{sec:NN2G0_est}.}
\end{proof}

\subsection{Approximating $(\LFMset{w},\opMetric)$ with \glspl{rnn}}
Having constructed \gls{relu} networks that realize elements of $\LFMsetZero{w}$ according to Lemma \ref{lm:NN2G_0}, we are now ready to describe the realization of systems in $\LFMset{w}$ through RNNs. This will be done by employing the connection between $\LFMset{w}$ and $\LFMsetZero{w}$, as established in Lemma \ref{lem:g0_g_isomorphism}. Specifically, we construct \glspl{rnn} that suitably remember past inputs and produce approximations of the desired output.

\begin{theorem}\label{th:RNN2G}
For every \textcolor{Iter1Color}{$s\geq 1$}, there exists $\epsilon_0>0$, such that for $\epsilon\in (0,\epsilon_0)$, with 
$$T\defeq \max\left\{ \ell \in \Nzero \Bigm| w[\ell] > \frac{\epsilon}{\bound} \frac{s}{s+1} \right\},$$
for every $G \in \LFMset{w}$, there is an \gls{rnn} $\rnn_{\Psi}$ associated with a \gls{relu} network $\Psi \in \relunn_{T+1,T+1}$,
satisfying 
\begin{equation*}
    \opMetric(\rnn_{\Psi}, G)\leq \epsilon.
\end{equation*}
Moreover, $\rnn_{\Psi}$ has $(2,\epsilon)$-quantized weights and there exists a {\uc} $C>0$
such that 
\begin{equation}\label{eq:RNN_nnz_bound}
    \nnz(\Psi)\leq C(T+1)^{2}  \prod_{\ell=0}^{T} \left(\frac{2Dw[\ell]}{\epsilon}\frac{s+1}{s} +4\right).
\end{equation}
\end{theorem}
\begin{proof}

Fix $s\geq1$ and $G\in \LFMset{w}$ arbitrarily. We proceed in two steps.

\textit{Step 1:} \textit{We construct a \gls{relu} network $\Phi: \R^{T+1} \rightarrow \R $ such that
    \begin{equation}\label{eq:rnn_output_app2G}
         \sup_{x\in \sigSpaceplus}\sup_{t\in \Nzero}\left|\Phi\left(\{x[t-\ell]\}_{\ell=0}^{T}\right)-G(x)[t]\right|\leq \epsilon.
    \end{equation}
} \\
To this end, we first note that by Lemma \ref{lem:g0_g_isomorphism}, one can find a $g \in \LFMsetZero{w}$ so that 
\begin{equation}\label{eq:mainth_G2g}
    g\left(\projLeftSide\mathbf{T}_{-t}x\right) = G(x)[t], \quad \text{for all } x\in \sigSpace \text{ and }t\in \mathbb{Z}.
\end{equation}
Furthermore, by Lemma \ref{lm:NN2G_0}, there exists $\epsilon_0>0$, such that for every $\epsilon\in(0,\epsilon_0)$, there is a \gls{relu} network $\Phi$ satisfying 
\begin{equation}\label{eq:mainth_Phi2g}
    \left|\Phi\left(\left\{z[-\ell]\right\}_{\ell=0}^{T}\right)-g(z)\right| \leq \epsilon, \quad \text{for all } z \in \sigSpaceminus.
\end{equation}
Next, fix an input $x\in \sigSpaceplus$ and a time index $t\in \Nzero$ and let 
\begin{equation}\label{eq:mainth_z}
    z'\defeq \projLeftSide\mathbf{T}_{-t} \{x\}.
\end{equation}
Note that 
\begin{align}\label{eq:mainth_translated_x}
\begin{split}
    z'[\ell] &= 0, \quad \text{for }  \ell \in \Nplus, \text{ and hence }z'\in \sigSpaceminus,\\
    z'[-\ell] &= x[t-\ell], \qquad \text{for } \ell \in \Nzero.
\end{split}
\end{align}
Inserting $z'$ from \eqref{eq:mainth_z} into \eqref{eq:mainth_Phi2g} and using \eqref{eq:mainth_translated_x} and \eqref{eq:mainth_G2g}, it follows that
\[
\left|\Phi\left(\{x[t-\ell]\}_{\ell=0}^{T}\right)-G(x)[t]\right|\leq \epsilon.
\]
As $x\in\sigSpaceminus$ and $t\in \Nzero$ were arbitrary, this proves \eqref{eq:rnn_output_app2G}.

\textit{Step 2:} \textit{We construct an RNN $\rnn_{\Psi}$ realizing the mapping $x \rightarrow \Phi\left(\{x[t-\ell]\}_{\ell=0}^{T}\right)_{t\in \Nzero}$.}\\
Fix $x\in\sigSpaceplus$ arbitrarily. Recall the \gls{rnn} Definition \ref{def:RNN}. The basic idea is to identify an RNN $\rnn_{\Psi}$ which, for every time step $t\in \Nzero$,
\begin{itemize}
    \item delivers the output
    \begin{equation}\label{eq:output_vec}
    y[t] = \Phi\left(\{x[t-\ell]\}_{\ell=0}^{T}\right)
    \end{equation}
    \item and memorizes the $T$ past inputs $x[t],x[t-1],\dots, x[t-T+1]$ in the hidden state vector $h[t]$, i.e.,
    \begin{equation}\label{eq:hidden_vec}
    h_\ell[t] = x[t-\ell+1], \qquad \text{for}\quad \ell\in \{1,2,\dots, T\}.
    \end{equation}
\end{itemize}
To memorize the $T$ past inputs, we note that the one-layer neural network
\begin{equation}\label{eq:nn_memorize}
    \Phi_T(z) = \begin{pmatrix}
        \mathbb{I}_T & 0_{T}
    \end{pmatrix} z, \quad \text{for }z\in \R^{T+1},
\end{equation} 
satisfies 
\begin{equation}
\begin{aligned}
    \Phi_T\left(\{x[t-\ell]\}_{\ell=0}^{T}\right) &= \begin{pmatrix}
        \mathbb{I}_T & 0_{T}
    \end{pmatrix}(x[t],x[t-1],\dots,x[t-T])^T\\
    &= (x[t],x[t-1],\dots,x[t-T+1])^T \in \R^T.
\end{aligned}
\end{equation}
Now, we apply Lemma \ref{lm:aug_depth} to augment $\Phi_T$ to depth $\depth(\Phi)$ without changing its input-output relation. This results in the \gls{relu} network $\Phi_T^*$. Then, we apply Lemma \ref{lm:para_nn} to parallelize $\Phi$ and $\Phi_T^*$ leading to the desired \gls{relu} network 
\begin{equation}\label{eq:rnn_structure}
    \Psi = P(\Phi,\Phi_T^*).
\end{equation}
By Definition \ref{def:RNN}, the corresponding \gls{rnn} $\rnnOp{\Psi}$ effects the input-output mapping according to
\begin{equation*}
    \begin{pmatrix}
        y[t] \\ h[t]
    \end{pmatrix}=\Psi\left(\begin{pmatrix}
        x[t] \\ h[t-1]
    \end{pmatrix}\right), \quad \text{for all } t \in \Nzero.
\end{equation*}

With these choices, \eqref{eq:output_vec} and \eqref{eq:hidden_vec} can now be proven by induction over $t \in \Nzero$. The base case is immediate as $x[t] =0$, for $t<0$, owing to $x\in\sigSpaceplus$ and, by Definition \ref{def:RNN}, $h[-1] = 0_{T}$. To establish the induction step, we assume that \eqref{eq:output_vec} and \eqref{eq:hidden_vec} hold for $t-1$ with $t\in \Nplus$, i.e.,
\begin{align*}
    y[t-1] &= \Phi\left(\{x[t-1-\ell]\}_{\ell=0}^{T}\right),\\
    h_\ell[t-1] &= x[t-\ell], \qquad \text{for}\quad \ell\in \{1,2,\dots, T\}.
\end{align*}
Now, for time step $t$, we note that 
\begin{equation*}
    \begin{aligned}
        \Psi\left(\begin{pmatrix}
        x[t] \\ h[t-1]
    \end{pmatrix}\right) &= P(\Phi,\Phi_T^*) \left(\begin{pmatrix}
        x[t] \\ h[t-1]
    \end{pmatrix}\right)\\
        &= \begin{pmatrix}
        \Phi\left(\{x[t-\ell]\}_{\ell=0}^{T}\right) \\[1mm] \Phi^*_T\left(\{x[t-\ell]\}_{\ell=0}^{T}\right)\end{pmatrix}\\
        & = \begin{pmatrix}            
        \Phi\left(\{x[t-\ell]\}_{\ell=0}^{T}\right) \\[1mm] (x[t],x[t-1],\dots,x[t-T+1])^T\end{pmatrix}\\
        &= \begin{pmatrix}
        \Phi\left(\{x[t-\ell]\}_{\ell=0}^{T}\right) \\[1mm] h[t]\end{pmatrix}.
    \end{aligned}
\end{equation*}

As $x$ was arbitrary, this completes the induction and thereby Step 2.

To conclude, we combine the results in Steps 1 and 2 according to 
\begin{align}
    \opMetric(\rnn_{\Psi},G)& = \sup_{x\in \sigSpaceplus}\sup_{t\in \Nzero}\left|y[t]-G(x)[t]\right|\\
    &=\sup_{x\in \sigSpaceplus}\sup_{t\in \Nzero}\left|\Phi\left(\{x[t-\ell]\}_{\ell=0}^{T}\right)-G(x)[t]\right| \label{eq:lemma_rnn2G_error_est_1}\\
    &\leq \epsilon \label{eq:lemma_rnn2G_error_est_2},
\end{align}
where \eqref{eq:lemma_rnn2G_error_est_1} follows from \eqref{eq:output_vec} and \eqref{eq:lemma_rnn2G_error_est_2} is by \eqref{eq:rnn_output_app2G}. Furthermore, we have
\begin{equation*}
\begin{aligned}
    \weightsSet(\Psi) \overset{\text{\eqref{eq:rnn_structure}, Lemma \ref{lm:para_nn}}}&{=} \weightsSet(\Phi) \cup (\weightsSet(\Phi^*_T))\\
    \overset{\text{Lemma \ref{lm:aug_depth}}}&{\subset} \weightsSet(\Phi) \cup \weightsSet(\Phi_T)\cup (-\weightsSet(\Phi_T)) \cup \{1,-1\}\\
    \overset{\text{\eqref{eq:nn_memorize}}}&{=} \weightsSet(\Phi) \cup \{1,-1\} \\
    \overset{\text{Lemma \ref{lm:NN2G_0}}}&{\subset} 2^{-2\ceil{\log (\epsilon^{-1})}} \ZZ \cap \left[-\epsilon^{-2},\epsilon^{-2}\right].
\end{aligned}
\end{equation*}
Thus, $\Psi$ has $(2, \epsilon)$-quantized weights. Finally, we get an upper bound on $\nnz(\Psi)$ according to
\begin{align*}
    \nnz(\Psi) \overset{\text{\eqref{eq:rnn_structure}, Lemma \ref{lm:para_nn}}}&{=} \nnz(\Phi) + \nnz(\Phi_T^*)\\
    \overset{\text{\eqref{eq:nn_memorize}, Lemma \ref{lm:aug_depth}}}&{\leq} \nnz(\Phi) + \nnz(\Phi_T) + T\width(\Phi_T)+2T(\depth(\Phi)-\depth(\Phi_T))\\
     \overset{\text{\eqref{eq:nn_memorize}, Lemma \ref{lm:NN2G_0}}}&{\leq} C(T+1)^{2}  \prod_{\ell=0}^{T} \left(\frac{2Dw[\ell]}{\epsilon}\frac{s+1}{s} +4\right),
\end{align*}
with the {\uc} $C>0$ chosen sufficiently large. 
\end{proof}
\section{Metric-entropy-optimal approximation of \gls{elfm} and \gls{plfm} systems}\label{sec:optimality_RNN}
So far we have characterized the description complexity of \gls{elfm} and \gls{plfm} systems based on order, type, and generalized dimension (Section \ref{sec:rate_elfm_plfm}) and we constructed \glspl{rnn} approximating general \gls{lfm} systems (Section \ref{sec:RNN2G}). We are now ready to state the main results of the paper, namely that the \glspl{rnn} we constructed are optimal for \gls{elfm} and \gls{plfm} system approximation in terms of description complexity. 

\textcolor{Iter2Color}{To this end, we first compute the number of bits needed by the canonical RNN decoder in Definition \ref{def:canonical_rnn_decoder} to obtain the \gls{rnn} constructed in Theorem \ref{th:RNN2G}, specifically its topology and quantized weights. The following result holds for general \gls{lfm} systems and will later be particularized to \gls{elfm} and \gls{plfm} systems.}

\begin{corollary}\label{col:RNN2G}
    \textcolor{Iter1Color}{The class of \gls{lfm} systems $(\LFMset{w}, \opMetric)$ is representable by the canonical RNN decoder $\decoderRNN$} with 
    \begin{equation*}
        \Ldec{\epsilon}{\decoderRNN}{\LFMset{w}}{\opMetric}  \leq C_1 M \log(M) \log(\epsilon^{-1}), 
    \end{equation*}
    where 
    \begin{align*}
        M & \defeq (T+1)^{2} \prod_{\ell=0}^{T} \left(\frac{12\bound w[\ell]}{\epsilon}\right),\\
        T &\defeq \max\left\{ \ell \in \Nzero \Bigm| w[\ell] > \frac{\epsilon}{2\bound}\right\},
    \end{align*}
    and $C_1>0$ is a {\uc}.
\end{corollary}
\begin{proof}
    Applying Theorem \ref{th:RNN2G} and setting $s=1$, it follows that there exists an $\epsilon_0>0$, such that for every $\epsilon\in (0,\epsilon_0)$ and every $G \in \LFMset{w}$, we can find an \gls{rnn} $\rnn_{\Psi}$, associated with a \gls{relu} network $\Psi \in \relunn_{T+1,T+1}$, satisfying 
    \begin{equation}\label{eq:rnn2G_proof}
    \opMetric(\rnn_{\Psi}, G)\leq \epsilon.
    \end{equation}
    Moreover, $\rnn_{\Psi}$ has $(2,\epsilon)$-quantized weights and the number of non-zero weights in $\Psi$ can be upper-bounded according to
    \begin{equation}\label{eq:rnn_nnz_bound_proof}
        \nnz(\Psi)\leq C(T+1)^{2}  \prod_{\ell=0}^{T} \left(\frac{4Dw[\ell]}{\epsilon} +4\right).
    \end{equation}
    By the definition of the canonical neural network decoder, Remark \ref{rm:bitstring}, and Definition \ref{def:canonical_rnn_decoder}, there exists a bitstring $b\in \{0,1\}^L$ with 
    \begin{equation}\label{eq:rnn2G_bitlength_proof}
        L \leq 2C_0 \nnz(\Psi) \log(\nnz(\Psi)) \log(\epsilon^{-1}),
    \end{equation}
    such that 
    \begin{equation}\label{eq:rnn2G_bitstring_proof}
        \decoderRNN(b) = \rnn_{\Psi}.
    \end{equation}
    Combining \eqref{eq:rnn2G_proof}, \eqref{eq:rnn2G_bitlength_proof}, \eqref{eq:rnn2G_bitstring_proof}, and Definition \ref{def:repre_sys}, upon noting that $G$ was chosen arbitrarily, it follows that $(\LFMset{w}, \opMetric)$ is representable by the canonical RNN decoder $\decoderRNN$, with the minimum number of required bits satisfying
    \begin{align}
        \Ldec{\epsilon}{\decoderRNN}{\LFMset{w}}{\opMetric}
        &\leq 2C_0 \nnz(\Psi) \log(\nnz(\Psi)) \log(\epsilon^{-1})\label{eq:rnn_decoder_length_0}\\
        &\leq  2C_0C(T+1)^{2} \left[\prod_{\ell=0}^{T} \left(\frac{12Dw[\ell]}{\epsilon} \right) \right] \log(\epsilon^{-1}) \nonumber\\
        &\phantom{\leq} \;\cdot \log\left( C(T+1)^2\left(\prod_{\ell=0}^{T} \left(\frac{12Dw[\ell]}{\epsilon}\right)\right) \right)\label{eq:rnn_decoder_length_1}\\
        &\leq 2C_0CC'(T+1)^{2} \left[\prod_{\ell=0}^{T} \left(\frac{12Dw[\ell]}{\epsilon} \right) \right] \log(\epsilon^{-1}) \nonumber\\
        & \phantom{\leq} \;\cdot\log\left((T+1)^{2}\prod_{\ell=0}^{T} \left(\frac{12Dw[\ell]}{\epsilon}\right)\right) \label{eq:rnn_decoder_length_2}\\
        &= C_1 M \log(M) \log(\epsilon^{-1}),\label{eq:rnn_decoder_length_3}
    \end{align}
    where \eqref{eq:rnn_decoder_length_0} is a consequence of \eqref{eq:rnn2G_bitlength_proof} and Definition \ref{def:repre_sys}, \eqref{eq:rnn_decoder_length_1} follows from \eqref{eq:rnn_nnz_bound_proof} and the fact that $2Dw[\ell]/ \epsilon > 1$, for $\ell \in \ssidedEnum{T}$, \eqref{eq:rnn_decoder_length_2} holds upon choosing the {\uc} $C'$ sufficiently large, namely s.t. 
    \begin{equation*}
        \log\left( C(T+1)^2\left(\prod_{\ell=0}^{T} \left(\frac{12Dw[\ell]}{\epsilon}\right)\right) \right) \leq C'\log\left((T+1)^2\left(\prod_{\ell=0}^{T} \left(\frac{12Dw[\ell]}{\epsilon}\right)\right) \right),
    \end{equation*}
    and \eqref{eq:rnn_decoder_length_3} follows by setting $C_1 = 2C_0CC'$.    
\end{proof}

\subsection{RNNs optimally \textcolor{black}{approximate} \gls{elfm} systems}\label{sec:rnn2elfm}
We now particularize the result in Corollary \ref{col:RNN2G} to \gls{elfm} systems, which allows us to determine the growth rate, with respect to $\epsilon$, of the minimum number of bits $\Ldec{\epsilon}{\decoderRNN}{\LFMset{\weightSeqExp{a}{b}}}{\opMetric}$ \textcolor{Iter2Color}{needed to represent $(\LFMset{\weightSeqExp{a}{b}},\opMetric)$ by the canonical RNN decoder $\decoderRNN$}. A comparison with the description complexity of $(\LFMset{\weightSeqExp{a}{b}},\opMetric)$ established in Lemma \ref{lm:dim_ELFM} then leads to the conclusion that RNNs are capable of \textcolor{black}{approximating} \gls{elfm} systems in metric-entropy-optimal manner.
\begin{theorem}
\label{th:RNN_ELFM}
Let $a\in (0,1]$ and $b>0$. The class of \gls{elfm} systems (Definition \ref{def:weight_exp}) $(\LFMset{\weightSeqExp{a}{b}},\opMetric)$ 
is optimally representable by the canonical RNN decoder $\decoderRNN$ (Definition \ref{def:canonical_rnn_decoder}).
\end{theorem}
\begin{proof}
By Corollary \ref{col:RNN2G}, the minimum number of bits $\Ldec{\epsilon}{\decoderRNN}{\LFMset{\weightSeqExp{a}{b}}}{\opMetric}$ needed to represent $(\LFMset{\weightSeqExp{a}{b}},\opMetric)$ by the canonical RNN decoder $\decoderRNN$ satisfies
\begin{equation}\label{eq:number_bit}
\Ldec{\epsilon}{\decoderRNN}{\LFMset{\weightSeqExp{a}{b}}}{\opMetric} \leq C_1 M \log(M) \log(\epsilon^{-1}),
\end{equation}
where 
\begin{align*}    
    M &\defeq (T+1)^{2} \prod_{\ell=0}^{T} \left(\frac{12\bound \weightSeqExp{a}{b}[\ell]}{\epsilon}\right),\\
    T &\defeq  \max \left\{ \ell \in \Nzero \Bigm| \weightSeqExp{a}{b}[\ell]> \frac{\epsilon}{2\bound} \right\}.
\end{align*}
By Lemma \ref{lm:dim_ELFM}, $(\LFMset{\weightSeqExp{a}{b}},\opMetric)$ is of order $1$ and type $2$, with generalized dimension $$\genDim{} = \frac{1}{2b \log (e)}.$$ 
Thus, by Lemma \ref{lm:foundamental_limit} and Definition \ref{def:metric_entropy_optimal}, it suffices to prove that 
\begin{equation}\label{eq:need_prove_exp}
    \lim_{\epsilon \rightarrow 0} \frac{\log \left(C_1 M \log(M) \log(\epsilon^{-1})\right)}{\log^2 (\epsilon^{-1})} = \frac{1}{2b \log (e)}.
\end{equation}
To this end, we first note that
\begin{align}
    \log(M) & = 2\log(T+1) + \log\left( \prod_{\ell=0}^{T} \left(\frac{12\bound \weightSeqExp{a}{b}[\ell]}{\epsilon}\right) \right)\\
    & = \bigo(\log^{(2)}(\epsilon^{-1})) + \frac{1}{2b \log (e)} \log^{2}(\epsilon^{-1}) + \smallo\left(\log^{2}(\epsilon^{-1})\right)\label{eq:est_logM}\\
    & = \frac{1}{2b \log (e)} \log^{2}(\epsilon^{-1}) + \smallo\left(\log^{2}(\epsilon^{-1})\right),
\end{align}
where \eqref{eq:est_logM} follows from Lemma \ref{lm:estimation} and $T + 1= \bigo(\log(\epsilon^{-1}))$ thanks to \eqref{eq:T_elfm}. Now, we rewrite the numerator in \eqref{eq:need_prove_exp} according to
\begin{align}
    \log \left(C_1 M \log(M) \log(\epsilon^{-1})\right) &= \log (M) + \log^{(2)}(M) + \smallo\left(\log^2(\epsilon^{-1})\right) \label{eq:est_numerator_1}\\
    &=\frac{1}{2b \log (e)} \log^{2}(\epsilon^{-1}) + \smallo\left(\log^{2}(\epsilon^{-1})\right)  \nonumber\\
    &\phantom{=} \;+ \log\left(\frac{1}{2b \log (e)} \log^{2}(\epsilon^{-1}) + \smallo\left(\log^{2}(\epsilon^{-1})\right)\right)  \\
    &= \frac{1}{2b \log (e)} \log^{2}(\epsilon^{-1}) + \smallo\left(\log^{2}(\epsilon^{-1})\right).\label{eq:est_numerator}
\end{align}

Dividing \eqref{eq:est_numerator} by $\log^2(\epsilon^{-1})$ and taking $\epsilon \rightarrow 0$, concludes the proof.
\end{proof}

\subsection{RNNs optimally \textcolor{black}{approximate} \gls{plfm} systems}
We next particularize the result in Corollary \ref{col:RNN2G} to \gls{plfm} systems and will see that
the minimum number of bits $\Ldec{\epsilon}{\decoderRNN}{\LFMset{\weightSeqPoly{p}{q}}}{\opMetric}$ required to represent $(\LFMset{\weightSeqPoly{p}{q}},\opMetric)$ by the canonical RNN decoder $\decoderRNN$ grows significantly faster (with respect to the prescribed error $\epsilon$) than for \gls{elfm} systems. This can be attributed to the fact that the memory of \gls{plfm} systems decays much more slowly.
Nonetheless, as shown next, RNNs are capable of \textcolor{black}{approximating} \gls{plfm} systems in metric-entropy-optimal manner.
\begin{theorem}
\label{th:RNN_PLFM}
    Let $q\in (0,1]$ and $p>0$. The class of \gls{plfm} systems $(\LFMset{\weightSeqPoly{p}{q}},\opMetric)$ (Definition \ref{def:weight_poly}) 
    is optimally representable by the canonical RNN decoder $\decoderRNN$ (Definition \ref{def:canonical_rnn_decoder}).
\end{theorem}
\begin{proof}
By Corollary \ref{col:RNN2G}, the minimum number of bits $\Ldec{\epsilon}{\decoderRNN}{\LFMset{\weightSeqPoly{p}{q}}}{\opMetric}$ needed to represent $(\LFMset{\weightSeqPoly{p}{q}},\opMetric)$ by the canonical RNN decoder $\decoderRNN$ satisfies
\begin{equation}\label{eq:number_bit_poly}
\Ldec{\epsilon}{\decoderRNN}{\LFMset{\weightSeqPoly{p}{q}}}{\opMetric} \leq C_1 M \log(M) \log(\epsilon^{-1}),
\end{equation}
where 
\begin{align*}    
    M &\defeq (T+1)^{2} \prod_{\ell=0}^{T} \left(\frac{12\bound \weightSeqPoly{p}{q}[\ell]}{\epsilon}\right),\\
    T &\defeq  \max \left\{ \ell \in \Nzero \Bigm| \weightSeqPoly{p}{q}[\ell]> \frac{\epsilon}{2\bound} \right\}.
\end{align*}
By Lemma \ref{lm:dim_PLFM}, $(\LFMset{\weightSeqPoly{p}{q}},\opMetric)$ is of order $2$ and type $1$, with generalized dimension $$\genDim{} = \frac{1}{p}.$$ 
Thus, by Lemma \ref{lm:foundamental_limit} and Definition \ref{def:metric_entropy_optimal}, it suffices to prove that 
\begin{equation}\label{eq:need_prove_poly}
    \lim_{\epsilon \rightarrow 0} \frac{\log^{(2)} \left(C_1 M \log(M) \log(\epsilon^{-1})\right)}{\log (\epsilon^{-1})} = \frac{1}{p}.
\end{equation}
To this end, we first note that
\begin{align}
    \log(M) & = 2\log(T+1) + \log\left( \prod_{\ell=0}^{T} \left(\frac{12\bound \weightSeqPoly{p}{q}[\ell]}{\epsilon}\right) \right)\\
    & = \bigo(\log(\epsilon^{-1})) + \asymporder(\epsilon^{-1/p})\label{eq:est_logM_poly}\\
    &= \asymporder(\epsilon^{-1/p}),
\end{align}
where \eqref{eq:est_logM_poly} follows from Lemma \ref{lm:estimation_poly} and \eqref{eq:T_plfm}. Now, we rewrite the numerator in \eqref{eq:need_prove_poly} according to
\begin{align}
    \log^{(2)} \left(C_1 M \log(M) \log(\epsilon^{-1})\right) = ~& \log\left(\log (M) + \log^{(2)}(M) + \smallo\left(\log(\epsilon^{-1})\right) \right)\label{eq:est_numerator_poly_1}\\
    =~ & \log\left(\asymporder(\epsilon^{-1/p}) + \log\left(\asymporder(\epsilon^{-1/p})\right) + \smallo\left(\log(\epsilon^{-1})\right)\right)\\ 
    =~ & \log\left(\asymporder(\epsilon^{-1/p})\right)\label{eq:est_numerator_poly_2},
\end{align}
where 
\eqref{eq:est_numerator_poly_2} follows from 
\begin{equation*}
    \asymporder(\epsilon^{-1/p}) + \log\left(\asymporder(\epsilon^{-1/p})\right) + \smallo\left(\log(\epsilon^{-1})\right) = \asymporder(\epsilon^{-1/p}).
\end{equation*}

Finally, dividing \eqref{eq:est_numerator_poly_2} by $\log(\epsilon^{-1})$ and taking $\epsilon \rightarrow 0$, concludes the proof.
\end{proof}
\section{Conclusion}

Returning to Table \ref{tbl:nn}, we note that it can be complemented by our results for 
\gls{elfm} and \gls{plfm} systems as summarized in Table \ref{tbl:rnn} below. 
As both classes of systems in Table \ref{tbl:rnn} are optimally representable by the canonical \gls{rnn} decoder (Definition \ref{def:canonical_rnn_decoder}), we can conclude that, remarkably, the metric-entropy-optimal universality of 
\gls{relu} networks extends from function classes to nonlinear dynamical systems, 
in the latter case simply by embedding the \gls{relu} network in a recurrence.

{
\linespread{1.5}
\begin{table}[H]
\centering
\begin{tabular}{|ll|ll|l|l|l|}
\hline
 & Metric &                                                                         $\mathcal{C}$                      &                                & $\order$ & $\type$ & $\genDim$                                        \\ \hline
       $\{\R^\ZZ\to\R^\ZZ\}$        &  $\opMetric$ (Def. \ref{def:metric})     &  \gls{elfm} systems (Section \ref{sec:ELFM_rate}) &   $\LFMset{\weightSeqExp{a}{b}}$                                                                                                                                            &  1    &  2    &     $\frac{1}{2b\log{(e)}}$                                                         \\
       $\{\R^\ZZ\to\R^\ZZ\}$        &  $\opMetric$ (Def. \ref{def:metric})     &  \gls{plfm} systems (Section \ref{sec:PLFM_rate})    &    $\LFMset{\weightSeqPoly{p}{q}}$                                                                                                                                            & 2   &  1    &     $\frac{1}{p}$                                                         \\
       \hline
\end{tabular}
\vspace{0.5em}
\caption{Scaling behavior of the covering numbers (Definition \ref{def:order_type_dim}) for classes of nonlinear dynamical systems. }
\label{tbl:rnn}
\end{table}
}

Finally, we remark that many of the results in this paper apply to general weight sequences $w[\cdot]$, specifically the bounds 
in Section \ref{sec:app_rate_lfms} on the exterior covering number of $(\LFMset{w}, \opMetric)$ as well as the \gls{rnn} constructions in Section \ref{sec:RNN2G}.
\clearpage

\appendix
\section{Representing \gls{relu} networks by bitstrings}\label{app:define_nn_bitstring}
\newcommand{\qnet}{{\Phi}}
\newcommand{\Mtmp}{M}

\begin{definition}\label{def:bitstring_orga}
    
Let $\qnet$ be a \gls{relu} network with $(m, \epsilon)$-quantized weights. Denote the number of non-zero weights by $\Mtmp \coloneqq \nnz(\qnet)$ 
and the depth by $L\coloneqq \depth(\qnet)$. 
We organize the bitstring representation of $\qnet$ in $6$ segments as follows.

\begin{enumerate}[label=\lbrack Segment \arabic*\rbrack, leftmargin=*]
    \item \label{seg1} The bitstring starts with $\Mtmp{}$ 1's followed by a single 0. 
    \item $L$ is specified in binary representation. As $L \leq \Mtmp{}$, it suffices to allocate $\ceil{\log(\Mtmp{})}$ bits. 
    \item $N_0, \dots, N_L \leq M$ are specified in binary representation using a total of $(\Mtmp{}+1) \ceil{\log(\Mtmp{})}$ bits.
    \item The topology of the network, i.e., the locations of the non-zero entries in the $A_\ell$ and $b_\ell$, $\ell\in \{1, \dots, L\}$, is encoded as follows. We denote the bitstring corresponding to the binary representation of an integer $i\in\{1, \dots, M\}$ by $b(i)\in \{0, 1\}^{\ceil{\log(\Mtmp{})}}$.
    For $\ell \in \{1, \dots, L\}$, $i \in \{1, \dots, N_\ell\}$, $j\in\{1, \dots, N_{\ell-1}\}$, a non-zero entry $(A_\ell)_{ij}$ is indicated by $[b(\ell), b(i), b(j)]$ and a non-zero entry $(b_\ell)_i$ by $[b(\ell), b(i), b(i)]$.
    Thus, encoding the topology of the network requires a total of $3\ceil{\log(\Mtmp{})}M$ bits.
    \item The quantity $m\ceil{\log(\epsilon^{-1})}$ is represented by a bitstring of that many 1's followed by a single 0.
    \item The value of each non-zero weight and bias is represented by a bitstring of length $B_\epsilon = 2m\ceil{\log{\left(\epsilon^{-1}\right)}}+1$.
\end{enumerate}
The overall length of the bitstring is now given by 
\begin{align}\label{eq:exact_bitstring_length}
    \underbrace{\Mtmp{}+1}_{\textrm{Segment 1}} +
    \underbrace{\ceil{\log(\Mtmp{})}}_{\textrm{Segment 2}} + 
    \underbrace{(\Mtmp{}+1) \ceil{\log(\Mtmp{})}}_{\textrm{Segment 3}} + 
    \underbrace{3\ceil{\log(\Mtmp{})} M}_{\textrm{Segment 4}} + 
    \underbrace{m\ceil{\log(\epsilon^{-1})}+1}_{\textrm{Segment 5}} + 
    \underbrace{\Mtmp{} B_\epsilon }_{\textrm{Segment 6}} \hspace{-0.2cm}.
\end{align}

The \gls{relu} network $\qnet$ can be recovered by sequentially reading out $M$,$L$, the $N_\ell$, the topology, the quantity $m\ceil{\log(\epsilon^{-1})}$, and the quantized weights from the overall bitstring.  It is not difficult to verify that the bitstring is crafted such that this yields unique decodability. 
\end{definition}



\section{Comparison with \cite{deepAT2019}}\label{app:compare_to_dennis}
As mentioned in Section \ref{sec:def_opt}, compared to \cite{deepAT2019}, we use refined notions of massiveness of sets, as the system classes we consider are significantly more massive than the function classes dealt with in \cite{deepAT2019}.
We next detail how the results from \cite{deepAT2019} fit into our framework. 

In \cite{deepAT2019} the scaling behavior of covering numbers is quantified in terms of the optimal exponent $\optExponent$ \cite[Definition IV.1]{deepAT2019}. The following result relates $\optExponent$ to the generalized dimension (Definition \ref{def:order_type_dim}) employed here.

\begin{lemma}\label{lm:gen_dimension_vs_exponent}
Let $d\in \Nplus,\; \Omega\subseteq \R^d$, and let $\mathcal{C}\subset \Ltwo(\Omega)$ be compact and such that the optimal exponent $\optExponent(\entSet)$ according to \cite[Definition IV.1]{deepAT2019} is finite and non-zero. Then, $\entSet$ is, with respect to the metric $\rho(f,g):= \norm{f-g}_{\Ltwo(\Omega)}$, of order $\kappa=1$, type $\type=1$, and generalized dimension 
\[
    \genDim = \frac{1}{\optExponent}.
\]    
\end{lemma}
\begin{proof}
\newcommand{\tmpDim}{D}
For order $\order=1$ and type $\type=1$, the generalized dimension is given by
\begin{equation}\label{teq:95}
    \genDim =\limsup _{\epsilon \rightarrow 0} \frac{\log ^{(2)} \covering^{\text{ext}} (\epsilon ; \entSet, \genMetric)}{\log \left(\epsilon^{-1}\right)}.
\end{equation}
We note that by \cite[Remark 5.10, Equation (5.5)]{grohs2015optimally} and Lemma \ref{lm:number_relation}, it holds that
\begin{equation}
    \optExponent = \sup \left\{ \gamma > 0  : \log \coveringExt(\epsilon ; \entSet, \genMetric) \in \bigo({\epsilon^{-1/\gamma}}), \epsilon \to 0 \right\}. 
\end{equation}
We first establish that $\genDim\leq \frac{1}{ \optExponent}$.  To this end,  fix $\Delta>0$ arbitrarily and observe that $\log \coveringExt(\epsilon ; \entSet, \genMetric) \in \bigo({\epsilon^{-1/(\optExponent - \Delta)}})$. Hence, there exist  $\epsilon_0, C>0$ such that 
\[
    \log \coveringExt(\epsilon ; \entSet, \genMetric) \leq C \epsilon^{-1/(\optExponent-\Delta)}, \qquad \forall \epsilon \in (0, \epsilon_0),
\]
and thus 
\begin{align}
   \genDim = \limsup _{\epsilon \rightarrow 0} \frac{\log ^{(2)} \covering^{\text{ext}} (\epsilon ; \entSet, \genMetric)}{\log \left(\epsilon^{-1}\right)} 
   \leq \limsup _{\epsilon \rightarrow 0} \left( \frac{1}{\optExponent-\Delta} +  \frac{\log (C)}{\log \left(\epsilon^{-1}\right)} \right)
   =  \frac{1}{\optExponent-\Delta}.
\end{align}
As $\Delta>0$ was arbitrary, we have established that $\genDim \leq \frac{1}{\optExponent}$.

Next, we show that $\genDim \geq \frac{1}{\optExponent}$. Again, fix $\Delta >0$ arbitrarily. By \eqref{teq:95}, there exists an $\epsilon_0 > 0$ such that for all $ \epsilon \in (0, \epsilon_0)$,
\begin{equation}
     \frac{\log ^{(2)} \covering^{\text{ext}} (\epsilon ; \entSet, \genMetric)}{\log \left(\epsilon^{-1}\right)} \leq \genDim + \Delta.
\end{equation}
This implies 
\begin{equation}
    \log \covering^{\text{ext}} (\epsilon ; \entSet, \genMetric) \leq \epsilon^{-(\genDim+\Delta)}, \qquad \forall \epsilon \in (0, \epsilon_0),
\end{equation} 
and thus $\log \covering^{\text{ext}} (\epsilon ; \entSet, \genMetric) \in \bigo({\epsilon^{-1/(\genDim+\Delta)^{-1}}})$. Hence, $\optExponent \geq (\genDim+\Delta)^{-1}$. As $\Delta$ was arbitrary, this finalizes the proof. 
\end{proof}

Table \ref{tbl:nn} in the present paper now follows from \cite[Table I]{deepAT2019} by application of Lemma \ref{lm:gen_dimension_vs_exponent}. Furthermore, \cite{deepAT2019} shows through the transference principle \cite[Section VII]{deepAT2019} that a wide range of function classes, including those in Table \ref{tbl:nn}, are \emph{optimally representable by neural networks} \cite[Definition VI.5]{deepAT2019}. The following Lemma hence allows us to conclude that every row in Table \ref{tbl:nn} is optimally representable (according to Definition \ref{def:metric_entropy_optimal}) by the canonical neural network decoder.

\begin{lemma}\label{lm:dennis_nn_optimal}
Let $d\in \Nplus,\; \Omega\subseteq \R^d$, and let $\mathcal{C}\subset \Ltwo(\Omega)$ be compact. If the function class $\entSet\subset \Ltwo(\Omega)$ is \emph{optimally representable by neural networks} according to \cite[Definition VI.5]{deepAT2019}, then $(\entSet, \genMetric)$ is optimally representable by the canonical neural network decoder with respect to the metric $\rho(f,g)\coloneqq \norm{f-g}_{\Ltwo{(\Omega)}}$.
\end{lemma}
\begin{proof}
\newcommand{\approxNN}{\widetilde{\Phi}_{\epsilon, f}}
\newcommand{\mWeights}{M_\epsilon}
First, we note that by Lemma \ref{lm:gen_dimension_vs_exponent}, $\mathcal{C}$ is of order $\kappa=1$, type $\type=1$, and has generalized dimension $\genDim = \frac{1}{\optExponent{(\mathcal{C})}}$. By \cite[Definition VI.5]{deepAT2019}, we have $\optExponent(\mathcal{C}) = \optExponentNN(\mathcal{C})$,
with $\optExponentNN(\mathcal{C})$ as per \cite[Definition VI.3]{deepAT2019}.
Next, fix $\Delta>0$ arbitrarily. Now, following the proof of \cite[Theorem VI.4, p. 2602]{deepAT2019} with $\optExponentNN - \Delta$ in place of $\gamma$, we can conclude the existence of a polynomial $\pi^*$, a constant $C$, and a mapping $\Psi: (0, \frac{1}{2}) \times \mathcal{C} \rightarrow \mathcal{N}_{d,1}$ with the following properties. For every $f\in\mathcal{C}$ and every $\epsilon \in (0, \frac{1}{2})$, the network $\approxNN = \Psi( \epsilon, f)$ has $(\ceil{\pi^*(\log(\epsilon^{-1}))}, \epsilon)$-quantized weights and satisfies
\[
\norm{f-\approxNN}_{\Ltwo(\Omega)} \leq \epsilon \quad \text{and} \quad \nnz(\approxNN) \leq C\epsilon^{-1/(\optExponentNN - \Delta)} =: \mWeights{}.
\] 
Thus, $\approxNN$ can, according to Remark \ref{rm:bitstring}, be reconstructed uniquely by the canonical neural network decoder $\nnDec$ from a bitstring of length no more than
\[
    C_0 \ceil{\pi^*(\log(\epsilon^{-1}))} \log(\epsilon^{-1}) \mWeights{} \log(\mWeights{}) .
\]
Therefore, $\mathcal{C}$ is representable by the canonical neural network decoder $\nnDec$ with
$
    \Ldec{\epsilon}{\decoderNN}{\mathcal{C}}{\genMetric} \leq C_0 \mWeights{} \log(\mWeights{}) \log^{q_0}(\epsilon^{-1}),
$
where $q_0$ is a constant depending on $\pi^*$ only,
and hence
\begin{align}
    \log\Ldec{\epsilon}{\decoderNN}{\mathcal{C}}{\genMetric}  \leq \frac{1}{\optExponentNN - \Delta}\log(\epsilon^{-1}) + \smallo(\log(\epsilon^{-1})).
\end{align}
As $\Delta>0$ was arbitrary, we thus get
\begin{align}
    \limsup_{\epsilon\to0} \frac{\Ldec{\epsilon}{\decoderNN}{\mathcal{C}}{\genMetric}}{\log(\epsilon^{-1})} \leq \frac{1}{\optExponentNN},
\end{align}
which, together with Lemma \ref{lm:foundamental_limit} and the fact that $\genDim = \frac{1}{\optExponent} = \frac{1}{\optExponentNN}$ implies that $\mathcal{C}$ is optimally representable by the canonical neural network decoder according to Definition \ref{def:metric_entropy_optimal}.
\end{proof}

\section{Proofs}\label{sec:app_proof}
\subsection{Proof of Lemma \ref{lm:foundamental_limit}}\label{sec:proof_fundamental_limit}
\begin{proof}[\unskip\nopunct]
\newcommand{\tmpL}{L(\epsilon)}
To simplify notation, we let $\tmpL \coloneqq  \Ldec{\epsilon}{\decoderGen}{\entSet}{\genMetric} $.
We first establish that
\begin{equation}\label{eq:lowerbound_bitlength}
    \tmpL{} \geq \log \left(N^{\text{ext}}(\epsilon;\entSet, \genMetric)\right) -1, \quad \forall \epsilon > 0.
\end{equation}
By way of contradiction, assume that 
\begin{equation*}
    \quad 2^{\tmpL{} +1} < N^{\text{ext}}(\epsilon;\entSet, \genMetric), \quad \textrm{for some } \epsilon > 0.
\end{equation*}
It then follows from Definition \ref{def:representable} that for this $\epsilon$, for every $f\in\entSet$, there is an integer $\ell \leq {\tmpL{} }$ and a bitstring $\bitstring_f\in \{0,1\}^{\ell}$, such that $\genMetric(\decoderGen(\bitstring_f),f)\leq \epsilon$.
This directly implies that the set
\begin{equation*}
    \coveringset \defeq \left \{ \decoderGen(\bitstring)  \middle|\, \bitstring \in \bigcup_{\ell=0}^{\tmpL{} } \{0,1\}^\ell \right \},
\end{equation*}
is an $\epsilon$-net for $\entSet$. Furthermore, 
\begin{equation*}
    |\coveringset| \leq \sum_{\ell=0}^{{\tmpL{} }} 2^\ell = \frac{2^{{\tmpL{} }+1}-1}{2-1}  \leq 2^{{\tmpL{} }+1} < N^{\text{ext}}(\epsilon;\entSet, \genMetric).
\end{equation*}
Hence, $\coveringset$ is an $\epsilon$-net of cardinality strictly smaller than $N^{\text{ext}}(\epsilon;\entSet, \genMetric)$, which stands in contradiction to Definition \ref{def:both_covering_numbers} and so \eqref{eq:lowerbound_bitlength} must hold. This, in turn, implies that
\begin{align*}
    \limsup_{\epsilon \rightarrow 0} \frac{\log ^{(\kappa)} \tmpL{} }{\log ^{\type}(\epsilon^{-1})} &\geq  \limsup_{\epsilon \rightarrow 0}
    \frac{\log^{(\kappa)} \left(\log  \left(N^{\text{ext}}(\epsilon;\entSet, \opMetric)\right)  - 1 \right)}{\log ^{\type}(\epsilon^{-1})}\\
    &=  
    \limsup_{\epsilon \rightarrow 0} \frac{\log^{(\kappa+1)}   N^{\text{ext}}(\epsilon;\entSet, \opMetric)}{\log ^{\type}(\epsilon^{-1})}
    =  
    \genDim{},
\end{align*}
as desired.
\end{proof}

\subsection{Proof of Lemma \ref{lem:sup_over_N_same}}\label{app:proof_sub_over_N_same}
\begin{proof}[\unskip\nopunct]
Fix $G, G'\in \LFMset{w}$.
First, note that
\begin{align}\label{proof:isometry1_lower}
    \opMetric(G, G') &= \sup_{x \in \sigSpaceplus} \sup_{t\in \Nzero} |(Gx)[t] - (G'x)[t]| \nonumber\\ 
    & \leq  \sup_{x \in \sigSpace} \sup_{t\in\ZZ} |(Gx)[t] - (G'x)[t]|.
\end{align}
Next, arbitrarily fix $\Delta>0$. 
By definition of the supremum, it follows that there exist $x_0\in \sigSpace$ and $ e\in \ZZ$, such that
\begin{equation}\label{eq:sup_def}
    \sup_{x \in \sigSpace} \sup_{t\in\ZZ} |(Gx)[t] - (G'x)[t]| - \Delta/2 \leq \left| (Gx_0)[e] - (G'x_0)[e] \right|.
\end{equation}
Since $\lim_{t\rightarrow \infty}w[t]=0$, there exists  $T>0$ so that $w[t] \leq \Delta / (4\bound)$, for all $t > T$.
Next, define 
\begin{equation*}
    x_1[t] = x_0[t-(T-e)] \in \sigSpace \quad \text{and} \quad
    y[t] = x_1[t]\cdot \indArg{t \geq 0} \in \sigSpaceplus.
\end{equation*}
We then get
\begin{equation}\label{teq:sig_difference}
    \sup_{\timeShift\geq 0} w[\timeShift]|(x_1[T-\timeShift]-y[T-\timeShift])| 
    = \sup_{\timeShift > T} w[\timeShift]|x_1[T-\timeShift]|
    \leq \sup_{\timeShift> T} w[\timeShift]\bound \leq \Delta / 4,
\end{equation}
where we used that $|x_1[\cdot]|\leq D$ by Definition \ref{def:signal_set}. We furthermore obtain 
\begin{align}
    |(Gx_1)[T] - (G'x_1)[T]| &\leq |(Gx_1)[T] - (Gy)[T]| 
    + |(G'x_1)[T] - (G'y)[T]| \nonumber\\
    &\phantom{\leq} \; + |(Gy)[T] - (G'y)[T]| \label{teq:triang}\\
    & \leq \Delta / 4 + \Delta / 4 + |(Gy)[T] - (G'y)[T]| \label{teq:use_lip} \\
    &= \Delta / 2 + |(Gy)[T] - (G'y)[T]| \label{eq:lip_inequality},
\end{align}
where in \eqref{teq:triang} we used the triangle inequality and in \eqref{teq:use_lip} we invoked the Lipschitz fading-memory property of $G$ and $G'$ in combination with \eqref{teq:sig_difference}.
Next, note that by time-invariance of $G$ and $G'$, it holds that
\begin{align}
\begin{split}
    (Gx_1)[T] &= (G\shiftOp{(T-e)}x_0)[T] = (\shiftOp{(T-e)}G x_0)[T] = (G x_0)[e],\label{eq:shift_G}\\
    (G'x_1)[T] &= (G'\shiftOp{(T-e)}x_0)[T] = (\shiftOp{(T-e)}G' x_0)[T] = (G' x_0)[e].
\end{split}
\end{align}

Combining all these results yields
\begin{align}\label{proof:isometry1_upper}
    \sup_{x \in \sigSpace} \sup_{t\in\ZZ} |(Gx)[t] - (G'x)[t]| - \Delta/2 
    \overset{\eqref{eq:sup_def}}&{\leq} \left| (Gx_0)[e] - (G'x_0)[e] \right| \nonumber\\
    \overset{\eqref{eq:shift_G}}&{=} \left| (Gx_1)[T] - (G'x_1)[T] \right| \nonumber\\
    \overset{\eqref{eq:lip_inequality}}&{\leq} \Delta / 2 + \left| (Gy)[T] - (G'y)[T] \right| \nonumber\\
    \overset{y\in\sigSpaceplus}&{\leq} \Delta/2 + \sup_{x \in \sigSpaceplus} \sup_{t\in \Nzero} |(Gx)[t] - (G'x)[t]| \nonumber\\
    \overset{\textrm{Def. \ref{def:metric}}}&{=} \Delta/2 + \opMetric(G,G'). 
\end{align}
As $\Delta>0$ was arbitrary, we have thus established that 
\[
\sup_{x \in \sigSpace} \sup_{t\in\ZZ} |(Gx)[t] - (G'x)[t]| \leq \opMetric(G,G'),
\]
which, together with \eqref{proof:isometry1_lower}, completes the proof. 
\end{proof}

\subsection{Proof of Lemma \ref{lem:g0_g_isomorphism}}\label{app:proof_g0_g_isomorphism}
\begin{proof}[\unskip\nopunct]

Recall the projection operator $\projLeftSide: \sigSpace \rightarrow \sigSpaceminus$ defined according to
    \[
        (\projLeftSide x)[t] = x[t]\cdot \indArg{t \leq 0}.
    \]
First, we need to show that $\isometry: \LFMsetZero{w}\to \LFMset{w}$ given by
\[
    g \to \isometryOf{g}, \quad \textrm{with }  \isometryOfEvaluated{g}{x}[t] = g (\projLeftSide \shiftOp{-t} x),
\]
is a well-defined mapping from $\LFMsetZero{w}$ to $\LFMset{w}$, i.e., we need to verify that for every $g \in \LFMsetZero{w}$, indeed $\isometryOf{g} \in \LFMset{w}$. This will be done by establishing that $\isometryOf{g}$ satisfies the conditions of Definition \ref{def:lfm_sys_set}.
We first verify that $\isometryOf{g}$ is causal. Note that for every $T\in \ZZ$, for every $x,x'\in \sigSpace$ with $x[t]=x'[t]$, $\forall t\leq T$, it holds that $\projLeftSide\shiftOp{-T} x = \projLeftSide\shiftOp{-T} x'$, and hence
\[
\isometryOfEvaluated{g}{x}[T] = g(\projLeftSide\shiftOp{-T} x) = g(\projLeftSide\shiftOp{-T} x') = \isometryOfEvaluated{g}{x'}[T] .
\]
Thus, $\isometryOf{g}$ is, indeed, causal according to Definition \ref{def:causality}.
Next, we verify time-invariance as follows:
\begin{align*}
    (\shiftOp{\tau} (\isometryOf{g}(x)))[t] &= \isometryOfEvaluated{g}{x}[t-\tau] \\
    & = g(\projLeftSide \shiftOp{\tau-t} x) \\
    & = g(\projLeftSide \shiftOp{-t} \shiftOp{\tau}x) \\
    & = \isometryOfEvaluated{g}{\shiftOp{\tau}x}[t].
\end{align*}
The \acrlong{lfm} property of $\isometryOf{g}$ according to Definition \ref{def:lip_fading_memory} is established by noting that 
\begin{align}
    &\left|\isometryOfEvaluated{g}{x}[t] - \isometryOfEvaluated{g}{x'}[t]\right| \\
    &= |g(\projLeftSide \shiftOp{-t} x) - g(\projLeftSide \shiftOp{-t} x')| \notag\\ 
    & \leq \sup_{\tau\geq 0} w[\tau] |((\projLeftSide \shiftOp{-t} x)[-\tau] - (\projLeftSide \shiftOp{-t} x')[-\tau])| \label{teq:use_g0_in_G0} \\
    & = \sup_{\tau\geq 0} w[\tau] |( x[t-\tau] - x'[t-\tau])|, \quad \forall t \in \ZZ, \forall x,x' \in \sigSpace,\notag
\end{align}
where in \eqref{teq:use_g0_in_G0} we used \eqref{eq:nonLinSpace}. Finally, we observe that, by \eqref{eq:nonLinSpace}, 
\[
\isometryOfEvaluated{g}{0}[t]= g(\projLeftSide \shiftOp{-t} 0) = g(0) = 0,\; \forall t\in\ZZ .
\] 

In summary, we have thus shown that $\isometryOf{g}\in\LFMset{w}$ and hence $\isometry$ is, indeed, well-defined. Furthermore,  we have 
\begin{align}
    \opMetric(\isometryOf{g}, \isometryOf{g'}) & =
    \sup_{x \in \sigSpace} \sup_{t\in\mathbb{Z}} \left|\isometryOfEvaluated{g}{x}[t] - \isometryOfEvaluated{g'}{x}[t]\right|\label{teq:use_lemmaX}\\
    &= \sup_{x \in \sigSpace} \sup_{t\in\mathbb{Z}} |g(\projLeftSide \shiftOp{-t}x) - g'(\projLeftSide \shiftOp{-t}x)|\notag\\
    & = \sup_{y \in \sigSpace} |g(\projLeftSide y) - g'(\projLeftSide y)|\label{teq:shift_invariant}\\
    &= \sup_{z \in \sigSpaceminus} |g(z) - g'(z)|\notag\\
    &= \metricZero(g, g'), \qquad \forall g \in \LFMsetZero{w}, \forall g' \in \LFMsetZero{w}, \notag
\end{align}
where in \eqref{teq:use_lemmaX} we invoked Lemma \ref{lem:sup_over_N_same} and in \eqref{teq:shift_invariant} we used that $\sigSpace{}$ is closed under time shifts. This establishes that $\isometry{}$ is isometric and consequently injective.

Next, we prove that $\isometry$ is surjective. To this end, we fix $G \in \LFMset{w}$ arbitrarily, consider 
\begin{equation}\label{def:original_g0}
    g(s) \defeq (Gs)[0], \quad \forall s\in S_-,
\end{equation}
and show that $g \in \LFMsetZero{w}$ as well as $\isometryOf{g}=G$. First, we establish that $g\in\LFMsetZero{w}$.
The Lipschitz property of $g$ can be verified according to
\begin{align}
    |g(x) - g(x')| &= |(Gx)[0] - (Gx')[0]| \notag \\
    &\leq \sup_{\tau \geq 0} w[\tau]|(x[-\tau] - x'[-\tau])| \label{teq:use_lipschitz_property_of_G}\\
    &=\|x-x'\|_w, \quad \forall x,x'\in\sigSpaceminus, \label{teq:use_def_w_norm}
\end{align}
where in \eqref{teq:use_lipschitz_property_of_G} we used the fact that $G\in\LFMset{w}$ has Lipschitz fading memory (Definition \ref{def:lip_fading_memory}) and \eqref{teq:use_def_w_norm} is by \eqref{eq:w_norm}. Furthermore, we have $g(0) = (G0)[0] = 0$ and hence $g \in \LFMsetZero{w}.$ 

It remains to show that $\isometryOf{g} = G$. To this end, we fix $x \in \sigSpace$ and $t \in \mathbb{Z}$, both arbitrarily, and prove that
\[
\isometryOfEvaluated{g}{x}[t] = (Gx)[t].
\]
First, note that $(\projLeftSide \shiftOp{-t} x)[t'] = (\shiftOp{-t} x)[t']$, for all $t'\leq 0$. By causality of $G$, we conclude that $(G\projLeftSide \shiftOp{-t} x)[0] = (G \shiftOp{-t} x)[0]$, which yields
\begin{align}
    \isometryOfEvaluated{g}{x}[t]  &= g(\projLeftSide\shiftOp{-t} x) \notag 
    \\
    &= (G\projLeftSide\shiftOp{-t} x) [0] \label{teq:use_def_g0}\\
    &= (G\shiftOp{-t} x) [0] = (\shiftOp{-t}G x) [0]\label{teq:use_timeinvariance}\\
    &= (G x) [t], \notag
\end{align}
where in \eqref{teq:use_def_g0} we used \eqref{def:original_g0}, and in \eqref{teq:use_timeinvariance} we invoked the causality and the time-invariance of $G$. As $x\in\sigSpace$ and $t\in\ZZ$ were arbitrary, this proves that $\isometryOf{g} = G$. Furthermore, since $G\in\LFMset{w}$ was arbitrary, we have established that $\isometry$ is surjective. Thus, $\isometry$ is, indeed, an isometric isomorphism between $(\LFMset{w}, \opMetric)$ and $(\LFMsetZero{w}, \metricZero)$. Application of Lemma \ref{lm:ent_inv} then yields $\covering(\epsilon; \LFMsetZero{w}, \metricZero) = \covering(\epsilon; \LFMset{w}, \opMetric)$ and  $\packing(\epsilon; \LFMsetZero{w}, \metricZero) = \packing(\epsilon; \LFMset{w}, \opMetric)$.
\end{proof}

\subsection{Proof of Lemma \ref{lm:packing_number}}
\label{app:proof_packing_number}

The proof relies on the following auxiliary result.
\begin{lemma}\label{lem:packing_signal}
    \letWbeWeightSequence{} The packing number of $\sigSpaceminus$ w.r.t. $\norm{\cdot}_w$ satisfies
    \[
           \packing(\epsilon; \sigSpaceminus, \norm{\cdot}_w) \geq \prod_{\ell=0}^T \ceil{\frac{2\bound w[\ell]}{\epsilon}},
    \]
    with $T \defeq max\{T' \in \Nzero \mid w[T'] > \frac{\epsilon}{2\bound}\}$.
\end{lemma}
\begin{proof}
    For $t \in \ssidedEnum{T}$, we let $N_{t} := \ceil{\frac{2\bound{}w[t] }{\epsilon}}\in \left[\frac{2\bound{}w[t] }{\epsilon}, \frac{2\bound{}w[t] }{\epsilon} + 1 \right)$ and $\delta_t = \frac{2D}{N_t - 1 } > \frac{\epsilon}{w[t]} $. Next, we define the set 
    \begin{align}
        \packingset & = \left\{ x_{i_0,\dots, i_{T}} \mid  i_t \in \ssidedEnum{N_t-1}\hspace{-1pt}, \textrm{ for } t \in \ssidedEnum{T} \right\},\textrm{ where}\\
        x_{i_0,\dots, i_{T}} &\coloneqq \begin{cases}
            -D + i_t\delta_t, & -T \leq t \leq 0 \\
            0, & \text{else}
        \end{cases},
    \end{align}
    and show that $\packingset$ constitutes an $\epsilon$-packing for $(\sigSpaceminus, \norm{\cdot}_w)$. First, we establish that $\packingset \subset \sigSpaceminus$ by verifying that, for all $x_{i_0,\dots, i_{T}} \in \packingset$, $x_{i_0,\dots, i_{T}} [t] \in [-D, D]$, for all $t\in\ZZ$, and $x_{i_0,\dots, i_{T}} [t] = 0$, for all $t\in \Nplus$. Indeed, for $t \in\{ -T, \dots, 0\} $, we have
    \[
    -D \leq x_{i_0,\dots, i_{T}}[t] = -D + i_t \delta_t \leq -D + (N_t-1)\delta_t = -D + 2D = D,
    \]
    and for $t \notin \{-T, \dots, 0\}$, 
    \[ 
        x_{i_0,\dots, i_{T}}[t] = 0.
    \]
    Next, we show that for distinct $x_{i_0,\dots, i_{T}}, x_{j_0,\dots, j_{T}} \in \packingset$ (i.e., there is at least one $\ell \in \ssidedEnum{T}$ such that $i_\ell \neq j_\ell$), it holds that $\norm{x_{i_0,\dots, i_{T}} - x_{j_0,\dots, j_{T}}}_w > \epsilon$.
    Indeed, for any such $\ell$, we have
    \begin{align}
        \norm{x_{i_0,\dots, i_{T}} - x_{j_0,\dots, j_{T}}}_w &= \sup_{t \in \Nzero} w[t]\left|\left(x_{i_0,\dots, i_{T}}[-t] - x_{j_0,\dots, j_{T}}[-t]\right)\right| \notag\\
        &\geq  w[\ell]\left|\left(x_{i_0,\dots, i_{T}}[-\ell] - x_{j_0,\dots, j_{T}}[-\ell]\right)\right|\notag\\
        &=  w[\ell]\left|\left(-D + i_\ell \delta_\ell + D - j_\ell \delta_\ell\right)\right|\notag\\
        &= (i_\ell - j_\ell)\delta_\ell w[\ell] > \epsilon,\label{teq:151}
    \end{align}
    where \eqref{teq:151} follows from $\delta_\ell> \frac{\epsilon}{w[\ell]}$.
    This establishes that $\packingset$, indeed, constitutes an $\epsilon$-packing for $(\sigSpaceminus, \norm{\cdot}_w)$, and  we therefore have
    \[
        \packing(\epsilon; \sigSpaceminus, \norm{\cdot}_w) \geq |\packingset| = \prod_{\ell=0}^T N_t = \prod_{\ell=0}^T \ceil{\frac{2\bound w[\ell]}{\epsilon}}. 
    \]
\end{proof}

\begin{proof}[Proof of Lemma \ref{lm:packing_number}]

\newcommand{\tmparg}{\mu} 

Let $M \defeq \prod_{\ell=0}^T \ceil{\frac{2\bound w[\ell]}{{\epsilon}}}$ and take $\{x_1, \dots, x_M\}$ to be an $\epsilon$-packing for $(\sigSpaceminus, \norm{\cdot}_w)$ according to Lemma \ref{lem:packing_signal}. Hence, with $\delta := \min_{\ell \neq j} \norm{x_\ell-x_j}_w > \epsilon$, there exists an $ \epsilon' \in (\epsilon, \delta)$.  Next, define balls of radius ${\epsilon ' }/2$ centered at the packing points $\{x_1, \dots, x_M\}$
according to $\sigBall_\ell = \{ x[\cdot] \in \sigSpaceminus \hspace{-2pt} \mid \hspace{-3pt} \norm{x-x_\ell}_w \leq {\epsilon ' }/2\}$, for $\ell \in \{1,\dots, M\}$.
We now show that these balls are non-overlapping. Indeed, assuming that, for $j\neq\ell$, there exists an $x \in \sigSpaceminus$ with $\norm{x-x_j}_w<{\epsilon ' /2} $ and $\norm{x-x_\ell}_w < {\epsilon ' }/2$, leads to the contradiction 
\[
    {\epsilon ' }< \delta \leq \norm{x_\ell - x_j}_w = \norm{x_\ell - x + x - x_j}_w \leq \norm{x_\ell - x}_w + \norm{x- x_j}_w \leq {\epsilon ' } /2 + {\epsilon ' }/2.
\]
As the balls $\sigBall_\ell$, $\ell \in \{1, \dots, M\}$, are non-overlapping, the signal $x[\cdot]=0$ is contained in at most one ball $\sigBall_\ell$ which we take to be $\sigBall_M$ w.l.o.g. Now, we define the set $\packingset$, with elements indexed by bitstring subscripts, according to
\begin{align*}
    \packingset & = \left\{g_{\alpha_1, \dots, \alpha_{M-1}}(\cdot)  \mid \alpha_\ell \in \{0,1\}, \textrm{ for } \ell \in \{1, \dots, M-1\} \right\}\hspace{-2pt}, \textrm{ where }\\
    g_{\alpha_1, \dots, \alpha_{M-1}}(x) & \defeq \begin{cases}
        (2\alpha_\ell - 1)({\epsilon ' /2} - \norm{x - x_\ell}_w), & \text{for } x\in\sigBall_\ell, \; \ell \in \{1, \dots, M-1\} \\
        0, & \text{else}, 
    \end{cases} 
\end{align*}
and show that $\packingset$ constitutes an $\epsilon$-packing for  $(\LFMsetZero{w}, \metricZero)$. First, we establish that $\packingset \subset \LFMsetZero{w}$
by verifying that every $g_{\alpha_1, \dots, \alpha_{M-1}}(\cdot)\in \packingset$ satisfies the conditions in \eqref{eq:nonLinSpace}.
Indeed, $\approxFunctional(0) = 0$ because the zero signal is either in no ball or in $\sigBall_M$. Next, we show that $|\approxFunctional(x) - \approxFunctional(x')| \leq \norm{x-x'}_w, \, \forall x,x' \in \sigSpaceminus$. This will be done by distinguishing cases. First, assume that $x$ and $x'$ are contained in the same ball $\sigBall_\ell$, for some $\ell \in \{1, \dots, M-1\}$. Then, we have
\begin{align}
    \begin{split}\label{teq:35}
    |\approxFunctional(x)- \approxFunctional(x')| &= |2\alpha_\ell-1|\cdot|\norm{x'-x_\ell}_w-\norm{x-x_\ell}_w| \\
    & \leq \norm{(x'-x_\ell) - (x-x_\ell)}_w = \norm{x'-x}_w,
    \end{split}
\end{align} 
where we used the reverse triangle inequality. Next, assume that $x\in\sigBall_\ell$ and $x'\in\sigBall_j$, with $\ell, j \in \{1, \dots, M-1\}$ and $\ell \neq j$. We let $z(\tmparg):= x + \tmparg (x'-x), \; \tmparg \in [0,1]$. Since $z(0)=x\in\sigBall_\ell$ and $z(1)=x'\notin\sigBall_\ell$ (because the balls are non-overlapping), there must be a $\tmparg_1 \in (0, 1) $ such that $\norm{z(\tmparg_1) - x_\ell}_w = {\epsilon'}/{2}$. As $z(\tmparg_1)\notin \sigBall_j$ and $z(1)=x'\in \sigBall_j$, there must similarly be a $\tmparg_2 \in (\tmparg_1, 1)$ so that $\norm{z(\tmparg_2)-x_j}_w=\epsilon'/2$. For notational simplicity, we let $z_1 \coloneqq z(\tmparg_1)$, $z_2 \coloneqq z(\tmparg_2)$, and $g\coloneqq\approxFunctional$.
Next, we bound
\begin{align}
    |g(x)-g(x')| 
        &\leq |g(x)-g(z_1)| + |g(z_1)-g(z_2)| + |g(z_2)-g(x')| \notag \\
        &\leq |g(x)-g(z_1)| + 0 + |g(z_2)-g(x')| \label{teq:39} \\
        &\leq \norm{x-z_1}_w + \norm{z_2-x'}_w \label{teq:40} \\
        &= \norm{\tmparg_1(x-x')}_w + \norm{(1-\tmparg_2)(x-x')}_w \label{teq:41}\\
        &= (1+ \tmparg_1 - \tmparg_2 )\norm{x-x'}_w < \norm{x-x'}_w \label{teq:42},
\end{align}
where in \eqref{teq:39} we used that $g(z_1)=g(z_2)=0$, in (\ref{teq:40}) we applied \eqref{teq:35} upon noting that $x, z_1$ and $z_2, x'$
are contained in the same ball each, in \eqref{teq:41} we inserted the definition of $z_1$ and $z_2$, and in \eqref{teq:42} we used $\tmparg_2 > \tmparg_1$. Finally, assume that $x\in \sigBall_\ell$, for some $\ell \in \{1, \dots, M-1\}$, and $x'\notin \sigBall_j$, for all $j\in\{1, \dots, M-1\}$. 
We let $z(\tmparg):= x + \tmparg (x'-x), \; \tmparg \in [0,1]$. Since $z(0)=x\in\sigBall_\ell$ and $z(1)=x'\notin\sigBall_\ell$, there must be a $\tmparg_1 \in (0, 1) $ such that $\norm{z(\tmparg_1) - x_\ell}_w = {\epsilon'}/{2}$. Again, we set $z_1 = z(\tmparg_1)$ and bound
\begin{align}
    |g(x)-g(x')| 
        &\leq |g(x)-g(z_1)| + |g(z_1)-g(x')| \notag \\
        &\leq |g(x)-g(z_1)| + 0 \label{teq:154} \\
        &\leq \norm{x-z_1}_w 
        = \norm{\tmparg_1(x-x')}_w \label{teq:155}\\
        &= \tmparg_1 \norm{x-x'}_w < \norm{x-x'}_w \label{teq:156},
\end{align}
where in \eqref{teq:154} we used that $g(z_1)=g(x')=0$, in \eqref{teq:155} we applied \eqref{teq:35} upon noting that $x, z_1$ are contained in the same ball, and in \eqref{teq:156} we used $\tmparg_1 < 1$.
We have thus established that $|\approxFunctional(x) - \approxFunctional(x')| \leq \norm{x-x'}_w, \; \forall x,x' \in \sigSpaceminus$, and hence $\packingset \subset \LFMsetZero{w}$.

Next, we show that for distinct $g_{\alpha_1, \dots, \alpha_{M-1}}, g_{\beta_1, \dots, \beta_{M-1}} \in \packingset$ (i.e., there is at least one $\ell \in \{1, \dots, M-1\}$ such that $\alpha_\ell \neq \beta_\ell$), it holds that $\metricZero(g_{\alpha_1, \dots, \alpha_{M-1}}, g_{\beta_1, \dots, \beta_{M-1}}) > \epsilon$. 
Indeed, for any such $\ell$, we have
\begin{align}
    \metricZero(g_{\alpha_1, \dots, \alpha_{M-1}}, g_{\beta_1, \dots, \beta_{M-1}}) &=\sup_{\xtilde\in\sigSpaceminus} |g_{\alpha_1, \dots, \alpha_{M-1}}(\xtilde) -g_{\beta_1, \dots, \beta_{M-1}}(\xtilde)| \label{teq:140}\\
    &\geq |g_{\alpha_1, \dots, \alpha_{M-1}}(x_\ell) -g_{\beta_1, \dots, \beta_{M-1}}(x_\ell)|\label{teq:141}\\
    &= |2(\alpha_\ell -\beta_\ell)\epsilon'/2 | = \epsilon' > \epsilon, \notag
\end{align}
where in \eqref{teq:140} we used \eqref{eq:def_zero_metric} and in \eqref{teq:141} we inserted the particular choice $\xtilde = x_\ell$ to lower-bound the $\sup$.
This establishes that $\packingset$, indeed, constitutes an $\epsilon$-packing for $(\LFMsetZero{w}, \metricZero)$ and we therefore have
\[
    \log\packing(\epsilon; \LFMsetZero{w}, \metricZero) \geq \log |\packingset| = M-1 = \left(\prod_{\ell=0}^T \ceil{\frac{2\bound w[\ell]}{{\epsilon}}}\right) -1. 
\]
\end{proof}

\subsection{Proof of Lemma \ref{lm:estimation}}\label{sec:ELFM_est}
\begin{proof}[\unskip\nopunct]
Note that $\weightSeqExp{a}{b}[t]> \epsilon/d$ gives $t < \frac{\log \left(\frac{ad}{\epsilon}\right)}{b\log (e)}$,
and thereby 
\begin{equation}\label{eq:T_elfm}
    T = \max \left\{ t \in \Nzero \;\middle|\; t < \frac{\log \left(\frac{ad}{\epsilon}\right)}{b\log (e)}\right\} = \ceil{\frac{\log \left(\frac{ad}{\epsilon}\right)}{b\log (e)}} - 1.
\end{equation}
Next, we have 
\begin{align}
    & \log\left( \prod_{\ell=0}^T \frac{c \weightSeqExp{a}{b}[\ell]}{\epsilon}\right)=  \log\left( \prod_{\ell=0}^T \frac{ac e^{-b\ell }}{\epsilon}\right) \nonumber\\
    &=  (T+1)\log(\epsilon^{-1}) - b \log (e) \frac{T(T+1)}{2} + (T+1)\log(ac)\nonumber\\
    &= \ceil{\frac{\log \left(\frac{ad}{\epsilon}\right)}{b\log (e)}}\log(\epsilon^{-1}) - b \log (e) \frac{\ceil{\frac{\log \left(\frac{ad}{\epsilon}\right)}{b\log (e)}}\left(\ceil{\frac{\log \left(\frac{ad}{\epsilon}\right)}{b\log (e)}}-1\right)}{2} \nonumber\\
     &\phantom{=} \;+ \ceil{\frac{\log \left(\frac{ad}{\epsilon}\right)}{b\log (e)}}\log(ac).\label{eq:est_expression}
\end{align}
Assuming that $\epsilon$ is sufficiently small to guarantee that $\frac{\log \left(\frac{ad}{\epsilon}\right)}{b\log (e)} -1\geq 0$, we can upper-bound (\ref{eq:est_expression}) according to 
\begin{align}
    &\log\left( \prod_{\ell=0}^T \frac{c \weightSeqExp{a}{b}[\ell]}{\epsilon}\right)\\
     &\leq \left( \frac{\log \left(\frac{ad}{\epsilon}\right)}{b\log (e)} +1\right)\log(\epsilon^{-1}) - b \log (e) \frac{ \frac{\log \left(\frac{ad}{\epsilon}\right)}{b\log (e)} \left( \frac{\log \left(\frac{ad}{\epsilon}\right)}{b\log (e)} -1\right)}{2} + \left( \frac{\log \left(\frac{ad}{\epsilon}\right)}{b\log (e)} +1\right)\log(ac)\nonumber\\
    &= \frac{1}{2b\log (e)} \log^2 (\epsilon^{-1}) + \left(\frac{\log (ad)}{b\log (e)}+1-\frac{1}{2b\log (e)}\left(2\log (ad)-b\log (e)\right)  \right.\nonumber \\
    &\phantom{=} \;\left.+ \frac{\log (ac)}{b\log (e)}\right)\log (\epsilon^{-1})-\frac{\log (ad)\left(\log (ad) - b\log (e)\right)}{2b \log (e)} + \log(ac) \left(\frac{\log(ad)}{b\log (e)}+1\right)\nonumber\\
    &= \frac{1}{2b\log (e)} \log^2(\epsilon^{-1}) + \smallo\left(\log^2(\epsilon^{-1})\right)\label{eq:est_upper}.
\end{align}
In the same spirit, we can lower-bound (\ref{eq:est_expression}) as 
\begin{align}
    &\log\left( \prod_{\ell=0}^T \frac{c \weightSeqExp{a}{b}[\ell]}{\epsilon}\right)\\
     &\geq  \left( \frac{\log \left(\frac{ad}{\epsilon}\right)}{b\log (e)} \right)\log(\epsilon^{-1}) - b \log (e) \frac{ \frac{\log \left(\frac{ad}{\epsilon}\right)}{b\log (e)} \left( \frac{\log \left(\frac{ad}{\epsilon}\right)}{b\log (e)} +1\right)}{2} + \left( \frac{\log \left(\frac{ad}{\epsilon}\right)}{b\log (e)} \right)\log(ac)\nonumber\\
    &= \frac{1}{2b\log (e)} \log^2 (\epsilon^{-1}) + \left(\frac{\log (ad)}{b\log (e)}-\frac{1}{2b\log (e)}\left(2\log (ad)+b\log (e)\right) \right. \nonumber \\
    & \phantom{=} \;\left.+ \frac{\log (ac)}{b\log (e)}\right)\log (\epsilon^{-1})-\frac{\log (ad)\left(\log (ad) + b\log (e)\right)}{2b \log (e)} + \log(ac) \left(\frac{\log(ad)}{b\log (e)}\right)\nonumber\\
    &= \frac{1}{2b\log (e)} \log^2(\epsilon^{-1}) + \smallo\left(\log^2(\epsilon^{-1})\right)\label{eq:est_lower}.
\end{align}

Finally, combining (\ref{eq:est_upper}) and (\ref{eq:est_lower}) yields 
$$\log\left( \prod_{\ell=0}^T \frac{c \weightSeqExp{a}{b}[\ell]}{\epsilon}\right)= \frac{1}{2b\log (e)} \log^2(\epsilon^{-1}) + \smallo\left(\log^2(\epsilon^{-1})\right),$$
as desired.
\end{proof}

\subsection{Proof of Lemma \ref{lm:estimation_poly}}\label{sec:PLFM_est}
\begin{proof}[\unskip\nopunct]
    We start by noting that 
    \begin{equation}\label{eq:T_plfm}
        T = \max\left\{t\in \Nzero \;\middle|\; \weightSeqPoly{p}{q}>\frac{\epsilon}{d}\right\}= \ceil{\left(\frac{dq}{\epsilon}\right)^{1/p}}-2.
    \end{equation}
    Next, 
    \begin{align}
        \log\left( \prod_{\ell=0}^T \frac{c \weightSeqPoly{p}{q}[\ell]}{\epsilon}\right) &= (T+1)\log (c) + (T+1)\log(\epsilon^{-1}) + \sum_{\ell=0}^{T} \log \left(\weightSeqPoly{p}{q}[\ell]\right) \nonumber \\ 
         &=(T+1)\log (cq) + (T+1)\log(\epsilon^{-1}) - p \log ((T+1)!) \nonumber \\
         &= (T+1)\log (cq) + (T+1)\log(\epsilon^{-1}) - p \left( (T+1)\log (T+1) \right. \nonumber\\ 
        &\phantom{=} \;\left.- (T+1) \log (e)\right) + \bigo(\log(T+1))  \label{eq:poly_bound_stirling}\\
        &=  (T+1)\log (cqe^p) + (T+1)\left(\log(\epsilon^{-1}) - p \log\left(\ceil{\left(\frac{dq}{\epsilon}\right)^{1/p}}-1\right)\right) \nonumber\\
        &\phantom{\leq} \; + \bigo(\log(\epsilon^{-1})), \label{eq:lm_poly_est}
    \end{align}
    where \eqref{eq:poly_bound_stirling} follows from the logarithm form of Stirling's approximation, namely
    \begin{equation}
        \log(n!)=n\log (n)-n\log (e)+\bigo(\log (n)).
    \end{equation}
    
    Now, assuming that  $\epsilon$ is sufficiently small to guarantee that $1\leq \frac{1}{2}(dq/\epsilon)^{1/p}$, applying $\log\left((dq/\epsilon)^{1/p}-1\right)\geq \log\left(\frac{1}{2}(dq/\epsilon)^{1/p}\right) \geq 0$, we can upper-bound \eqref{eq:lm_poly_est} as
    \begin{align}
        &(T+1)\log (cqe^p) + (T+1)\left(\log(\epsilon^{-1}) - p \log\left(\left(\frac{dq}{\epsilon}\right)^{1/p}-1\right)\right)+ \bigo(\log(\epsilon^{-1})) \nonumber\\
        &\leq (T+1)\log (cqe^p) + (T+1)\left(\log(\epsilon^{-1}) - p \log\left(\frac{1}{2}\left(\frac{dq}{\epsilon}\right)^{1/p}\right)\right)+ \bigo(\log(\epsilon^{-1})) \nonumber\\
        &=(T+1)(\log (ce^p/d) + p) + \bigo(\log(\epsilon^{-1})) \nonumber\\
        &= \left(\ceil{\left(\frac{dq}{\epsilon}\right)^{1/p}}-1\right)(\log (ce^p/d) + p) + \bigo(\log(\epsilon^{-1})) \nonumber\\
        &\leq \left(\frac{dq}{\epsilon}\right)^{1/p} (\log (ce^p/d) + p) + \bigo(\log(\epsilon^{-1})). \label{eq:lm_poly_est_upper}
    \end{align}
    Furthermore, we can lower-bound \eqref{eq:lm_poly_est} according to 
    \begin{align}
        &(T+1)\log (cqe^p) + (T+1)\left(\log(\epsilon^{-1}) - p \log\left(\left(\frac{dq}{\epsilon}\right)^{1/p}\right)\right)+ \bigo(\log(\epsilon^{-1})) \nonumber\\
        &= (T+1)\log (ce^p/d) + \bigo(\log(\epsilon^{-1})) \nonumber\\
        &= \left(\ceil{\left(\frac{dq}{\epsilon}\right)^{1/p}}-1\right)\log (ce^p/d) + \bigo(\log(\epsilon^{-1})) \nonumber\\
        &\geq  \left(\left(\frac{dq}{\epsilon}\right)^{1/p} - 1\right) \log (ce^p/d) + \bigo(\log(\epsilon^{-1})) \nonumber\\
        &= \left(\frac{dq}{\epsilon}\right)^{1/p} \log (ce^p/d) + \bigo(\log(\epsilon^{-1})). \label{eq:lm_poly_est_lower}
    \end{align}    

    Combining \eqref{eq:lm_poly_est}, \eqref{eq:lm_poly_est_upper}, and \eqref{eq:lm_poly_est_lower} yields the desired result 
    \begin{equation*}
        \log\left( \prod_{\ell=0}^T \frac{c \weightSeqPoly{p}{q}[\ell]}{\epsilon}\right) = \asymporder(\epsilon^{-1/p}).
    \end{equation*}
\end{proof}

\subsection{Proof of Lemma \ref{lm:dim_PLFM}}\label{sec:app_rate_plfm}
\begin{proof}[\unskip\nopunct]
    Consider $\epsilon \in (0,\epsilon_0)$ with $\epsilon_0=\frac{D\weightSeqPoly{p}{q}[0]}{2} = \frac{aD}{2}$. By Theorem \ref{th:main_entropy_result}, the exterior covering number of $(\LFMset{\weightSeqPoly{p}{q}},\rho)$ satisfies
    \begin{equation}\label{eq:bound_covering_poly}
        \left (\prod_{\ell=0}^{T'} \ceil{\frac{\bound \weightSeqPoly{p}{q}[\ell]}{{\epsilon}}} \right ) - 1 
        \leq \log \covering^{\text{ext}}(\epsilon; \LFMset{\weightSeqPoly{p}{q}}, \opMetric) 
        \leq \log \left(3\right)\prod_{\ell=0}^{T''}\left(2\ceil{\frac{4\bound \weightSeqPoly{p}{q}[\ell]}{\epsilon}  } + 1\right),
    \end{equation}
where 
\begin{equation*}
    T^{\prime} \defeq \max \left\{ \ell \in \Nzero \Bigm| \weightSeqPoly{p}{q}[\ell]> \frac{\epsilon}{\bound}\right\} \text{ and } T''  \defeq  \max \left\{ \ell \in \Nzero \Bigm| \weightSeqPoly{p}{q}[\ell]> \frac{\epsilon}{4\bound } \right\} .
\end{equation*}
We can further lower-bound the left-most term in \eqref{eq:bound_covering_poly} according to
\begin{align}
    \left (\prod_{\ell=0}^{T^{\prime}} \left\lceil\frac{\bound \weightSeqPoly{p}{q}[\ell]}{{\epsilon}}\right\rceil \right ) - 1 &\geq \left(\prod_{\ell=0}^{T^{\prime}} \frac{\bound \weightSeqPoly{p}{q}[\ell]}{{\epsilon}}\right) -1 \nonumber \\
    & \geq \frac{1}{2} \prod_{\ell=0}^{T^{\prime}} \frac{\bound \weightSeqPoly{p}{q}[\ell]}{{\epsilon}} \label{eq:lower_bound_covering_poly},
\end{align}
where \eqref{eq:lower_bound_covering_poly} follows from
\begin{equation*}
    \frac{1}{2} \prod_{\ell=0}^{T^{\prime}} \frac{\bound \weightSeqPoly{p}{q}[\ell]}{{\epsilon}}\geq 1,
\end{equation*}
which, in turn, is a consequence of 
\begin{align*}
   & \frac{\bound \weightSeqPoly{p}{q}[0]}{{\epsilon}} \geq \frac{\bound \weightSeqPoly{p}{q}[0]}{{\epsilon_0}} = 2,\\
   & \frac{\bound \weightSeqPoly{p}{q}[\ell]}{{\epsilon}} \geq \frac{\bound \weightSeqPoly{p}{q}[T']}{{\epsilon}} > 1, \quad \text{for}~~\ell \in \ssidedEnum{T'}\setminus\{0\}.
\end{align*}
Similarly, we can further upper-bound the right-most term in (\ref{eq:bound_covering_poly}) as
\begin{align}
    \log(3)\left (\prod_{\ell=0}^{T''} \left(2\ceil{\frac{4\bound \weightSeqPoly{p}{q}[\ell]}{\epsilon}  } + 1\right) \right ) &\leq \log(3)\left (\prod_{\ell=0}^{T''} \left(\frac{8\bound \weightSeqPoly{p}{q}[\ell]}{{\epsilon}}+3\right) \right ) \nonumber \\
    & \leq \log(3)\left (\prod_{\ell=0}^{T''} \frac{20\bound \weightSeqPoly{p}{q}[\ell]}{{\epsilon}} \right )\label{eq:upper_bound_covering_poly},
\end{align}
where \eqref{eq:upper_bound_covering_poly} follows from 
$$\frac{4\bound \weightSeqPoly{p}{q}[\ell]}{{\epsilon}} \geq \frac{4\bound \weightSeqPoly{p}{q}[T'']}{{\epsilon}} > 1, \quad \text{for}~~\ell \in \ssidedEnum{T''} .$$

Combining 
(\ref{eq:bound_covering_poly})--(\ref{eq:upper_bound_covering_poly}), 
taking logarithms two more times, dividing the results by $\log (\epsilon^{-1})$, and applying Lemma \ref{lm:estimation_poly}, we obtain 
\begin{equation} \label{eq:covering_final_bound_poly}
\frac{\log\left(\asymporder(\epsilon^{-1/p})\right)}{\log (\epsilon^{-1})}\leq \frac{\log^{(3)} \covering^{\text{ext}}(\epsilon; \LFMset{\weightSeqPoly{p}{q}},\opMetric)}{\log (\epsilon^{-1})} \leq \frac{\log\left(\asymporder(\epsilon^{-1/p})\right)}{\log (\epsilon^{-1})}.
\end{equation}
Taking the limit $\epsilon \rightarrow 0$, then yields 
$$\lim_{\epsilon \rightarrow 0} \frac{\log^{(3)} \covering^{\text{ext}}(\epsilon; \LFMset{\weightSeqPoly{p}{q}}
,\opMetric)}{\log (\epsilon^{-1})} = \frac{1}{p},$$
which implies that $(\LFMset{\weightSeqPoly{p}{q}},\opMetric)$ is of order 2 and type 1, with generalized dimension $$\genDim{} = \frac{1}{p}.$$
\end{proof}

\subsection{Auxiliary results on \gls{relu} networks}\label{sec:app_minmax}
\begin{lemma}\label{lm:NN2minmax}
For $d \in \Nplus$ with $d\geq 2$, consider the following functions,
\begin{align*}
    f^{\text{min}}_d(z) &\defeq \min\{z_1,z_2,\dots,z_d\}, \qquad \text{for }  z \in \R^{d},\\
    f^{\text{max}}_d(z) &\defeq \max\{z_1,z_2,\dots,z_d\}, \qquad \text{for } z \in \R^{d}.
\end{align*}
Then, there exist \gls{relu} networks $\Phi^{\text{min}}_d\in \relunn_{d,1}$ and $\Phi^{\text{max}}_d\in \relunn_{d,1}$, with $\depth(\Phi_d^{\text{min}})= \depth(\Phi_d^{\text{max}})=\ceil{\log(d+1)}+1$, $\nnz(\Phi_d^{\text{min}})=\nnz(\Phi_d^{\text{max}})\leq 14d-9$, $\width(\Phi_d^{\text{min}})=\width(\Phi_d^{\text{max}})\leq 3d$, $\weightsSet(\Phi_d^{\text{min}})=\weightsSet(\Phi_d^{\text{max}})=\{1,-1\}$, $\weightmeg(\Phi_d^{\text{min}})=\weightmeg(\Phi_d^{\text{max}})= 1$,
such that 
\begin{align*}
    \Phi^{\text{min}}_d (z) &= f^{\text{min}}_d(z), \qquad \text{for all} \quad z \in \R^{d},\\
    \Phi^{\text{max}}_d (z) &= f^{\text{max}}_d(z) ,\qquad \text{for all} \quad z \in \R^{d}.
\end{align*}
\end{lemma}
\begin{proof}
    We only need to show the result for the max function as $\min\{z_1, z_2, \dots, z_d\} = -\max\{-z_1,\allowbreak -z_2, \dots, -z_d\}$. First, we realize $\max\{x_1,x_2\}$ according to 
    \begin{equation}\label{eq:max_two}
        \max\{x_1,x_2\} = x_1 + \relu(x_2-x_1)
    \end{equation}
    by the \gls{relu} network
\begin{equation}\label{eq:max_two_nn}
    \widehat{\Phi} = \widehat{W}_2\circ \relu \circ \widehat{W}_1,
\end{equation}
with
\begin{equation}\label{eq:weight_A2A1}
    \begin{aligned}
    \widehat{W}_1(x) &= \begin{pmatrix}
     1& 0\\ -1& 0\\ -1& 1\end{pmatrix}x \eqdef A_1x,\\
    \widehat{W}_2(x) &= \begin{pmatrix} 
    1&-1 & 1\end{pmatrix}x\eqdef A_2x.
\end{aligned}
\end{equation}

Now, for arbitrary $d\in \Nplus$ with $d\geq 2$, choose $\ell\in \Nplus$ such that $2^\ell\leq d< 2^{\ell+1}$. Then, we double the first $k=2^{\ell+1}-d$ elements and retain the remaining $d-k$ elements as follows
\begin{equation}\label{eq:max_extension}
    \max\{x_1,\dots,x_d\} = \max\{x_1,x_1,\dots,x_k,x_k,x_{k+1},x_{k+2},\dots,x_{d}\}.
\end{equation}
The primary reason for doubling elements is that by extending the set $\{x_1,\dots,x_d\}$ to a set whose cardinality is a power of 2, we can utilize \eqref{eq:max_two} together with a divide-and-conquer approach to determine the maximum value of the set $\{x_1,\dots,x_d\}$. This then leads to a ReLU network realization of depth scaling logarithmically in $d$. Now, let $W_{\ell+1} (x) \defeq A_2 x$, $W_{j}(x)\defeq B_j x$, for $j\in \{\ell,\ell-1,\dots,1\}$, and $W_{0}(x)\defeq A_0x$, with 
\begin{equation}\label{eq:max_nn_mats}
    \begin{aligned}
        B_o &= A_1  \operatorname{diag}\left(A_2,A_2\right)\\
        &=\begin{pmatrix}
        1&-1&1&0&0&0\\-1&1&-1&0&0&0\\-1&1&-1&1&-1&1\end{pmatrix},\\
        B_j &= \operatorname{diag}(\underbrace{B_o,\dots,B_o}_{2^{\ell-j}}),\\  
        A_{0}&= \operatorname{diag}(\underbrace{A_1,\dots,A_1}_{2^{\ell}})\operatorname{diag}(\underbrace{\ind_{2},\dots,\ind_{2}}_{k},\Imat{d-k})\\
        &= \operatorname{diag}(\underbrace{A_1\ind_{2},\dots,A_1\ind_{2}}_{k},\underbrace{A_1,\dots,A_1}_{2^\ell-k}),\\
        A_1\ind_{2}&= (1,-1,0)^T.
        \end{aligned}
\end{equation}
Combining \eqref{eq:max_two}-\eqref{eq:max_nn_mats}, we can show that 
\begin{equation}\label{eq:construct_nn_max}
    \Phi_d^{max} \defeq W_{\ell+1}\circ \relu \circ W_{\ell} \circ \relu \circ \dots \relu \circ W_{1} \circ \relu \circ W_0 = f^{\text{max}}_d
\end{equation}
as follows:
\begin{enumerate}
    \item[(i)] We double the first $k=2^{\ell+1}-d$ elements and retain the remaining $d-k$ elements of $\{x_1,x_2,\dots,x_d\}$ according to
    \begin{equation*}
    \begin{aligned}
        &\begin{pmatrix}
        x_1&x_1&\dots&x_k&x_k&x_{k+1}&x_{k+2}&\dots&x_{d}\end{pmatrix}^T \\
        &= \operatorname{diag}(\underbrace{\ind_{2},\dots,\ind_{2}}_{k},\Imat{d-k}) 
        \begin{pmatrix}
        x_1&x_2&\dots&x_{d}\end{pmatrix}^T.
    \end{aligned}
    \end{equation*}
    For simplicity, we denote the resulting set $\{x_1,x_1,\dots,x_k,x_k,x_{k+1},x_{k+2},\dots,x_{d}\}$ as $\{y_1,y_2,\allowbreak \dots,y_{2^{\ell+1}}\}$.
    \item[(ii)] Application of \eqref{eq:max_two_nn} to the $2^{\ell}$ pairs in $\{y_1,y_2,\dots,y_{2^{\ell+1}}\}$, results in $$\{\max\{y_1,y_2\},\max\{y_3,y_4\},\dots,\max\{y_{2^{\ell+1}-1},y_{2^{\ell+1}}\}\}$$ and hence reduces the number of elements from $2^{\ell+1}$ to $2^{\ell}$. This reduction is applied iteratively until we get $\max\{x_1,x_2,\dots,x_d\}$. The compositions $A_1\circ \operatorname{diag}\left(A_2,A_2\right)$, formed in the iterative application of \eqref{eq:max_two_nn}, constitute the main diagonal elements of the matrices $B_k$, $k\in \{\ell,\ell-1,\dots,1\}$. We illustrate this part of the procedure using the simplest possible example, namely for $\ell=1$ and hence $\max\{y_1,y_2,y_3,y_4\}$. 
    Specifically, we have
    \begin{equation*}
        \begin{aligned}
            &\max\{y_1,y_2,y_3,y_4\} = \max\{\max\{y_1,y_2\},\max\{y_3,y_4\}\}\\
            &=\max\left\{\operatorname{diag}\left(A_2,A_2\right) \relu\left( \operatorname{diag}\left(A_1,A_1\right)((y_1,y_2),(y_3,y_4))^T\right)\right\}\\
            &= A_2\relu\left( A_1 \operatorname{diag}\left(A_2,A_2\right) \relu\left( \operatorname{diag}\left(A_1,A_1\right)((y_1,y_2),(y_3,y_4))^T\right)\right)\\
            &= (W_2 \circ \relu \circ  W_1 \circ \relu \circ W_0) (y).
        \end{aligned}
    \end{equation*}
\end{enumerate}

The proof is concluded by noting that 
\begin{equation*}
    \begin{aligned}
    \depth(\Phi_d^{\text{max}}) \overset{\text{\eqref{eq:construct_nn_max}}}&{=}\ell+1=\ceil{\log(d+1)}+1,\\
    \nnz(\Phi_d^{\text{max}}) \overset{\text{\eqref{eq:construct_nn_max}}}&{=} \sum_{j=0}^{\ell+1}\nnz(W_j) = \nnz(A_2)+\sum_{j=1}^{\ell}\nnz(B_j) + \nnz(A_0)\\
    \overset{\text{\eqref{eq:max_nn_mats}, \eqref{eq:weight_A2A1}}}&{=}  3 + 2k + 4(2^{\ell}-k) + 12\sum_{j=1}^{\ell} 2^{\ell-j} \leq 14d-9,\\
    \width(\Phi_d^{\text{max}}) \overset{\text{\eqref{eq:construct_nn_max}}}&{=} \max_{j=0,1\dots,\ell+1} \width(W_j) \overset{\text{\eqref{eq:max_nn_mats}, \eqref{eq:weight_A2A1}}}{=} 3 \cdot 2^{\ell}\leq 3d,\\
    \weightsSet(\Phi_d^{\text{max}}) \overset{\text{\eqref{eq:construct_nn_max}}}&{\subset} \bigcup_{j=0}^{\ell+1} \weightsSet(W_j) \overset{\text{\eqref{eq:max_nn_mats}, \eqref{eq:weight_A2A1}}}{=} \{1,-1\}, \\
    \weightmeg(\Phi_d^{\text{max}}) &= \max_{b\in \weightsSet(\Phi_d^{\text{max}})} |b|=1.
    \end{aligned}
\end{equation*}
\end{proof}
\begin{lemma}[Composition of \gls{relu} networks \cite{deepAT2019}]\label{lm:compo_nn}
    For $i=1,\dots,n$, let $d_i\in \Nplus$ and $\Phi_i\in \relunn_{d_i,d_{i+1}}$. Then, there exists a network $\Psi\in \relunn_{d_1,d_{n+1}}$ with
    \begin{equation}
        \Psi(x)=(\Phi_{n}\circ\Phi_{n-1}\circ\cdots\circ \Phi_1)(x),  \quad\text{for all }x\in\R^{d_1},
    \end{equation}
    satisfying
    \begin{equation}
        \begin{aligned}
        \depth(\Psi)& = \sum_{i=1}^n\depth(\Phi_i),\\
        \nnz(\Psi)& \leq2\sum_{i=1}^n\nnz(\Phi_i)  ,\\
        \width(\Psi)&\leq\max\left\{\max_{i=1,...,n-1}\{2d_i\},\max_{i=1,...,n}\{\width(\Phi_i)\}\right\},\\
        \weightsSet(\Psi)& \subset\bigcup_{i=1}^n(\weightsSet(\Phi_i)\cup(-\weightsSet(\Phi_i))),\\
        \weightmeg(\Psi) &= \max_{i=1,\dots,n}\weightmeg(\Phi_i).
        \end{aligned}
    \end{equation}
\end{lemma}
\begin{proof}
    Follows along the same lines as the proof of \cite[Lemma II.3]{deepAT2019}.
\end{proof}

\begin{lemma}[Parallelization of \gls{relu} networks of the same depth \cite{deepAT2019}]\label{lm:para_nn}
    For $i=1,\dots,n$, let $d_i,d'_i\in \Nplus$ and $\Phi_i\in \relunn_{d_i,d'_i}$ with $\depth(\Phi_i)=L$. Then, there exists a \gls{relu} network $P(\Phi_1,\Phi_2,\dots,\Phi_n) \allowbreak \in \relunn_{\sum_{i=1}^nd_i,\sum_{i=1}^nd'_i}$ with
    \begin{equation}
        P(\Phi_1,\Phi_2,\ldots,\Phi_n)(x)=(\Phi_1(x),\Phi_2(x),\ldots,\Phi_n(x))^T,\quad\text{for all }x\in\R^{\sum_{i=1}^nd_i},
    \end{equation}
    satisfying
    \begin{equation}
        \begin{aligned}
        \depth(P(\Phi_1,\Phi_2,\ldots,\Phi_n))& =L,  \\
        \nnz(P(\Phi_1,\Phi_2,\ldots,\Phi_n))& =\sum_{i=1}^n \nnz(\Phi_i), \\
        \width(P(\Phi_1,\Phi_2,\ldots,\Phi_n))&\leq\sum_{i=1}^n\width(\Phi_i), \\
        \weightsSet(P(\Phi_1,\Phi_2,\ldots,\Phi_n))& =\bigcup_{i=1}^n\weightsSet(\Phi_i),\\
        \weightmeg(P(\Phi_1,\Phi_2,\ldots,\Phi_n)) &=\max_{i=1,\dots,n}\weightmeg(\Phi_i).
        \end{aligned}
    \end{equation}
\end{lemma}
\begin{proof}
    Follows along similar 
    lines as the proof of \cite[Lemma II.5]{deepAT2019}.
\end{proof}

\begin{lemma}[Augmenting network depth \cite{deepAT2019}]\label{lm:aug_depth}
    Let $d_1,d_2,K\in \Nplus$, and $\Phi \in \relunn_{d_1,d_2}$ with $\depth(\Phi)<K$. Then, there exists a network $\Psi\in \relunn_{d_1,d_2}$ with 
    \begin{equation}
        \begin{aligned}
            \depth(\Psi)&=K,\\
            \nnz(\Psi)&\leq \nnz(\Phi)+d_2\width(\Phi)+2d_2(K-\depth(\Phi)),\\
            \width(\Psi)&=\max\{2d_2,\width(\Phi)\},\\
            \weightsSet(\Psi) &\subset \left(\weightsSet(\Phi)\cup (-\weightsSet(\Phi)) \cup \{1,-1\}\right),\\
            \weightmeg(\Psi) &= \max\{1, \weightmeg(\Phi)\},
        \end{aligned}
    \end{equation}
    satisfying $\Psi(x) = \Phi(x)$, for all $x \in \R^{d_1}$.
\end{lemma}
\begin{proof}
    Follows along the same lines as the proof of \cite[Lemma II.4]{deepAT2019}.
\end{proof}

\subsection{Remainder of the proof of Lemma \ref{lm:NN2G_0}}\label{sec:NN2G0_est}
\begin{proof}[\unskip\nopunct]
    We verify that $\Phi=\Phi_2\circ\Phi_1$ defined in \eqref{eq:structure_NN} has $(2,\epsilon)$-quantized weights.
\begin{itemize}
    \item The weights in $\Phi_2=\widetilde{W}^{\Sigma}$ defined in \eqref{eq:weights_coe_shift}: 
    \begin{align*}
        \gQuan_{\vecindex} \overset{\eqref{eq:quantize_g}}&{=} \quantize{2}{\epsilon}(g(\xtilde_{\vecindex})) \overset{\text{Definition \ref{def:quan_nn}}}{=} \ceil{g(\xtilde_{\vecindex}) / 2^{-2\ceil{\log (\epsilon^{-1})}}} \cdot 2^{-2\ceil{\log (\epsilon^{-1})}}\\
        &\in 2^{-2\ceil{\log (\epsilon^{-1})}} \ZZ,\nonumber\\
        \left|\gQuan_{\vecindex}\right| &\leq \left|g(\xtilde_{\vecindex})\right| + 2^{-2\ceil{\log (\epsilon^{-1})}} \overset{\eqref{eq:nonLinSpace}}{\leq} D+2^{-2\ceil{\log (\epsilon^{-1})}} \overset{\eqref{eq:eps0}}{\leq} D+\epsilon^2\\
        &\leq D+1 \overset{\eqref{eq:eps0}}{\leq} \epsilon^{-2},\nonumber
    \end{align*}
    which yields
    \begin{equation}\label{eq:weights_Phi_2}
        \weightsSet(\Phi_2) \subset 2^{-2\ceil{\log (\epsilon^{-1})}} \ZZ \cap \left[-\epsilon^{-2},\epsilon^{-2}\right].
    \end{equation}
    \item The weights in $\Phi_1$ defined in \eqref{eq:para_shifted_spike}: 
    \begin{equation}\label{eq:weights_Phi_1_decompose}
        \weightsSet(\Phi_1) = \bigcup_{\vecindex \in \latticeNQuan} \left(\weightsSet(\Phi^{\vecindex}_{1,2}) \cup \weightsSet(\Phi^{\vecindex}_{1,1})\right) = \weightsSet(\Psi)\cup \left(\bigcup_{\vecindex \in \latticeNQuan}~\weightsSet(\widehat{W}_{\vecindex})\right).
    \end{equation}
    Recalling that $\Psi$ realizes the spike function and applying Lemma \ref{lm:NN2phi}, we obtain
    \begin{equation}\label{eq:weights_Psi}
        \weightsSet(\Psi)=\{1,-1\}\subset2^{-2\ceil{\log (\epsilon^{-1})}} \ZZ \cap \left[-\epsilon^{-2},\epsilon^{-2}\right].
    \end{equation}
    Based on \eqref{eq:weights_coe_shift},
    we get
    \begin{equation}\label{eq:weights_shift}
        \bigcup_{\vecindex \in \latticeNQuan}\weightsSet(\widehat{W}_{\vecindex}) = \bigcup_{\ell=0}^{T}\left(\bigcup_{\vecindex \in \latticeNQuan}\{\deltaQuan_{\ell}^{-1}, \vecindex_{\ell}\} \right).
    \end{equation}
    Moreover, for all $\ell =0,\dots,T$ and $\vecindex \in \latticeNQuan$, we have
    \begin{align*}
        \deltaQuan_{\ell}^{-1} \overset{\eqref{eq:quantize_delta}}&{=} \quantize{2}{\epsilon}(\delta_{\ell}^{-1})\overset{\text{Definition \ref{def:quan_nn}}}{=} \ceil{\frac{s+1}{s}\frac{w[\ell]}{\epsilon} \Big/ 2^{-2\ceil{\log (\epsilon^{-1})}}} \cdot 2^{-2\ceil{\log (\epsilon^{-1})}} \\
        &\in 2^{-2\ceil{\log (\epsilon^{-1})}} \ZZ ~,\\
        \left|\deltaQuan_{\ell}^{-1}\right| &\leq \frac{s+1}{s}\frac{w[\ell]}{\epsilon} + 2^{-2\ceil{\log (\epsilon^{-1})}} \overset{\text{Definition \ref{def:lip_fading_memory}}}{\leq} \frac{s+1}{s} \epsilon^{-1} +\epsilon^2 \overset{\eqref{eq:eps0}}{\leq} \epsilon^{-2}.
    \end{align*} 
    Hence,
    \begin{equation}\label{eq:check_delta_quan}
        \deltaQuan_{\ell}^{-1} \in 2^{-2\ceil{\log (\epsilon^{-1})}} \ZZ \cap \left[-\epsilon^{-2},\epsilon^{-2}\right],
    \end{equation}
    and
    \begin{align*}
        \vecindex_{\ell} &\in \ZZ \subset 2^{-2\ceil{\log (\epsilon^{-1})}} \ZZ~,\\
        \left|\vecindex_{\ell}\right| &\leq  \latticeBoundQuan_{\ell} \overset{\eqref{eq:lattice_parameters}}{=} \ceil{\frac{D}{\deltaQuan_{\ell}}} \leq D \left|\deltaQuan_{\ell}^{-1}\right|+1 \\
        &\leq D\left(\frac{s+1}{s} \epsilon^{-1} +\epsilon^2\right) +1 \\
        &\overset{\eqref{eq:eps0}}{\leq} \left(D\frac{s+1}{s}+D+1\right)\epsilon^{-1}\overset{\eqref{eq:eps0}}{\leq} \epsilon^{-2},
    \end{align*}  
    which yields 
    \begin{equation}\label{eq:check_n_quan}
        \vecindex_{\ell} \in 2^{-2\ceil{\log (\epsilon^{-1})}} \ZZ \cap \left[-\epsilon^{-2},\epsilon^{-2}\right].
    \end{equation}
    Based on \eqref{eq:weights_shift}, \eqref{eq:check_delta_quan}, and \eqref{eq:check_n_quan}, we have
    \begin{equation}\label{eq:weights_W_n}
        \bigcup_{\vecindex \in \latticeNQuan}\weightsSet(\widehat{W}_{\vecindex}) \subset  2^{-2\ceil{\log (\epsilon^{-1})}} \ZZ \cap \left[-\epsilon^{-2},\epsilon^{-2}\right].
    \end{equation}
    Combining \eqref{eq:weights_Phi_1_decompose}, \eqref{eq:weights_Psi}, \eqref{eq:weights_shift}, and \eqref{eq:weights_W_n} yields 
    \begin{equation}\label{eq:weights_Phi_1}
        \weightsSet(\Phi_1) \subset 2^{-2\ceil{\log (\epsilon^{-1})}} \ZZ \cap \left[-\epsilon^{-2},\epsilon^{-2}\right].
    \end{equation}
\end{itemize}
Using \eqref{eq:weights_Phi_2} and \eqref{eq:weights_Phi_1}, Lemma \ref{lm:compo_nn} shows that $\Phi=\Phi_2\circ\Phi_1$, indeed, has $(2,\epsilon)$-quantized weights.

Finally, we compute $\depth(\Phi)$ and derive an upper bound on $\nnz(\Phi)$ according to
\begin{align*}
    \depth(\Phi) \overset{\eqref{eq:structure_NN}, \text{Lemma} \ref{lm:compo_nn}}&{=} \depth(\Phi_2)+\depth(\Phi_1)\\
    \overset{\eqref{eq:para_shifted_spike}, \eqref{eq:structure_NN_parts}, \text{Lemma} \ref{lm:para_nn}}&{=} \depth(\Psi) + 2\\
    \overset{\text{Lemma} \ref{lm:NN2phi}}&{=} \ceil{\log(T+2)} + 6,\\
    \nnz(\Phi) \overset{\text{Lemma \ref{lm:compo_nn}}}&{\leq} 2 \left(\nnz(\Phi_2) + \nnz(\Phi_1)\right)\\
    \overset{\text{Lemma \ref{lm:para_nn}}}&{=} 2 \left(\nnz(\Phi_2) + \sum_{i=1}^{|\latticeNQuan|}\left(2\nnz(\Psi)+2\nnz(\widehat{W}_{\vecindex^i})\right)\right)\\
    \overset{\text{\eqref{eq:weights_coe_shift}-\eqref{eq:structure_NN}, Lemma \ref{lm:NN2phi}}}&{\leq} 2 |\latticeNQuan| \left(1 + 120(T+1)-56 + 2(T+1)\right)\\
    & \leq 244(T+1)|\latticeNQuan|\\
     \overset{\eqref{eq:lattice_parameters}}&{\leq}  244(T+1) \prod_{\ell=0}^{T} \left(2D\deltaQuan_{\ell}^{-1} +3\right)\\
     \overset{\eqref{eq:lattice_parameters}, \eqref{eq:quantize_delta}}&{\leq} 244(T+1)  \prod_{\ell=0}^{T} \left(2D\left(\frac{s+1}{s}\frac{w[\ell]}{\epsilon} + \epsilon^2\right) +3\right)\\
    \overset{\eqref{eq:eps0}}&{\leq} 244(T+1)  \prod_{\ell=0}^{T} \left(\frac{2Dw[\ell]}{\epsilon}\frac{s+1}{s} +4\right).
\end{align*}
The proof is concluded by setting
\begin{equation}\label{eq:eps0}
    \epsilon_0 = \min\left\{1,\frac{1}{2}, \frac{1}{s+1},\sqrt{\frac{1}{D+1}},\left(D\frac{s+1}{s}+D+1\right)^{-1}\right\}.
\end{equation}
\end{proof}

\section*{Acknowledgments}
The authors would like to thank Erwin Riegler for useful suggestions regarding the proof of Lemma \ref{lm:NN2minmax} and Thomas Allard for help with Table \ref{tbl:nn}.

\bibliography{arXiv_reference}
\end{document}